%% file: neurips_2026.tex
\documentclass{article}

\usepackage{microtype}
\usepackage{graphicx}
\graphicspath{{Figures/camera_ready/}{Figures/}}

\usepackage{graphicx}
\usepackage{subcaption}
\usepackage{booktabs} 

\usepackage{hyperref}

\PassOptionsToPackage{numbers,sort&compress}{natbib}
\usepackage[main,final]{neurips_2026}

\usepackage{amsmath}
\DeclareMathOperator{\rank}{rank}
\DeclareMathOperator{\Sp}{Sp}
\usepackage{amssymb}
\usepackage{mathtools}
\usepackage{amsthm}
\usepackage{subcaption} 

\usepackage[capitalize,noabbrev]{cleveref}

\theoremstyle{plain}
\newtheorem{theorem}{Theorem}[section]
\newtheorem{proposition}[theorem]{Proposition}
\newtheorem{lemma}[theorem]{Lemma}
\newtheorem{corollary}[theorem]{Corollary}
\theoremstyle{definition}

\theoremstyle{remark}

\usepackage[textsize=tiny]{todonotes}

\usepackage{algorithm}

\usepackage{float}    
\usepackage{placeins} 

\usepackage{algpseudocode} 
\algrenewcommand\algorithmicrequire{\textbf{Input:}}
\algrenewcommand\algorithmicensure{\textbf{Output:}}

\let\STATE\State

\usepackage{algpseudocode} 
\algrenewcommand\algorithmicrequire{\textbf{Input:}}
\algrenewcommand\algorithmicensure{\textbf{Output:}}

\usepackage{capt-of}

\usepackage{tikz}
\usetikzlibrary{
  arrows.meta,
  positioning,
  calc,
  fit,
  backgrounds,
  decorations.pathreplacing,
  shapes.geometric
}

\definecolor{hamblue}{RGB}{52,114,183}
\definecolor{hamorange}{RGB}{221,132,82}
\definecolor{hamgreen}{RGB}{78,167,121}
\definecolor{hamred}{RGB}{196,78,82}
\definecolor{hamgold}{RGB}{204,160,70}
\definecolor{hamgray}{RGB}{90,90,90}
\definecolor{hamlight}{RGB}{246,248,251}
\definecolor{hamlightblue}{RGB}{232,240,250}
\definecolor{hamlightorange}{RGB}{251,239,229}
\definecolor{hamlightgreen}{RGB}{234,244,238}
\definecolor{hamlightred}{RGB}{250,236,236}
\definecolor{hamlightgold}{RGB}{249,244,230}

\newcommand{\hm}[1]{\(\vphantom{\int}\displaystyle #1\)}

\tikzset{
  hamarrow/.style={
    -{Latex[length=2.8mm,width=1.9mm]},
    line width=0.95pt,
    draw=black!70
  },
  hamsoftarrow/.style={
    -{Latex[length=2.2mm,width=1.5mm]},
    line width=0.7pt,
    draw=black!50
  },
  hambox/.style={
    draw=black!70,
    rounded corners=2.5pt,
    line width=0.9pt,
    fill=white,
    align=center,
    inner sep=4pt,
    font=\small,
    text height=1.8ex,
    text depth=.3ex
  },
  hamsmall/.style={
    draw=black!65,
    rounded corners=2pt,
    line width=0.8pt,
    fill=white,
    align=center,
    inner sep=2.8pt,
    font=\footnotesize,
    text height=1.7ex,
    text depth=.28ex
  },
  hamgroup/.style={
    draw=black!45,
    rounded corners=4pt,
    line width=0.9pt,
    fill=hamlight
  },
  hammath/.style={
    font=\small,
    align=center,
    inner sep=1.6pt,
    text height=1.8ex,
    text depth=.3ex
  },
  qfill/.style={fill=hamlightblue, draw=hamblue!70!black},
  pfill/.style={fill=hamlightorange, draw=hamorange!80!black},
  regfill/.style={fill=hamlightgreen, draw=hamgreen!60!black},
  warnfill/.style={fill=hamlightred, draw=hamred!65!black},
  goldfill/.style={fill=hamlightgold, draw=hamgold!70!black},
}

\usepackage[utf8]{inputenc} 
\usepackage[T1]{fontenc}    
\usepackage{hyperref}       
\usepackage{url}            
\usepackage{booktabs}       
\usepackage{amsfonts}       
\usepackage{nicefrac}       
\usepackage{microtype}      
\usepackage{xcolor}         

\title{Beyond Isotropy in JEPAs: Hamiltonian Geometry and Symplectic Prediction}

\author{%
  Robert Jenkinson Alvarez
}

\begin{document}

\maketitle

\begin{abstract}
JEPAs often regularize one-view embeddings toward an isotropic Gaussian, implicitly baking Euclidean symmetry into the representation. We show that this is not merely a benign default. For a known structured downstream geometry $H\succ0$, the minimax and maximum-entropy covariance under a Hamiltonian energy budget is $(c/d)H^{-1}$, and Euclidean isotropy incurs a closed-form price of isotropy. More importantly, when the downstream geometry is unknown, no geometry-independent fixed marginal target is canonical: every fixed covariance shape can be maximally misaligned for some structured geometry. We further show that even oracle one-view marginals do not identify the JEPA view-to-view predictive coupling. Within continuous linear phase-space flows, stationarity of the lifted oracle law, momentum reversal, and fixed-unit kinematics uniquely select Hamiltonian dynamics.
These results suggest that the structural bias in JEPAs should enter the cross-view coupling rather than a fixed encoder marginal. We instantiate this principle with \textbf{HamJEPA}, which encodes each view as a phase-space state $(q,p)$ and predicts view-to-view transitions with a learned Hamiltonian leapfrog map, while non-isotropic scale and spectral floors discourage collapse. In a deliberately headless token protocol, HamJEPA improves over SIGReg on CIFAR-100 by $+4.89$ kNN@20 and $+3.52$ linear-probe points at 30 epochs, and, over two training seeds at 80 epochs, by $+7.24$ kNN@20 and $+9.67$ linear-probe points relative to the full 2048-dimensional SIGReg representation. Matched non-symplectic and non-reversible controls retain the gain, so the experiments do not identify either property as individually necessary. On ImageNet-100, HamJEPA-$q$ improves by $+4.82$ kNN@20 and $+7.52$ linear-probe points at 45 epochs.
\end{abstract}

\section{Introduction}
\label{sec:Introduction}

Joint Embedding Predictive Architectures (JEPAs) learn representations by predicting one embedding from another.
For a fixed predictive loss, downstream performance depends on embedding geometry: scale, anisotropy, and collapse.
We study objectives of the form
$\min_\theta\ \mathcal{L}_{\text{JEPA}}(\theta) + \lambda\,\mathcal{L}_{\text{reg}}(Z_\theta)$,
where $\mathcal{L}_{\text{reg}}$ regularizes the embedding distribution induced by the encoder. 

A standard target is an \emph{isotropic Gaussian} embedding distribution.
LeJEPA motivates this choice under a ``no downstream structure'' viewpoint and enforces it with Sketched Isotropic Gaussian Regularization (SIGReg) \cite{lejepa}.
Isotropy is a symmetry assumption: it treats all embedding directions as exchangeable.
In many settings (patch-token grids, sequences, graphs), the index set carries known structure that induces preferred correlations and invariances; enforcing Euclidean isotropy can misalign the representation with this geometry.

This paper makes three linked contributions. First, for a fixed structured downstream geometry $H$, we derive the minimax and maximum-entropy covariance $(c/d)H^{-1}$ and the price paid by Euclidean isotropy. Second, we show that this oracle result should not be interpreted as a new universal marginal target: when $H$ is unknown, no geometry-independent fixed covariance shape is canonical, and even oracle marginals do not identify predictive coupling. Under an explicit stationary Gaussian lift, momentum reversibility and fixed-unit kinematics then characterize a unique linear Hamiltonian coupling. Third, we instantiate the resulting design principle with a phase-space JEPA (\emph{HamJEPA}) that places Hamiltonian structure in the cross-view predictor through a symplectic leapfrog map while using non-isotropic scale and spectral floors to discourage collapse. 

\noindent Code:
\href{https://github.com/R02AJ/HamJEPA-HamSIGReg}{GitHub: HamJEPA-HamSIGReg}.
Source provenance: Appendix~\ref{app:source_provenance}.

\subsection{Positioning relative to prior regularizers and dynamics priors}
\label{sec:positioning}

LeJEPA regularizes Euclidean embeddings toward an isotropic Gaussian using SIGReg, a sliced characteristic-function (CF) matching objective~\cite{lejepa}. KerJEPA broadens this marginal-discrepancy viewpoint with kernel-based discrepancies~\cite{kerjepa}. HamJEPA departs from this line by asking when Euclidean isotropy is misaligned with structured task geometry and by placing Hamiltonian structure in the predictive map rather than prescribing a Euclidean-isotropic encoder marginal.

Several SSL objectives prevent collapse through low-order or spectral statistics, including Barlow Twins, VICReg, MCR$^2$, and coding-rate-regularized self-distillation~\cite{Zbontar2021,Bardes2021,Yu2020MCR2,Wu2025SimDINO}. HamJEPA uses related scale and spectral non-degeneracy controls, but couples them to a phase-space JEPA with a symplectic predictor; we do not claim these anti-collapse penalties are novel in isolation.

Hamiltonian Neural Networks, Hamiltonian Generative Networks, and Symplectic ODE-Net impose energy-based or symplectic structure on learned dynamics, typically for temporal or control data~\cite{hnns,Toth2019,Zhong2019}. HamJEPA uses the same geometric bias in a different role: a predictor prior for paired augmentations, not a physics model of observed trajectories.

\section{Setup}
\label{sec:background}

\paragraph{Data, views, and randomness.}
Let $x \sim \mathcal{D}$ be a data sample and let $\tau \sim \mathcal{A}$ be a random augmentation.
A \emph{view} is $v := \tau(x)$.
We sample paired views of the same sample:
\begin{equation}
  v_a = \tau_a(x),\qquad v_b = \tau_b(x).
  \label{eq:views}
\end{equation}

We use two representation parameterizations:

\emph{(A) Euclidean embedding.}
An encoder $f_\theta:\mathcal{X}\to\mathbb{R}^d$ maps a view to an embedding random variable
\begin{equation}
  Z \;:=\; f_\theta(V)\in\mathbb{R}^d,
  \label{eq:Z-rv}
\end{equation}
where $V$ denotes a generic random view.
Assume finite second moments. Let $\mu_\theta:=\mathbb{E}[Z]$ and $\Sigma_\theta:=\mathrm{Cov}(Z)\succeq 0$.
We write centered embeddings as $\widetilde Z := Z-\mu_\theta$ so $\mathbb{E}[\widetilde Z]=0$ and $\mathrm{Cov}(\widetilde Z)=\Sigma_\theta$.

\emph{(B) Phase-space state (HamJEPA).}
A phase-space encoder $E_\theta:\mathcal{X}\to\mathbb{R}^{2d_0}$ outputs a state
\begin{equation}
  S := E_\theta(V) = (Q,P),\qquad Q,P\in\mathbb{R}^{d_0}.
  \label{eq:S-rv}
\end{equation}
We interpret $Q$ as a \emph{content} coordinate and $P$ as an auxiliary \emph{momentum} coordinate used to model structured view-to-view transitions.
We write projections $\pi_q(Q,P)=Q$ and $\pi_p(Q,P)=P$.

\paragraph{Batch (row) convention.}
For a minibatch of size $n$, we stack embeddings/states as row vectors:
$\mathbf{Z}\in\mathbb{R}^{n\times d}$ has rows $z_i^\top$, and
$\mathbf{S}\in\mathbb{R}^{n\times 2d_0}$ has rows $s_i^\top=(q_i^\top,p_i^\top)$.
We center batches by subtracting the row-mean, e.g.
$\widetilde{\mathbf{Z}}=\mathbf{Z}-\frac{1}{n}\mathbf{1}\mathbf{1}^\top\mathbf{Z}$.

\paragraph{Two Hamiltonians (to avoid notation collisions).}
Section~\ref{sec:price-of-isotropy-theory} uses an SPD operator $H\succ 0$ to specify a \emph{geometry} on embedding space.
HamJEPA uses a \emph{scalar energy} $\mathcal H_\phi(q,p)$ generating a symplectic map on phase space.
We reserve $H$ for the SPD operator and $\mathcal H_\phi$ for the scalar energy.

\section{A Theory of Marginal Geometry and Predictive Coupling in JEPAs}
\label{sec:price-of-isotropy-theory}

This section is pure theory: given a structured notion of task complexity and an energy budget, what covariance/distribution is canonical, and what is the cost of forcing Euclidean isotropy?
Throughout this section, fix an SPD operator $H\succ0$, which defines the structured embedding geometry through the quadratic energy $z^\top H z$; Hamiltonian-isotropic covariance in this geometry has the form $\alpha H^{-1}$.
We work with the centered embedding $\widetilde Z := Z-\mathbb E[Z]$ and write $\Sigma := \mathrm{Cov}(\widetilde Z)$.

\subsection{Setup: structured tasks and a Hamiltonian budget}
\label{sec:price-setup}
We model downstream \emph{linear} tasks by $y=w^\top \widetilde Z$ and constrain task complexity by the $H^{-1}$-unit ball
\begin{equation}
  \mathcal{W}_H := \{w\in\mathbb{R}^d:\ w^\top H^{-1}w \le 1\}.
  \label{eq:WH-theory}
\end{equation}
We pair this with a representation energy budget in the $H$-geometry:
\begin{equation}
  \mathbb{E}[\widetilde Z^\top H\widetilde Z] = \mathrm{tr}(H\Sigma) = c,
  \qquad c>0.
  \label{eq:H-budget-theory}
\end{equation}
Given $\Sigma\succeq 0$, define the worst-case task variance functional
\begin{equation}
  V(\Sigma) := \sup_{w\in\mathcal{W}_H} w^\top \Sigma w.
  \label{eq:Vdef-theory}
\end{equation}
This functional measures worst-case linear-task variance under the stated task class; the results below are not bounds on classification or prediction error.
The next lemma gives a convenient spectral form.

\begin{lemma}[Spectral form of $V(\Sigma)$]
\label{lem:spectralV}
For any $\Sigma\succeq 0$,
\begin{equation}
  V(\Sigma) = \lambda_{\max}\!\big(H^{1/2}\Sigma H^{1/2}\big).
  \label{eq:V-spectral-theory}
\end{equation}
Moreover, if $u_{\max}$ is a unit top eigenvector of $H^{1/2}\Sigma H^{1/2}$, then
$w_{\max}:=H^{1/2}u_{\max}$ attains the supremum in \eqref{eq:Vdef-theory}.
\end{lemma}
\noindent\emph{Proof.} See Appendix~\ref{app:proof-lem-spectralV}. \hfill $\square$

\subsection{Minimax covariance under a Hamiltonian budget}
\label{sec:price-minimax}
We ask for the covariance that minimizes the worst-case task variance \eqref{eq:Vdef-theory} under the $H$-budget \eqref{eq:H-budget-theory}.

\begin{theorem}[Minimax covariance under a Hamiltonian budget]
\label{thm:minimax}
Fix $H\succ 0$ and $c>0$. Consider
\begin{equation}
  \min_{\Sigma\succeq 0:\ \mathrm{tr}(H\Sigma)=c}\ V(\Sigma).
  \label{eq:minimax-prob-theory}
\end{equation}
The unique minimizer is
\begin{equation}
  \Sigma^\star = \frac{c}{d}\,H^{-1},
  \label{eq:sigma-star-theory}
\end{equation}
and the optimal value is
\begin{equation}
  \min_{\Sigma\succeq 0:\ \mathrm{tr}(H\Sigma)=c}\ V(\Sigma) = \frac{c}{d}.
  \label{eq:minimax-value-theory}
\end{equation}
Equivalently, $H^{1/2}\Sigma^\star H^{1/2}=\frac{c}{d}I$.
\end{theorem}
\noindent\emph{Proof.} See Appendix~\ref{app:proof-thm-minimax}. \hfill $\square$

\subsection{Maximum entropy: the canonical distribution}
\label{sec:price-maxent}
The minimax result selects a canonical covariance. A complementary principle selects a canonical \emph{distribution}.

\begin{theorem}[Maximum entropy under a Hamiltonian energy constraint]
\label{thm:maxent}
Let $Z$ be an $\mathbb{R}^d$-valued random vector with a density, $\mathbb{E}[Z]=0$, and finite second moments.
Among all such distributions satisfying
\begin{equation}
  \mathbb{E}[Z^\top H Z] = c,
  \label{eq:maxent-constraint-theory}
\end{equation}
the unique maximizer of differential entropy is the Gaussian
\begin{equation}
  Z \sim \mathcal{N}\!\Big(0,\frac{c}{d}H^{-1}\Big).
  \label{eq:maxent-sol-theory}
\end{equation}
\end{theorem}
\noindent\emph{Proof.} See Appendix~\ref{app:proof-thm-maxent}. \hfill $\square$

\subsection{Price of Euclidean isotropy}
\label{sec:price-isotropy}
Many practical objectives encourage Euclidean isotropy, i.e., $\Sigma\propto I$ in the original coordinates.
This subsection quantifies the worst-case gap incurred by imposing Euclidean isotropy under the same Hamiltonian budget.

Under the restriction $\Sigma=\sigma^2 I$, the constraint $\mathrm{tr}(H\Sigma)=c$ forces the unique feasible Euclidean-isotropic covariance
\begin{equation}
  \Sigma_{\mathrm{iso}} = \frac{c}{\mathrm{tr}(H)}\, I.
  \label{eq:sigma-iso-theory}
\end{equation}

\begin{theorem}[Price of Euclidean isotropy]
\label{thm:price}
Fix $H\succ 0$ and $c>0$. Let $\Sigma^\star$ be as in Theorem~\ref{thm:minimax} and let $\Sigma_{\mathrm{iso}}$ be as in \eqref{eq:sigma-iso-theory}. Then
\begin{align}
  V(\Sigma_{\mathrm{iso}})
  &= \frac{c}{\mathrm{tr}(H)}\,\lambda_{\max}(H),
  \label{eq:V-iso-theory}\\
  \frac{V(\Sigma_{\mathrm{iso}})}{V(\Sigma^\star)}
  &= \rho(H) := \frac{d\,\lambda_{\max}(H)}{\mathrm{tr}(H)}.
  \label{eq:rho-theory}
\end{align}
Moreover $1\le \rho(H)\le d$, with $\rho(H)=1$ if and only if $H\propto I$.
\end{theorem}
\noindent\emph{Proof.} See Appendix~\ref{app:proof-thm-price}. \hfill $\square$

\begin{proposition}[Metric isotropy equivalence]
\label{prop:metric-isotropy-equivalence}
Let $H\succ 0$ and let $Z\in\mathbb{R}^d$ be a random vector with finite second moments and covariance
$\Sigma := \mathrm{Cov}(Z)$.
Define the $H$-whitened representation $Z' := H^{1/2}Z$.
Then
\[
\mathrm{Cov}(Z') \;=\; H^{1/2}\Sigma H^{1/2}.
\]
In particular, for any $\alpha>0$,
\[
\Sigma = \alpha H^{-1}
\quad\Longleftrightarrow\quad
\mathrm{Cov}(Z') = \alpha I.
\]
\end{proposition}
\noindent\emph{Proof.} See Appendix~\ref{app:proof_metric-isotropy-equivalence}. \hfill $\square$

\subsection{A quadratic Hamiltonian lift of the oracle structured Gaussian}
\label{sec:quadratic_hamiltonian_lift}

The structured Gaussian selected by Theorems~\ref{thm:minimax} and~\ref{thm:maxent}
admits a natural lift to a Gibbs law on phase space.

\begin{proposition}[Quadratic Hamiltonian lift of the oracle covariance]
\label{prop:quadratic_hamiltonian_lift}
Fix \(H\succ 0\) and \(c>0\), and define the quadratic Hamiltonian
\begin{equation}
\mathcal H_H(q,p)
:=
\tfrac12 q^\top H q+\tfrac12 \|p\|_2^2,
\qquad q,p\in\mathbb R^d.
\label{eq:quadratic_hamiltonian_lift}
\end{equation}
Consider the Gibbs density on phase space
\begin{equation}
\pi_H(q,p)
\propto
\exp\!\Big(-\frac{d}{c}\,\mathcal H_H(q,p)\Big).
\label{eq:pi_H_lift}
\end{equation}
Then \(q\) and \(p\) are independent under \(\pi_H\), with
\begin{equation}
q\sim \mathcal N\!\Big(0,\frac{c}{d}H^{-1}\Big),
\qquad
p\sim \mathcal N\!\Big(0,\frac{c}{d}I\Big).
\label{eq:q_p_marginals_lift}
\end{equation}
Moreover,
\begin{equation}
\begin{split}
\mathbb E_{\pi_H}[q^\top H q]=c, \qquad \mathbb E_{\pi_H}[\|p\|_2^2]=c, \qquad \mathbb E_{\pi_H}[\mathcal H_H(q,p)] = c.
\end{split}
\label{eq:lift_expected_energy}
\end{equation}
In particular, the \(q\)-marginal of \(\pi_H\) is exactly the minimax-optimal and maximum-entropy law from
Theorems~\ref{thm:minimax} and~\ref{thm:maxent}.
\end{proposition}
\noindent\emph{Proof.} See Appendix~\ref{app:proof-prop-quadratic-hamiltonian-lift}. \hfill $\square$

\begin{proposition}[The oracle family spans the SPD cone]
\label{prop:oracle_family_spans_spd}
Fix \(c>0\). As \(H\) ranges over SPD operators, the family
\begin{equation}
\Sigma^\star(H)=\frac{c}{d}H^{-1}
\label{eq:oracle_family_spd}
\end{equation}
ranges over the entire SPD cone. Equivalently, for every \(\Sigma\succ 0\), choosing
\begin{equation}
H=\frac{c}{d}\Sigma^{-1}
\label{eq:H_from_Sigma}
\end{equation}
yields \(\Sigma^\star(H)=\Sigma\).
\end{proposition}
\noindent\emph{Proof.} See Appendix~\ref{app:proof-prop-oracle-family-spans-spd}. \hfill $\square$

\paragraph{From the oracle covariance to a practical unknown-\(H\) relaxation.}
Proposition~\ref{prop:quadratic_hamiltonian_lift} shows that when \(H\) is known, the oracle covariance
\(\Sigma^\star=\frac{c}{d}H^{-1}\) is exactly the \(q\)-marginal of a quadratic Hamiltonian Gibbs model on phase space.
In that quadratic oracle case, the induced dynamics are
\begin{equation}
\dot q = p,\qquad \dot p = -Hq,
\label{eq:oracle_quadratic_dynamics}
\end{equation}
so a Hamiltonian predictor is not an unrelated construction, but a natural dynamical realization of the same structured geometry.

This does \emph{not} mean that the practical training mechanism is uniquely determined by Section~\ref{sec:price-of-isotropy-theory}.
Rather, the theory should be read as an oracle statement: when \(H\) is known, the stated minimax and maximum-entropy problems select \(\Sigma^\star=(c/d)H^{-1}\).
When \(H\) is unknown, Proposition~\ref{prop:oracle_family_spans_spd} shows that the oracle family
\(\{\Sigma^\star(H):H\succ 0\}\) spans the entire SPD cone, so the theory no longer singles out one fixed covariance target.

\paragraph{Why not learn another fixed marginal target?}
The previous results might suggest replacing Euclidean isotropy by a better fixed
target covariance. The next result shows why this is brittle when downstream
geometry is unknown.

\begin{theorem}[No universal fixed marginal target]
\label{thm:no_universal_target}
Assume $d\ge2$. Fix any SPD target shape $M\succ0$ independent of the downstream geometry.
Under geometry $H\succ0$, the budget-feasible covariance with shape $M$ is
\[
  \Sigma_M(H) := \frac{c}{\mathrm{tr}(HM)}M .
\]
Its regret relative to the oracle covariance
$\Sigma^\star(H)=(c/d)H^{-1}$ is
\[
  \frac{V_H(\Sigma_M(H))}{V_H(\Sigma^\star(H))}
  =
  \frac{d\,\lambda_{\max}(H^{1/2}MH^{1/2})}{\mathrm{tr}(HM)} .
\]
This quantity equals $1$ iff $M\propto H^{-1}$, and for every fixed
$M\succ0$,
\[
  \sup_{H\succ0}
  \frac{V_H(\Sigma_M(H))}{V_H(\Sigma^\star(H))}
  = d .
\]
Thus no geometry-independent fixed covariance shape is canonical for all
structured downstream geometries.
\end{theorem}
\noindent\emph{Proof.} See Appendix~\ref{app:proof_no_universal_target}. \hfill $\square$

\begin{proposition}[Oracle marginals do not identify the predictive coupling]
\label{prop:marginals_do_not_identify_coupling}
Let $Z_a,Z_b\in\mathbb R^d$ be centered jointly Gaussian random vectors with
$\mathrm{Cov}(Z_a)=\mathrm{Cov}(Z_b)=\Sigma\succ0$ and cross-covariance
$\mathrm{Cov}(Z_b,Z_a)=C$. Then
\[
  \mathbb E[Z_b\mid Z_a] = C\Sigma^{-1}Z_a .
\]
Therefore, the same oracle marginal covariance $\Sigma$ admits a continuum of
different predictive maps as the admissible cross-covariance $C$ varies.
Marginal matching alone cannot determine the JEPA view-to-view prediction problem.
\end{proposition}
\noindent\emph{Proof.} See Appendix~\ref{app:proof_marginals_do_not_identify_coupling}. \hfill $\square$

\subsection{A canonical reversible coupling of the oracle law}
\label{sec:canonical_reversible_lift}

The preceding non-identification result leaves open which additional assumptions select a coupling.
We now specify a linear phase-space oracle; the assumptions below are additional modeling choices, rather than consequences of marginal matching.

\begin{theorem}[Canonical linear oracle lift]
\label{thm:canonical_linear_oracle_lift}
Fix \(H\succ0\), \(\alpha=c/d>0\), and
\[
\Gamma_H=\alpha
\begin{pmatrix}H^{-1}&0\\0&I\end{pmatrix},
\qquad
\mathsf R=\begin{pmatrix}I&0\\0&-I\end{pmatrix}.
\]
Let \(F_t\) be a continuous one-parameter group of linear maps on \(\mathbb R^{2d}\).
Suppose it preserves \(\mathcal N(0,\Gamma_H)\), satisfies momentum reversibility
\(\mathsf R F_t\mathsf R=F_{-t}\), and obeys the fixed-unit kinematics \(\dot Q=P\).
Then, uniquely,
\begin{equation}
F_t=e^{tL_H},
\qquad
L_H=\begin{pmatrix}0&I\\-H&0\end{pmatrix},
\qquad
\dot Q=P,\quad \dot P=-HQ.
\label{eq:canonical_linear_oracle_flow}
\end{equation}
\end{theorem}
\noindent\emph{Proof.} See Appendix~\ref{app:proof_canonical_linear_oracle_lift}. \hfill \(\square\)

Thus the quadratic Hamiltonian from Proposition~\ref{prop:quadratic_hamiltonian_lift} is characterized within this explicit oracle class: symplecticity follows from the assumptions.
Stationarity and momentum reversal provide a sufficient latent construction for equal, exchangeable observable view marginals; those observable properties alone do not imply the latent assumptions.

\begin{corollary}[Observable coupling and geometry]
\label{cor:oracle_observable_coupling}
Under the stationary law in Theorem~\ref{thm:canonical_linear_oracle_lift},
\begin{equation}
\begin{split}
Q_t\mid Q_0=q
&\sim\mathcal N\!\left(C_tq,\,
\alpha H^{-1/2}\sin^2(tH^{1/2})H^{-1/2}\right),\\
C_t&=\cos(tH^{1/2}),\qquad
M(t):=\operatorname{Cov}(Q_t,Q_0)
\operatorname{Cov}(Q_0)^{-1}=C_t,\\
H&=-M''(0).
\end{split}
\label{eq:oracle_observable_conditional}
\end{equation}
The pair \((Q_0,Q_t)\) is exchangeable.
\end{corollary}
\noindent\emph{Proof.} See Appendix~\ref{app:proof_oracle_observable_coupling}. \hfill \(\square\)

Invertible full-state dynamics therefore induce a generally stochastic observable transition after momentum is marginalized.
The curvature identity describes the oracle coupling near time zero; it does not establish recovery of an unknown downstream metric by a single-horizon training loss.

\begin{proposition}[Exact quadratic leapfrog coupling]
\label{prop:quadratic_leapfrog_oracle_coupling}
Let \(h>0\) satisfy \(h^2\lambda_{\max}(H)<4\), and define
\[
R_h=I-\tfrac{h^2}{4}H,\qquad
\Gamma_{H,h}=\alpha\operatorname{diag}(H^{-1},R_h).
\]
Quadratic leapfrog for \(\mathcal H_H\) exactly preserves
\(\mathcal N(0,\Gamma_{H,h})\).
Under this law its normalized \(K\)-step observable cross-covariance is
\begin{equation}
M_{K,h}
:=
\operatorname{Cov}(Q_K,Q_0)
\operatorname{Cov}(Q_0)^{-1}
=T_K\!\left(I-\tfrac{h^2}{2}H\right),
\label{eq:quadratic_leapfrog_chebyshev}
\end{equation}
where \(T_K\) is the Chebyshev polynomial of the first kind.
\end{proposition}
\noindent\emph{Proof.} See Appendix~\ref{app:proof_quadratic_leapfrog_oracle_coupling}. \hfill \(\square\)

This exact quadratic shadow law retains the Section~\ref{sec:price-of-isotropy-theory} content marginal and corrects the momentum covariance by \(R_h\); finite-step leapfrog does not preserve the original full oracle lift exactly.
The filter converges to \(\cos(tH^{1/2})\) when \(h\to0\) and \(Kh\to t\).
For a sparse grid or graph operator \(H\), its degree-\(K\) polynomial also gives finite-hop locality, as detailed in the proof.
These exact Gaussian and polynomial statements concern the quadratic oracle.
Section~\ref{sec:hamjepa} uses a learned nonlinear potential and encoder-side penalties without requiring the encoder law to equal either oracle law.
The characterization is structural, not a downstream performance ordering; Appendix~\ref{app:oracle_compatibility_vs_performance} gives an explicit oracle-incompatible control with identical Bayes performance on a task.

Together, Theorem~\ref{thm:no_universal_target} and
Proposition~\ref{prop:marginals_do_not_identify_coupling} show that the
unknown-geometry JEPA problem should not be reduced to choosing a single fixed
encoder marginal. This motivates specifying additional structural bias in the view-to-view
coupling. HamJEPA instantiates this principle by using a Hamiltonian phase-space
coupling: the quadratic known-$H$ case recovers the oracle $q$-marginal, while
the practical unknown-$H$ method uses a learned symplectic predictor and
geometry-agnostic non-collapse constraints.

\paragraph{Takeaway.}
Section~\ref{sec:price-of-isotropy-theory} does not merely say that Euclidean isotropy can be suboptimal. It shows that, under unknown downstream geometry, fixed marginal matching alone does not specify the predictive problem: no geometry-independent covariance target is universally canonical, and even an oracle marginal does not determine the JEPA predictive coupling. This motivates additional structural assumptions on the view-to-view map.

\section{HamJEPA: Symplectic Predictive Structure Without Isotropic Marginals}
\label{sec:hamjepa}

Section~\ref{sec:price-of-isotropy-theory} shows that fixed marginal matching is brittle under unknown structured geometry and that marginals alone do not determine JEPA prediction. HamJEPA therefore places the structural bias in the view-to-view coupling: a phase-space encoder is trained with a symplectic, volume-preserving predictor, while encoder-side scale and spectral penalties discourage collapse without enforcing isotropic Gaussian marginals.
A visual overview of HamJEPA is provided in Appendix~\ref{app:visual_hamjepa}.

\subsection{Phase-space encoder}
\label{sec:hamjepa-encoder}

Given paired views $v_a,v_b$ as defined in Section~\ref{sec:background}, the encoder $E_\theta:\mathcal X\to\mathbb R^{2d_0}$ outputs phase-space states
$s_a=E_\theta(v_a)$ and $s_b=E_\theta(v_b)$ with $s=(q,p)\in\mathbb{R}^{2d_0}$.

\begin{equation}
  s = E_\theta(v) = (q,p),\qquad q,p\in\mathbb R^{d_0}.
  \label{eq:hamjepa-state}
\end{equation}

We interpret $q$ as a \emph{content} coordinate used for prediction and evaluation, and $p$ as an auxiliary \emph{momentum} coordinate enabling structured transitions.
We write $\pi_q(q,p)=q$ and $\pi_p(q,p)=p$.

\paragraph{Observable coordinate vs.\ conjugate latent.}
Section~\ref{sec:price-of-isotropy-theory} is formulated in terms of the representation used by downstream tasks; in HamJEPA, that conceptual role is played by \(q\).
The auxiliary coordinate \(p\) is the conjugate latent that makes Hamiltonian transport possible.
For non-invertible augmentations, this makes \(q\)-matching the cleanest observable formulation, since the target need not contain a uniquely recoverable dynamical state.
In practice, however, the exact matching choice is a stability knob rather than a theorem:
on CIFAR-100 we primarily use \(q\)-matching with an auxiliary \(0.2p\) matching term, while on ImageNet-100 we use full-state \((q,p)\)-matching as a pre-evaluation stability choice that avoids an additional matching-mode and momentum-weight search.

\subsection{Symplectic predictor via a separable Hamiltonian and leapfrog}
\label{sec:hamjepa-predictor}

\paragraph{Canonical symplectic structure.}
Let
\begin{equation}
  J :=
  \begin{pmatrix}
    0 & I \\
    -I & 0
  \end{pmatrix}
  \in \mathbb R^{2d_0\times 2d_0}
  \label{eq:hamjepa-J}
\end{equation}
denote the canonical symplectic matrix.

\paragraph{Hamilton's equations and flow map.}
Given a differentiable scalar energy $\mathcal H_\phi:\mathbb R^{2d_0}\to\mathbb R$, Hamilton's ODE on phase space is
\begin{equation}
  \dot s(t) \;=\; J\,\nabla_s \mathcal H_\phi\big(s(t)\big),
  \qquad s(0)=s_0,
  \label{eq:hamjepa_hamilton_ode}
\end{equation}
where $J$ is the canonical symplectic matrix \eqref{eq:hamjepa-J}.
We denote its time-$t$ flow map, wherever defined, by $\Phi_{\phi,t}$, i.e.\ $s(t)=\Phi_{\phi,t}(s_0)$.

\paragraph{Separable energy (the theory-tight choice).}
To obtain a discrete predictor map with rigorous symplectic guarantees under an explicit integrator, we use a separable Hamiltonian energy
\begin{equation}
  \mathcal H_\phi(q,p) = T(p) + V_\phi(q),
  \qquad
  T(p)=\tfrac12\|p\|_2^2,
  \label{eq:hamjepa-separable}
\end{equation}
where $V_\phi:\mathbb R^{d_0}\to\mathbb R$ is a learnable potential.
Fixing $T(p)=\frac12\|p\|^2$ is stabilizing and prevents scale-control constraints from being satisfied by rescaling a learnable kinetic metric.

\paragraph{Leapfrog (velocity Verlet) predictor map.}
Given step size $\Delta t>0$ and $K\in\mathbb N$ steps, define the leapfrog update
\begin{equation}
\begin{aligned}
p_{k+\frac12} &= p_k - \tfrac{\Delta t}{2}\,\nabla_q V_\phi(q_k),\\
q_{k+1} &= q_k + \Delta t\, p_{k+\frac12},\\
p_{k+1} &= p_{k+\frac12} - \tfrac{\Delta t}{2}\,\nabla_q V_\phi(q_{k+1}),
\end{aligned}
\qquad k=0,\dots,K-1.
\label{eq:hamjepa-leapfrog}
\end{equation}
This defines a one-step map $\Psi_{\phi,\Delta t}:\mathbb R^{2d_0}\to\mathbb R^{2d_0}$ and a $K$-step predictor
\begin{equation}
  \Phi_\phi^{(K)} := \Psi_{\phi,\Delta t}^{\,K}.
  \label{eq:hamjepa-Kstep}
\end{equation}
Given $s_a=(q_a,p_a)$, we predict $\hat s_b=(\hat q_b,\hat p_b)$ by
\begin{equation}
  \hat s_b = \Phi_\phi^{(K)}(s_a).
  \label{eq:hamjepa-predict}
\end{equation}

\subsection{Objective: content prediction plus non-isotropic anti-collapse}
\label{sec:hamjepa-objective}

The projected log-det and participation-ratio penalties should be understood as an encoder-side anti-collapse bridge, not as the primary novelty of the method; similar covariance-volume objectives appear elsewhere, whereas our contribution is to couple such non-collapse controls with a phase-space JEPA and a symplectic predictive prior.

Let $s_a=(q_a,p_a)=E_\theta(v_a)$, $s_b=(q_b,p_b)=E_\theta(v_b)$, and $\hat s_b=(\hat q_b,\hat p_b)=\Phi_\phi^{(K)}(s_a)$.
HamJEPA optimizes
\begin{equation}
\begin{aligned}
\min_{\theta,\phi}\;\;
&\mathcal L_{\text{bi}}
+\lambda_{\text{budget}}\mathcal L_{\text{budget}}
+\lambda_{\text{var}}\mathcal L_{\text{var}}
+\lambda_{\log\det}\sum_{x\in\{q,p\}}
  \left(\mathcal L_{\text{vol}}^x+k^{-2}\mathcal L_{\text{pr}}^x+\mathcal L_{\text{eig}}^x\right)
+\lambda_{\text{mean}}\mathcal L_{\text{mean}}.
\end{aligned}
\label{eq:hamjepa-objective}
\end{equation}

All expectations are implemented as minibatch averages; $\mathrm{sg}(\cdot)$ denotes stop-gradient. MSE averages squared errors over batch and coordinates. The spectral terms are applied to both $q$ and $p$, with block-specific floors and ceilings; the variance floor is applied to both blocks on CIFAR-100 and to $q$ on ImageNet-100. The weights $(\lambda_{\text{budget}},\lambda_{\text{var}},\lambda_{\log\det},\lambda_{\text{mean}})$ are $(0.2,0.5,0.5,0.05)$ on CIFAR-100 and $(0.2,1,1,0.05)$ on ImageNet-100. Appendix~\ref{app:impl} specifies the variance and mean penalties and all thresholds. Energy consistency is disabled; weight decay is applied by AdamW.

\paragraph{(1) Matching loss.}
Let $m\in\{q,qp\}$ denote the matching mode, with
$\Pi_q(q,p)=q$ and $\Pi_{qp}(q,p)=(q,p)$.
CIFAR-100 uses $m=q$ with auxiliary momentum weight $w_p=0.2$, while ImageNet-100 uses $m=qp$ for stability.
Let $\ell_q(s,t)=\mathrm{MSE}(\pi_q(s),\pi_q(t))+w_p\mathrm{MSE}(\pi_p(s),\pi_p(t))$ and $\ell_{qp}(s,t)=\mathrm{MSE}(s,t)$.
We define
\begin{equation}
  \mathcal L_{\mathrm{pred}}
  :=
  \mathbb E\!\left[
    \ell_m(\hat s_b,\operatorname{sg}(s_b))
  \right].
  \label{eq:pred-loss}
\end{equation}

\paragraph{(2) Bidirectional prediction.}
Define the backward prediction by running the \emph{same} integrator with step size $-\Delta t$ on CIFAR-100:
\[
  \hat s_a = \Phi_\phi^{(K)}(s_b;\,-\Delta t).
\]
The ImageNet-100 implementation instead uses $+\Delta t$ in both view orders. With $\eta=-1$ on CIFAR-100 and $\eta=+1$ on ImageNet-100, the bidirectional loss is
\begin{equation}
\begin{aligned}
\mathcal L_{\text{bi}}
:= \tfrac12\mathbb E\Big[
&\ell_m\!\left(\Phi_\phi^{(K)}(s_a;\,+\Delta t),\mathrm{sg}(s_b)\right)
+ \ell_m\!\left(\Phi_\phi^{(K)}(s_b;\,\eta\Delta t),\mathrm{sg}(s_a)\right)
\Big].
\end{aligned}
\label{eq:hamjepa-Lbi}
\end{equation}
Time reversibility of leapfrog (Corollary~\ref{cor:hamjepa-reversible}) makes the CIFAR-100 return map an ``exact inverse in-model'' constraint rather than an independent predictor. The ImageNet-100 objective averages two forward-time predictions and does not impose that negative-time inverse constraint.

\paragraph{(3) Fixed-units scale budget (anti-gauge).}
Symplecticity constrains the predictor map but does not prevent the encoder from collapsing all views to a point.
We therefore penalize deviations from scale budgets in fixed units (no Gaussianity, no isotropy):
\begin{equation}
\mathcal L_{\text{budget}}
:=
\Big(\frac{\mathbb E\|q\|_2^2}{d_0}-\alpha_q\Big)^2
+
\Big(\frac{\mathbb E\|p\|_2^2}{d_0}-\alpha_p\Big)^2,
\alpha_q,\alpha_p>0.
\label{eq:hamjepa-Lbudget}
\end{equation}

The spectral terms apply projected log-det and participation-ratio floors and a top-eigenvalue-fraction ceiling to both $q$ and $p$. Here $k$ is the projection dimension; Appendix~\ref{app:spectral_objective} gives their exact definitions and normalization.
Together, these terms discourage collapse while permitting anisotropy; \emph{HamJEPA} does not constrain representations to a unit sphere or isotropic Gaussian target.

\subsection{Structural guarantees}
\label{sec:hamjepa_theory}

The formal structural theory of HamJEPA is deferred to
Appendix~\ref{app:hamjepa_theory}. At a high level, the separable Hamiltonian
\eqref{eq:hamjepa-separable} and leapfrog update \eqref{eq:hamjepa-leapfrog} imply that the
$K$-step predictor $\Phi^{(K)}_\phi$ is symplectic:
\begin{equation}
  \big(D\Phi^{(K)}_\phi(s)\big)^\top J
  \big(D\Phi^{(K)}_\phi(s)\big)=J,
  \qquad
  \det D\Phi^{(K)}_\phi(s)=1 .
  \label{eq:main-symplectic-guarantee}
\end{equation}
Thus the predictor is locally volume-preserving and is inverted by running
leapfrog with step size $-\Delta t$. Since symplecticity constrains only the
predictor, encoder collapse is controlled separately by the fixed-unit scale,
projected log-det, and participation-ratio floors. Appendix~\ref{app:hamjepa_theory}
gives reciprocal singular-value pairing, shadow-Hamiltonian stability, and the
anti-collapse lemmas.

\section{Results}
\label{sec:results}

\paragraph{Controlled protocol.}
We evaluate frozen representations on CIFAR-100 and ImageNet-100 using cosine
$k$NN and linear probing. The probe uses LayerNorm followed by a linear classifier on frozen features. The goal is not a projector-heavy SOTA recipe, but an
objective-level comparison: all methods share the same backbone, augmentations,
optimizer, batch size, schedule, two global views, no local views, the encoder runs in a headless token regime using ResNet stage-3 (\texttt{layer3}) features, and identity projector. For HamJEPA, the encoder output
$s=(q,p)$ is split channel-wise; $q$ is the primary readout, while $p$ and
$(q,p)$ test auxiliary-state and dimension effects. Full details and diagnostics
are in Appendix~\ref{app:res}--\ref{app:impl}.
SIGReg is the isotropic-Gaussian JEPA baseline closest to our theoretical comparison. This protocol is intentionally underpowered by modern SSL standards. By removing projection heads, local crops, EMA teachers, and recipe-specific tuning, the comparison evaluates the complete training objectives under a shared minimal protocol.

\subsection{CIFAR-100: controlled gains and predictor ablation}
\label{sec:results_cifar}

HamJEPA improves over SIGReg after only $30$ epochs under the same backbone, views, optimizer, and identity projector.
Using the primary $q$ readout, HamJEPA gains $+4.89$ points in $k$NN@20 (31.45 vs.\ 26.56) and $+3.52$ points in linear probing (33.95 vs.\ 30.43).
The $p$ readout is also predictive (29.71 $k$NN@20; 33.07 linear), but remains behind $q$ in nearest-neighbor accuracy, matching the intended role of $q$ as the content coordinate and $p$ as an auxiliary dynamical coordinate.
Concatenating $(q,p)$ gives the strongest linear probe at $30$ epochs (34.18), while $q$ alone gives the strongest $k$NN geometry.

The original $80$-epoch predictor-family ablation gives a single-run comparison, retained in Appendix~\ref{app:diag_discussion}.

\begin{table}[t]
\centering
\small
\begin{tabular}{llccc}
\toprule
Setting & Method / readout & $k$NN@20 (\%) & Linear (\%) & $k$NN best (\%) \\
\midrule
\multicolumn{5}{l}{\emph{30 epochs, mean$\pm$std over 3 seeds}} \\
& SIGReg & 26.56 $\pm$ 0.18 & 30.43 $\pm$ 0.30 & -- \\
& HamJEPA-$q$ & \textbf{31.45 $\pm$ 0.27} & 33.95 $\pm$ 0.24 & -- \\
& HamJEPA-$p$ & 29.71 $\pm$ 0.25 & 33.07 $\pm$ 0.36 & -- \\
& HamJEPA-$(q,p)$ & 30.88 $\pm$ 0.39 & \textbf{34.18 $\pm$ 0.17} & -- \\
\bottomrule
\end{tabular}
\caption{\textbf{CIFAR-100 controlled results.}
The 30-epoch block shows short-budget performance. The original single-run 80-epoch results are retained in Table~\ref{tab:cifar80_predictor_ablation}.}
\label{tab:cifar_results}
\end{table}

\paragraph{Replicated comparison and matched controls.}
Table~\ref{tab:cifar80_replicated} reports two independent 80-epoch training seeds, with five probe optimizations per frozen representation. SIGReg-$q$ is its 1024-dimensional first block; SIGReg-full is the 2048-dimensional baseline used in the original table. HamJEPA retains its gain after centering. A parameter-identical, reversible, volume-preserving non-symplectic force-rotation control (VPR-NS), and its non-reversible second-order splitting control, perform comparably. These results do not identify symplecticity or reversibility as individually necessary. Full metric sweeps, uncertainty, dynamics, sensitivity, and cost are in Appendix~\ref{app:replicated_controls}.

\begin{table}[t]
\centering
\small
\begin{tabular}{lrccc}
\toprule
Method / readout & Dim. & \shortstack{Raw\\$k$NN@20} & \shortstack{Centered\\$k$NN@20} & Linear \\
\midrule
SIGReg-$q$ & 1024 & 25.61 & 25.62 & 32.22 \\
SIGReg-full & 2048 & 27.50 & 27.53 & 34.21 \\
MLP-HJEPA-$q$ & 1024 & 27.50 & 27.48 & 41.47 \\
VPR-NS-$q$ & 1024 & 34.59 & 34.33 & 44.18 \\
Non-reversible VPR-NS-$q$ & 1024 & 34.93 & 34.89 & 44.22 \\
HamJEPA-$q$ & 1024 & 34.74 & 34.59 & 43.88 \\
\bottomrule
\end{tabular}
\caption{\textbf{CIFAR-100 replicated results, 80 epochs.} Mean accuracy (\%) over two training seeds. Each linear result averages five probes per training seed. Sample standard deviations are in Appendix~\ref{app:replicated_controls}.}
\label{tab:cifar80_replicated}
\end{table}

\subsection{ImageNet-100: same minimal recipe, larger-scale signal}
\label{sec:results_imagenet}

On ImageNet-100, HamJEPA again improves over SIGReg under the same minimal recipe ($45$ epochs, single seed).
Using the primary $q$ readout, HamJEPA gains $+4.82$ points in $k$NN@20 (24.92 vs.\ 20.10) and $+7.52$ points in linear probing (31.92 vs.\ 24.40).
The best HamJEPA linear readout is $p$ at 32.08, giving a $+7.68$ point gain over SIGReg-$q$.
These gains are not explained by feature dimension.
SIGReg benefits modestly from concatenating its two 1024-d blocks: $q\rightarrow(q,p)$ improves by $+0.84$ $k$NN@20 points (20.10 $\rightarrow$ 20.94) and $+1.14$ linear-probe points (24.40 $\rightarrow$ 25.54).
In contrast, HamJEPA gains essentially nothing from concatenation in $k$NN (24.92 $\rightarrow$ 24.64) and only $+0.12$ in linear probing (31.92 $\rightarrow$ 32.04).
Crucially, the 1024-d HamJEPA-$q$ readout already outperforms the 2048-d SIGReg-$(q,p)$ readout by $+3.98$ points in $k$NN@20 and $+6.38$ points in linear probing.

We match the full phase space $(q,p)$ during pretraining for stability.
Consistent with this objective, $p$ is highly predictive and slightly exceeds $q$ on the linear probe (32.08 vs.\ 31.92), while remaining comparable in $k$NN (24.72 vs.\ 24.92).
However, concatenating $(q,p)$ does not improve $k$NN, suggesting that the nearest-neighbor geometry is not simply a consequence of evaluating a larger vector.
The result is therefore a controlled objective-level delta: HamJEPA improves frozen-feature geometry in a deliberately headless setting where neither projection-head capacity nor representation dimension explains the gain.

\begin{table}[t]
\centering
\small
\begin{tabular}{lcc}
\toprule
Method & $k$NN@20 (\%) & Linear probe (\%) \\
\midrule
SIGReg ($q$-eval) & 20.10 & 24.40 \\
SIGReg ($p$-eval) & 18.78 & 23.86 \\
SIGReg ($(q,p)$-eval) & 20.94 & 25.54 \\
\midrule
HamJEPA ($q$-eval) & \textbf{24.92} & 31.92 \\
HamJEPA ($p$-eval) & 24.72 & \textbf{32.08} \\
HamJEPA ($(q,p)$-eval) & 24.64 & 32.04 \\
\bottomrule
\end{tabular}
\caption{\textbf{ImageNet-100 downstream performance after $45$ epochs} of pretraining, single seed.
For HamJEPA, $q$ and $p$ are learned phase-space coordinates.
For SIGReg, the same split is a representation-block / dimension control rather than a literal position--momentum decomposition.}
\label{tab:main_imagenet}
\end{table}

\paragraph{Geometry diagnostics.}
The accuracy gains are accompanied by broader, non-collapsed covariance spectra
without enforcing an isotropic Gaussian marginal. On ImageNet-100, SIGReg-$q$
has effective rank $\approx38.33$ and PR $\approx24.91$, while HamJEPA-$q$
reaches $\approx94.73$ and $\approx45.27$; even the larger SIGReg-$(q,p)$
representation remains lower ($\approx51.16$ and $\approx36.17$).
CIFAR-100 shows the same qualitative pattern, with HamJEPA-$q$ increasing
effective rank from $\approx109.7$ to $\approx124.5$. Full spectra and
training-time diagnostics are deferred to Appendix~\ref{app:res}.

\section{Conclusion}
\label{sec:conclusion}
We showed that Euclidean-isotropic JEPA regularization can be misaligned with structured task geometry, and that fixed marginal targets are brittle when the downstream geometry is unknown. Under a specified stationary Gaussian lift, reversibility and fixed-unit kinematics characterize the linear Hamiltonian coupling; its quadratic leapfrog discretization has an exact corrected invariant law. HamJEPA instead places Hamiltonian structure in the view-to-view predictor: views are encoded as phase-space states, predicted with a symplectic leapfrog map, and regularized by non-isotropic scale and spectral floors. Under a deliberately headless token protocol, HamJEPA improves frozen-feature $k$NN and linear-probe performance over SIGReg on CIFAR-100 and ImageNet-100. Matched non-symplectic and non-reversible controls retain the gain, so these experiments do not identify either property as individually necessary.

\paragraph{Limitations.}
This is a controlled geometry study rather than a SOTA SSL recipe. The main
baseline is SIGReg, ImageNet-100 is single-seed, and we do not learn or recover
the downstream operator $H$; broader baselines, multi-seed ImageNet runs, and
learned structured geometries remain future work. The new CIFAR-100 controls use two independent training seeds; learned-marginal HamSIGReg and standard projector/EMA/multicrop recipes have not been evaluated.

\clearpage
\begin{ack}
The author has no funding or competing financial interests to disclose.
\end{ack}

\nocite{*}

\bibliographystyle{plainnat}
\bibliography{example_paper}

\newpage
\appendix
\onecolumn

\section{Societal impact.}
\label{app:soc}
HamJEPA is a foundational representation-learning method evaluated on standard image benchmarks, not a deployed decision system. Potential positive impacts include more controlled and sample-efficient visual representation learning. Potential negative impacts are generic to stronger visual representations: downstream systems could be used for surveillance, biometric identification, or biased high-stakes classification if deployed without appropriate safeguards. We recommend dataset and model documentation, license compliance, privacy review, and fairness evaluation before any deployment in sensitive domains.

\section{Theory: symplecticity, volume preservation, and anti-collapse}
\label{app:hamjepa_theory}

This appendix records the structural guarantees enjoyed by HamJEPA when using the separable Hamiltonian
\eqref{eq:hamjepa-separable} and the leapfrog integrator \eqref{eq:hamjepa-leapfrog}.
These results justify the summary in
Section~\ref{sec:hamjepa_theory}: leapfrog gives a symplectic, volume-preserving,
time-reversible predictor, while satisfied encoder-side scale, log-det, and
participation-ratio constraints give conditional projected-covariance bounds without imposing Euclidean
isotropic Gaussian marginals.
All proofs are in Appendix~\ref{app:hamjepa_proofs}.

\subsection{Statements}
\label{app:theory_hamjepa_statements}

\paragraph{Row/column convention (linear transforms).}
The paper uses column vectors for single states/embeddings and row-stacked matrices for minibatches.
The following lemma formalizes the equivalence.

\begin{lemma}[Row/column convention for left-multiplication by $H^{1/2}$]
\label{lem:rowcol_transform}
Let $H\in\mathbb R^{d\times d}$ be symmetric positive definite and let $z\in\mathbb R^d$ be a column vector.
Define $z' := H^{1/2}z$.
Let $\mathbf Z\in\mathbb R^{n\times d}$ be a row-stacked minibatch with rows $\mathbf Z_{i,:}=z_i^\top$.
Define $\mathbf Z' := \mathbf ZH^{1/2}$.
Then for every $i$,
\[
(\mathbf Z')_{i,:}^\top \;=\; H^{1/2}z_i.
\]
\end{lemma}
\noindent\emph{Proof.} See Appendix~\ref{app:proof-lem-rowcol-transform}. \hfill $\square$

\paragraph{Symplecticity of the continuous flow.}
We first state the standard symplecticity property of Hamiltonian flows.

\begin{theorem}[Symplecticity of Hamiltonian flows]
\label{thm:flow_symplectic}
Assume $\mathcal H_\phi\in C^2(\mathbb R^{2d_0})$ and let $\Phi_{\phi,t}$ denote the flow map of Hamilton's ODE
\eqref{eq:hamjepa_hamilton_ode}.
Then $\Phi_{\phi,t}$ is symplectic: for all $s$ and $t$ for which the flow is defined,
\begin{equation}
\big(D\Phi_{\phi,t}(s)\big)^\top J\,\big(D\Phi_{\phi,t}(s)\big) \;=\; J.
\label{eq:flow-symplectic}
\end{equation}
Consequently, $\det(D\Phi_{\phi,t}(s))=1$ for all such $s,t$, and $\Phi_{\phi,t}$ preserves phase-space volume (Liouville).
\end{theorem}
\noindent\emph{Proof.} See Appendix~\ref{app:proof-thm-flow-symplectic}. \hfill $\square$

\begin{corollary}[Volume-preserving maps preserve differential entropy]
\label{cor:entropy_preservation}
Let $F:\mathbb R^{m}\to\mathbb R^{m}$ be a $C^1$ diffeomorphism and let $X$ admit a density.
Assume \(h(X)\) is finite and \(\mathbb E|\log|\det DF(X)||<\infty\).
Then
\[
h(F(X)) \;=\; h(X) + \mathbb E\big[\log|\det(DF(X))|\big].
\]
In particular, if $|\det(DF(x))|=1$ for all $x$ (e.g., $F=\Phi_{\phi,t}$ from Theorem~\ref{thm:flow_symplectic}),
then $h(F(X))=h(X)$.
\end{corollary}
\noindent\emph{Proof.} See Appendix~\ref{app:proof-cor-entropy-preservation}. \hfill $\square$

\paragraph{Symplecticity of the discrete predictor.}
HamJEPA uses the leapfrog/Verlet map \eqref{eq:hamjepa-leapfrog} as a discrete predictor.
For separable Hamiltonians this map is symplectic.

\begin{theorem}[Leapfrog is symplectic and volume-preserving for separable Hamiltonians]
\label{thm:leapfrog_symplectic}
Let $\mathcal H_\phi(q,p)=T(p)+V_\phi(q)$ with $T(p)=\tfrac12\|p\|_2^2$ and $V_\phi\in C^2$.
Then the one-step leapfrog map $\Psi_{\phi,\Delta t}$ defined by \eqref{eq:hamjepa-leapfrog} is symplectic.
Therefore, for any $K\ge 1$, the $K$-step predictor $\Phi_\phi^{(K)}=\Psi_{\phi,\Delta t}^{\,K}$ is symplectic and
satisfies
\begin{equation}
\det\big(D\Phi_\phi^{(K)}(s)\big)=1
\quad\text{for all } s.
\label{eq:det-one}
\end{equation}
\end{theorem}
\noindent\emph{Proof.} See Appendix~\ref{app:proof-thm-leapfrog-symplectic}. \hfill $\square$

\begin{corollary}[Invertibility and time reversibility]
\label{cor:reversible}\label{cor:hamjepa-reversible}
Under the assumptions of Theorem~\ref{thm:leapfrog_symplectic},
the leapfrog map is bijective and satisfies
\begin{equation}
\Psi_{\phi,\Delta t}^{-1} = \Psi_{\phi,-\Delta t}.
\label{eq:reversibility}
\end{equation}
In particular, the $K$-step predictor admits an exact inverse by running the same integrator with step size $-\Delta t$:
$(\Phi_\phi^{(K)}(\cdot;\,+\Delta t))^{-1}=\Phi_\phi^{(K)}(\cdot;\,-\Delta t)$.
\end{corollary}
\noindent\emph{Proof.} See Appendix~\ref{app:proof-cor-reversible}. \hfill $\square$

\begin{proposition}[Approximation of separable Hamiltonian flows by HamJEPA predictors]
\label{prop:hamjepa_expressivity}
Fix a horizon $T>0$ and a compact set $\mathcal K\subset\mathbb{R}^{2d_0}$.
Let $\mathcal K_q := \{q\in\mathbb{R}^{d_0} : \exists p \text{ with } (q,p)\in\mathcal K\}$.

Let $V:\mathbb{R}^{d_0}\to\mathbb{R}$ be $C^3$ on an open neighborhood of a compact set $\Omega_q\supseteq\mathcal K_q$,
and assume $\nabla V$ is Lipschitz on $\Omega_q$.
Let $\Phi^{V}_t$ denote the time-$t$ flow map of Hamilton's ODE
\[
\dot q = p,\qquad \dot p = -\nabla V(q)
\]
(i.e., the separable Hamiltonian $H(q,p)=\tfrac12\|p\|_2^2 + V(q)$).

For a learnable potential $V_\phi\in C^3$ with $\nabla V_\phi$ Lipschitz on $\Omega_q$, let $\Psi_{\phi,\Delta t}$
be one leapfrog step \eqref{eq:hamjepa-leapfrog} and let $\Phi^{(K)}_\phi := \Psi_{\phi,\Delta t}^K$ with $K\Delta t = T$ \eqref{eq:hamjepa-Kstep}.
Assume that for all initial states in $\mathcal K$, both the exact trajectory under $V$ and the exact trajectory under $V_\phi$
remain inside $\Omega := \Omega_q\times \Omega_p$ for all $t\in[0,T]$, where \(\Omega_p\) is compact.
For all sufficiently small \(\Delta t=T/K\), assume also that the leapfrog iterates, their split stages, and the exact \(V_\phi\)-flow segments of length at most \(\Delta t\) starting at those iterates remain in \(\Omega\).

Then there exists a constant $C=C(T,\Omega,V_\phi)>0$ such that
\[
\sup_{(q,p)\in\mathcal K}\bigl\|\Phi^{(K)}_\phi(q,p)-\Phi^{V}_T(q,p)\bigr\|_2
\;\le\;
\]\[
C\Big(
\underbrace{\sup_{q\in\Omega_q}\|\nabla V_\phi(q)-\nabla V(q)\|_2}_{\text{modeling error}}
\;+\;
\underbrace{\Delta t^2}_{\text{integrator error}}
\Big).
\]
In particular, if the potential class can approximate $\nabla V$ uniformly on $\Omega_q$
(and admits bounded derivatives on $\Omega_q$, with a common Lipschitz bound for \(\nabla V_\phi\) and the stated trajectory containment along the approximating family), then for every $\varepsilon>0$ there exist $\phi$ and $\Delta t>0$
(hence $K=T/\Delta t$) such that
\[
\sup_{(q,p)\in\mathcal K}\bigl\|\Phi^{(K)}_\phi(q,p)-\Phi^{V}_T(q,p)\bigr\|_2 \le \varepsilon.
\]
\end{proposition}
\noindent\emph{Proof.} See Appendix~\ref{app:proof_hamjepa_expressivity}. \hfill $\square$

\begin{lemma}[Reciprocal singular-value pairing]
\label{lem:reciprocal_svs}
Let $A\in\mathbb R^{2d_0\times 2d_0}$ be symplectic: $A^\top J A = J$.
Then the singular values of $A$ occur in reciprocal pairs: if $\sigma$ is a singular value, so is $1/\sigma$
(with the same multiplicity).
Equivalently, $A$ cannot be contractive in all directions; any local contraction must be compensated by expansion.
\end{lemma}
\noindent\emph{Proof.} See Appendix~\ref{app:proof-lem-reciprocal-svs}. \hfill $\square$

\paragraph{Long-horizon stability (shadow Hamiltonians).}
A core practical advantage of symplectic integrators is the existence of a nearby modified energy that explains stable
multi-step rollouts. We record a standard backward-error-analysis statement.

\begin{theorem}[Shadow Hamiltonian / backward error analysis (standard)]
\label{thm:shadow}
Assume \(V_\phi\) is analytic on a neighborhood of a fixed compact set containing all numerical iterates and modified-flow segments under consideration.
For each fixed integer \(N\ge1\) and sufficiently small \(\Delta t>0\), there exists a modified Hamiltonian $\widetilde{\mathcal H}^{[N]}_{\phi,\Delta t}$ such that:
\begin{enumerate}
\item (\textbf{Closeness}) $\widetilde{\mathcal H}^{[N]}_{\phi,\Delta t}(s)=\mathcal H_\phi(s)+\mathcal O(\Delta t^2)$ uniformly on the region of interest.
\item (\textbf{Modified-flow representation}) One leapfrog step equals the time-$\Delta t$ flow of
$\widetilde{\mathcal H}^{[N]}_{\phi,\Delta t}$ up to local error $\mathcal O(\Delta t^{N+1})$.
\item (\textbf{Near-conservation}) Along the leapfrog iterates $(s_k)$, with \(T_n=n\Delta t\),
\[
|\mathcal H_\phi(s_n)-\mathcal H_\phi(s_0)|
\le C_0\Delta t^2+C_1T_n\Delta t^N,
\]
while the stated domain conditions hold.
For fixed \(N>2\), this is \(\mathcal O(\Delta t^2)\) on horizons \(T_n=\mathcal O(\Delta t^{2-N})\).
The constants may depend on the truncation order and controlled region.
\end{enumerate}
\end{theorem}
\noindent\emph{Proof (sketch via BCH/backward error analysis).} See Appendix~\ref{app:proof-thm-shadow}. \hfill $\square$

\paragraph{Anti-collapse: what each ingredient does (and does not do).}
HamJEPA uses an \emph{encoder-side} non-degeneracy bridge that avoids isotropic Gaussianity.
It is useful to separate (i) scale control, (ii) rank/volume control, and (iii) spike control.

\begin{lemma}[Scale budgets do not prevent rank collapse]
\label{lem:budget_not_rank}
Fix $c>0$ and consider PSD matrices $\Sigma\succeq 0$ with $\mathrm{tr}(\Sigma)=c$.
There exist such $\Sigma$ with rank $1$ (hence complete collapse except for one direction), e.g.
$\Sigma=\mathrm{diag}(c,0,\dots,0)$.
\end{lemma}
\noindent\emph{Proof.} See Appendix~\ref{app:proof-lem-budget-not-rank}. \hfill $\square$

\begin{lemma}[Log-det alone does not prevent ``one-spike'' spectra]
\label{lem:logdet_spike}
Fix $k\ge 2$ and $\tau\in\mathbb R$.
There exist positive definite $\Sigma\in\mathbb R^{k\times k}$ with $\frac{1}{k}\log\det(\Sigma)=\tau$
but with $\mathrm{PR}(\Sigma)$ arbitrarily close to $1$ (spike regime).
\end{lemma}
\noindent\emph{Proof.} Appendix~\ref{app:proof-lem-logdet-spike}. \hfill $\square$

\begin{lemma}[Participation ratio alone does not fix scale]
\label{lem:pr_not_scale}
For any PSD $\Sigma\succeq 0$ and any scalar $a>0$,
$\mathrm{PR}(a\Sigma)=\mathrm{PR}(\Sigma)$.
Thus a PR constraint alone does not prevent collapse by uniformly shrinking all eigenvalues.
\end{lemma}
\noindent\emph{Proof.} See Appendix~\ref{app:proof-lem-pr-not-scale}. \hfill $\square$

\begin{lemma}[Projected log-det detects full-space rank collapse]
\label{lem:proj_logdet_rank_certificate}
Let $Q\in\mathbb{R}^{d_0}$ be a zero-mean random vector with covariance $\Sigma_Q\succeq 0$ and let
$R\in\mathbb{R}^{d_0\times k}$ have full column rank $k$ (e.g., orthonormal columns as in Appendix~\ref{app:spectral_objective}).
Define the population projected covariance
\[
\Sigma_{Y,0} := R^\top \Sigma_Q R \in \mathbb{R}^{k\times k}.
\]
Then
\[
\rank(\Sigma_{Y,0}) \le \rank(\Sigma_Q).
\]
In particular, if $\rank(\Sigma_Q)<k$ then $\Sigma_{Y,0}$ is singular and therefore
$\det(\Sigma_{Y,0})=0$ and $\log\det(\Sigma_{Y,0})=-\infty$.

Moreover, if one uses the ridge-regularized version from \eqref{eq:hamjepa-SigmaY},
\[
\Sigma_Y := \Sigma_{Y,0} + \varepsilon I_k,\qquad \varepsilon>0,
\]
and if $r:=\rank(\Sigma_{Y,0})$, then $\Sigma_Y$ has exactly $k-r$ eigenvalues equal to $\varepsilon$ and hence
\[
\log\det(\Sigma_Y) = (k-r)\log\varepsilon + \sum_{i=1}^{r}\log(\lambda_i+\varepsilon),
\]
where $\lambda_1,\dots,\lambda_r>0$ are the nonzero eigenvalues of $\Sigma_{Y,0}$.
\end{lemma}
\noindent\emph{Proof.} See Appendix~\ref{app:proof_proj_logdet_rank_certificate}. \hfill $\square$

\paragraph{Projected log-det floor.}
Recall the projected covariance $\Sigma_Y$ from \eqref{eq:hamjepa-SigmaY}.
The following proposition quantifies how a log-det floor plus a trace/scale control lower-bounds the smallest eigenvalue.

\begin{proposition}[Projected log-det floor prevents rank collapse]
\label{prop:logdet_nocollapse}
Let \(k\ge2\) and let $\Sigma_Y\in\mathbb R^{k\times k}$ be SPD with eigenvalues $\lambda_1\ge\cdots\ge\lambda_k>0$.
Suppose $\log\det(\Sigma_Y)\ge k\tau$ and $\mathrm{tr}(\Sigma_Y)\le m$ for some $m>0$.
Then
\begin{equation}
\lambda_{\min}(\Sigma_Y)\;\ge\; \exp(k\tau)\,\frac{(k-1)^{k-1}}{m^{k-1}}.
\label{eq:logdet-lmin}
\end{equation}
If \(\Sigma_Y=\Sigma_{Y,0}+\varepsilon I_k\), the corresponding unregularized bound is
\[
\lambda_{\min}(\Sigma_{Y,0})
\ge
\exp(k\tau)\frac{(k-1)^{k-1}}{m^{k-1}}-\varepsilon.
\]
It certifies rank \(k\) of the projected unregularized covariance only when the right-hand side is positive.
When \(k<d_0\), this does not certify full ambient rank \(d_0\).
\end{proposition}
\noindent\emph{Proof.} See Appendix~\ref{app:proof-prop-logdet-nocollapse}. \hfill $\square$

\paragraph{Joint spectral non-degeneracy (budget + log-det + PR).}
The next theorem formalizes the ``LogDet + PR + budget'' bridge: it rules out both needle collapse and
the one-spike ``L-shaped'' failure mode.

\begin{theorem}[Spectral non-degeneracy via joint constraints]
\label{thm:spectral_nondeg_joint}
Let $\Sigma\in\mathbb R^{k\times k}$ be SPD with eigenvalues $\lambda_1\ge\cdots\ge\lambda_k>0$.
Assume the three constraints
\begin{equation}
  \mathrm{tr}(\Sigma)=c,\qquad
  \mathrm{PR}(\Sigma)\ge r_0,\qquad
  \frac{1}{k}\log\det(\Sigma)\ge \tau,
  \label{eq:joint-constraints}
\end{equation}
for some $c>0$, $r_0\in(1,k]$, and $\tau\in\mathbb R$.
Then
\begin{equation}
  \lambda_{\max}(\Sigma)\le \frac{c}{\sqrt{r_0}};\quad
  \lambda_{\min}(\Sigma)\ge \exp(k\tau)\left(\frac{\sqrt{r_0}}{c}\right)^{k-1}.
  \label{eq:joint-bounds}
\end{equation}
Consequently, $\Sigma$ cannot realize the spike regime $\lambda_1\gg \lambda_2,\dots,\lambda_k$ and cannot approach rank collapse $\lambda_{\min}\to 0$ under the stated floors.
\end{theorem}
\noindent\emph{Proof.} See Appendix~\ref{app:proof-thm-spectral-nondeg-joint}. \hfill $\square$

\begin{corollary}[Bounded conditioning under joint non-degeneracy constraints]
\label{cor:condition_number}
Under the assumptions of Theorem \ref{thm:spectral_nondeg_joint}, the condition number satisfies
\[
\kappa(\Sigma) := \frac{\lambda_{\max}(\Sigma)}{\lambda_{\min}(\Sigma)}
\;\le\;
\exp(-k\tau)\left(\frac{c}{\sqrt{r_0}}\right)^k.
\]
\end{corollary}
\noindent\emph{Proof.} See Immediate by dividing the eigenvalue bounds in Theorem \ref{thm:spectral_nondeg_joint}. \hfill $\square$

\paragraph{Why symplecticity is still meaningful despite encoder regularizers.}
Symplecticity constrains the predictor map, while the budget/volume/PR constraints constrain the encoder outputs.
The next proposition states the correct conditional claim: if the encoder occupies nontrivial volume locally, then the symplectic predictor cannot destroy it.

\begin{proposition}[Symplectic predictors cannot induce collapse]
\label{prop:symplectic_cannot_collapse}
Fix $\phi$ and consider the predictor map $\Phi_\phi^{(K)}$ from \eqref{eq:hamjepa-Kstep}.
If the encoder outputs occupy a set $\mathcal U\subset\mathbb R^{2d_0}$ with positive Lebesgue volume, then
$\Phi_\phi^{(K)}(\mathcal U)$ has the same volume as $\mathcal U$.
In particular, $\Phi_\phi^{(K)}$ cannot map a positive-volume set to a set of zero volume; any global representational collapse must arise from $E_\theta$, not from the symplectic predictor.
\end{proposition}
\noindent\emph{Proof.} See Appendix~\ref{app:proof-prop-symplectic-cannot-collapse}. \hfill $\square$

\paragraph{Takeaway.}
Theorems~\ref{thm:flow_symplectic} and \ref{thm:leapfrog_symplectic} ensure the predictor map is symplectic and volume-preserving (hence non-contracting in a precise geometric sense), with exact time reversibility (Corollary~\ref{cor:reversible}) and reciprocal singular-value pairing (Lemma~\ref{lem:reciprocal_svs}).
Since symplecticity alone does not prevent encoder collapse, HamJEPA uses a fixed-units scale budget plus projected log-det and PR floors; Theorem~\ref{thm:spectral_nondeg_joint} formalizes conditional non-degeneracy bounds when the stated constraints hold.
For a ridged projected covariance, the ridge correction in Proposition~\ref{prop:logdet_nocollapse} is required before inferring an unregularized rank bound.
The training penalties encourage these conditions without enforcing isotropic Gaussian marginals.

\section{Mathematical Proofs}

\subsection{Proofs for Section~\ref{sec:price-of-isotropy-theory}}
\label{app:proofs}

\subsubsection{Proof of Lemma~\ref{lem:spectralV}}
\label{app:proof-lem-spectralV}
Let $\Sigma\succeq 0$ and $H\succ 0$.
For $w\in\mathcal{W}_H$, define $u := H^{-1/2}w$. Then
\[
  w^\top H^{-1}w = \|H^{-1/2}w\|_2^2 = \|u\|_2^2 \le 1,
\]
so $\mathcal{W}_H = \{H^{1/2}u:\ \|u\|_2\le 1\}$.
Substituting $w=H^{1/2}u$ into the objective yields
\[
  w^\top \Sigma w
  = u^\top \big(H^{1/2}\Sigma H^{1/2}\big)u.
\]
Therefore
\[
  V(\Sigma)
  = \sup_{\|u\|_2\le 1} u^\top \big(H^{1/2}\Sigma H^{1/2}\big)u
  = \lambda_{\max}\!\big(H^{1/2}\Sigma H^{1/2}\big),
\]
where the last equality is the Rayleigh quotient characterization of the top eigenvalue.
If $u_{\max}$ is a unit top eigenvector of $H^{1/2}\Sigma H^{1/2}$, then the corresponding
$w_{\max}:=H^{1/2}u_{\max}$ satisfies $w_{\max}^\top H^{-1}w_{\max}=1$ and attains the supremum.

\subsubsection{Proof of Theorem~\ref{thm:minimax}}
\label{app:proof-thm-minimax}
Let $\Sigma\succeq 0$ and define
\[
  A := H^{1/2}\Sigma H^{1/2} \succeq 0.
\]
Then $\mathrm{tr}(H\Sigma)=\mathrm{tr}(A)$, so the constraint $\mathrm{tr}(H\Sigma)=c$ becomes $\mathrm{tr}(A)=c$.
By Lemma~\ref{lem:spectralV}, the objective is $V(\Sigma)=\lambda_{\max}(A)$.
Thus \eqref{eq:minimax-prob-theory} is equivalent to
\[
  \min_{A\succeq 0:\ \mathrm{tr}(A)=c} \lambda_{\max}(A).
\]
Let the eigenvalues of $A$ be $\lambda_1\ge \cdots \ge \lambda_d\ge 0$.
Then $\sum_{i=1}^d \lambda_i = c$, hence $\lambda_1 \ge c/d$.
Therefore, for every feasible $A$,
\[
  \lambda_{\max}(A) \ge \frac{c}{d}.
\]
Equality $\lambda_{\max}(A)=c/d$ holds if and only if $\lambda_1=\cdots=\lambda_d=c/d$, i.e.,
$A=(c/d)I$, which is unique.
Mapping back,
\[
  \Sigma^\star = H^{-1/2} A H^{-1/2} = \frac{c}{d} H^{-1},
\]
and the optimum value is $c/d$.

\subsubsection{Proof of Theorem~\ref{thm:maxent}}
\label{app:proof-thm-maxent}
Let $Z$ have a density, $\mathbb{E}[Z]=0$, and satisfy $\mathbb{E}[Z^\top H Z]=c$.
Define the invertible linear change of variables
\[
  Y := H^{1/2} Z.
\]
Then $\mathbb{E}[Y]=0$ and
\[
  \mathbb{E}\|Y\|_2^2
  = \mathbb{E}[Z^\top H Z]
  = c.
\]
By the change-of-variables formula for differential entropy,
\[
  h(Z) = h(Y) + \log|\det(H^{-1/2})|
       = h(Y) - \tfrac{1}{2}\log\det(H).
\]
Hence maximizing $h(Z)$ under \eqref{eq:maxent-constraint-theory} is equivalent to maximizing $h(Y)$ under $\mathbb{E}\|Y\|_2^2=c$.

Let $\Sigma_Y := \mathrm{Cov}(Y)\succeq 0$.
A standard entropy bound states that for any random vector with covariance $\Sigma_Y$,
\[
  h(Y) \le \frac{1}{2}\log\Big((2\pi e)^d \det(\Sigma_Y)\Big),
\]
with equality if and only if $Y$ is Gaussian with covariance $\Sigma_Y$.

Next, the constraint $\mathbb{E}\|Y\|_2^2=c$ implies $\mathrm{tr}(\Sigma_Y)=c$ (since $\mathbb{E}[Y]=0$).
For any PSD matrix with eigenvalues $\nu_1,\dots,\nu_d\ge 0$ and fixed sum $\sum_i \nu_i=c$,
the arithmetic--geometric mean inequality gives
\[
  \det(\Sigma_Y) = \prod_{i=1}^d \nu_i
  \le \Big(\frac{1}{d}\sum_{i=1}^d \nu_i\Big)^d
  = \Big(\frac{c}{d}\Big)^d,
\]
with equality if and only if $\nu_1=\cdots=\nu_d=c/d$, i.e., $\Sigma_Y=(c/d)I$.

Combining the two inequalities yields
\[
  h(Y) \le \frac{1}{2}\log\Big((2\pi e)^d (c/d)^d\Big),
\]
with equality if and only if $Y\sim \mathcal{N}(0,(c/d)I)$.
Therefore the unique maximizer is $Y\sim \mathcal{N}(0,(c/d)I)$, and hence
\[
  Z = H^{-1/2}Y \sim \mathcal{N}\!\Big(0,\frac{c}{d}H^{-1}\Big),
\]
as claimed.

\subsubsection{Proof of Theorem~\ref{thm:price}}
\label{app:proof-thm-price}
Under Euclidean isotropy $\Sigma=\sigma^2 I$, the budget $\mathrm{tr}(H\Sigma)=c$ forces
\[
  \sigma^2 \mathrm{tr}(H) = c
  \quad\Rightarrow\quad
  \sigma^2 = \frac{c}{\mathrm{tr}(H)},
\]
so $\Sigma_{\mathrm{iso}}=\frac{c}{\mathrm{tr}(H)}I$.

Using Lemma~\ref{lem:spectralV},
\[
  V(\Sigma_{\mathrm{iso}})
  = \lambda_{\max}\!\big(H^{1/2}\Sigma_{\mathrm{iso}} H^{1/2}\big)
  = \lambda_{\max}\!\Big(\frac{c}{\mathrm{tr}(H)} H\Big)
  = \frac{c}{\mathrm{tr}(H)}\lambda_{\max}(H),
\]
which is \eqref{eq:V-iso-theory}.
From Theorem~\ref{thm:minimax}, $V(\Sigma^\star)=c/d$, so the ratio is
\[
  \frac{V(\Sigma_{\mathrm{iso}})}{V(\Sigma^\star)}
  = \frac{\frac{c}{\mathrm{tr}(H)}\lambda_{\max}(H)}{c/d}
  = \frac{d\,\lambda_{\max}(H)}{\mathrm{tr}(H)}
  = \rho(H).
\]
Since $\lambda_{\max}(H)\ge \mathrm{tr}(H)/d$ for any SPD $H$, we have $\rho(H)\ge 1$.
Also $\lambda_{\max}(H)\le \mathrm{tr}(H)$ implies $\rho(H)\le d$.
Finally, $\rho(H)=1$ holds if and only if $\lambda_{\max}(H)=\mathrm{tr}(H)/d$, which occurs if and only if all eigenvalues of $H$ are equal, i.e.\ $H\propto I$.

\subsubsection{Proof of Proposition~\ref{prop:metric-isotropy-equivalence}}
\label{app:proof_metric-isotropy-equivalence}

Let $\mu := \mathbb{E}[Z]$ so that $\Sigma = \mathbb{E}[(Z-\mu)(Z-\mu)^\top]$.
Define $Z' := H^{1/2}Z$ and note $\mathbb{E}[Z'] = H^{1/2}\mu$.
Then
\begin{align*}
\mathrm{Cov}(Z')
&= \mathbb{E}\big[(Z' - \mathbb{E}[Z'])(Z' - \mathbb{E}[Z'])^\top\big] \\
&= \mathbb{E}\big[H^{1/2}(Z-\mu)(Z-\mu)^\top(H^{1/2})^\top\big] \\
&= H^{1/2}\,\mathbb{E}\big[(Z-\mu)(Z-\mu)^\top\big]\,(H^{1/2})^\top \\
&= H^{1/2}\Sigma H^{1/2},
\end{align*}
where we used that $H^{1/2}$ is symmetric for SPD $H$.

For the equivalence: if $\Sigma=\alpha H^{-1}$ then
\[
\mathrm{Cov}(Z') = H^{1/2}(\alpha H^{-1})H^{1/2} = \alpha I.
\]
Conversely, if $\mathrm{Cov}(Z')=\alpha I$, multiply left and right by $H^{-1/2}$ to obtain
\[
\Sigma = H^{-1/2}(\alpha I)H^{-1/2} = \alpha H^{-1}.
\]
\hfill$\square$

\subsubsection{Proof of Proposition~\ref{prop:quadratic_hamiltonian_lift}}
\label{app:proof-prop-quadratic-hamiltonian-lift}

By definition,
\[
\mathcal H_H(q,p)=\tfrac12 q^\top H q+\tfrac12 \|p\|_2^2,
\]
so
\[
\pi_H(q,p)
=
\frac{1}{Z_H}
\exp\!\Big(-\frac{d}{2c}q^\top H q\Big)
\exp\!\Big(-\frac{d}{2c}\|p\|_2^2\Big).
\]
Hence the density factorizes into a \(q\)-term and a \(p\)-term, so \(q\) and \(p\) are independent.

The partition function also factorizes:
\[
Z_H
=
\Big(\int_{\mathbb R^d}\exp\!\big(-\tfrac{d}{2c}q^\top H q\big)\,dq\Big)
\Big(\int_{\mathbb R^d}\exp\!\big(-\tfrac{d}{2c}\|p\|_2^2\big)\,dp\Big).
\]
Using the standard Gaussian integral
\[
\int_{\mathbb R^d}\exp\!\big(-\tfrac12 x^\top A x\big)\,dx
=
(2\pi)^{d/2}\det(A)^{-1/2}
\qquad (A\succ 0),
\]
with \(A=(d/c)H\) for the \(q\)-integral and \(A=(d/c)I\) for the \(p\)-integral, we obtain
\[
Z_H
=
\left(2\pi \frac{c}{d}\right)^{d}\det(H)^{-1/2}.
\]

Therefore
\[
\pi_H(q,p)
=
\underbrace{
\left(2\pi \frac{c}{d}\right)^{-d/2}\det(H)^{1/2}
\exp\!\Big(-\frac{d}{2c}q^\top H q\Big)
}_{\text{\(q\)-marginal}}
\cdot
\underbrace{
\left(2\pi \frac{c}{d}\right)^{-d/2}
\exp\!\Big(-\frac{d}{2c}\|p\|_2^2\Big)
}_{\text{\(p\)-marginal}}.
\]
Thus \(q\) is Gaussian with precision matrix \((d/c)H\), hence covariance \((c/d)H^{-1}\), and \(p\) is Gaussian with precision \((d/c)I\), hence covariance \((c/d)I\):
\[
q\sim \mathcal N\!\Big(0,\frac{c}{d}H^{-1}\Big),
\qquad
p\sim \mathcal N\!\Big(0,\frac{c}{d}I\Big).
\]

Next,
\[
\mathbb E_{\pi_H}[q^\top H q]
=
\mathrm{tr}\!\Big(H\,\mathrm{Cov}(q)\Big)
=
\mathrm{tr}\!\Big(H\,\frac{c}{d}H^{-1}\Big)
=
\frac{c}{d}\,\mathrm{tr}(I)
=
c.
\]
Similarly,
\[
\mathbb E_{\pi_H}[\|p\|_2^2]
=
\mathrm{tr}\!\big(\mathrm{Cov}(p)\big)
=
\mathrm{tr}\!\Big(\frac{c}{d}I\Big)
=
c.
\]
Hence
\[
\mathbb E_{\pi_H}[\mathcal H_H(q,p)]
=
\frac12\,\mathbb E_{\pi_H}[q^\top H q]
+
\frac12\,\mathbb E_{\pi_H}[\|p\|_2^2]
=
\frac12 c+\frac12 c
=
c.
\]

Finally, the \(q\)-marginal is exactly the Gaussian
\[
\mathcal N\!\Big(0,\frac{c}{d}H^{-1}\Big),
\]
which is the minimax-optimal covariance from Theorem~\ref{thm:minimax} and the maximum-entropy law from Theorem~\ref{thm:maxent}.
\hfill$\square$

\subsubsection{Proof of Proposition~\ref{prop:oracle_family_spans_spd}}
\label{app:proof-prop-oracle-family-spans-spd}

Let \(\Sigma\succ 0\) be arbitrary and define
\[
H := \frac{c}{d}\Sigma^{-1}.
\]
Since \(\Sigma^{-1}\succ 0\), we have \(H\succ 0\).
Then
\[
\Sigma^\star(H)
=
\frac{c}{d}H^{-1}
=
\frac{c}{d}
\left(\frac{c}{d}\Sigma^{-1}\right)^{-1}
=
\frac{c}{d}\cdot \frac{d}{c}\Sigma
=
\Sigma.
\]
Thus every \(\Sigma\succ 0\) arises as \(\Sigma^\star(H)\) for some \(H\succ 0\), so the family
\(\{\Sigma^\star(H):H\succ 0\}\) is exactly the SPD cone.
\hfill$\square$

\subsubsection{Proof of Theorem~\ref{thm:no_universal_target}}
\label{app:proof_no_universal_target}

First note that $\Sigma_M(H)$ is feasible for the $H$-budget:
\[
  \mathrm{tr}\!\left(H\Sigma_M(H)\right)
  =
  \mathrm{tr}\!\left(
    H\,\frac{c}{\mathrm{tr}(HM)}M
  \right)
  =
  c.
\]
By the spectral form of the worst-case task variance,
\[
  V_H(\Sigma)
  =
  \lambda_{\max}\!\left(H^{1/2}\Sigma H^{1/2}\right).
\]
Therefore
\[
  V_H(\Sigma_M(H))
  =
  \lambda_{\max}\!\left(
    H^{1/2}
    \frac{c}{\mathrm{tr}(HM)}M
    H^{1/2}
  \right)
  =
  \frac{c}{\mathrm{tr}(HM)}
  \lambda_{\max}\!\left(H^{1/2}MH^{1/2}\right).
\]
For the oracle covariance
\[
  \Sigma^\star(H)=\frac{c}{d}H^{-1},
\]
we have
\[
  H^{1/2}\Sigma^\star(H)H^{1/2}
  =
  \frac{c}{d}I,
\]
so
\[
  V_H(\Sigma^\star(H))=\frac{c}{d}.
\]
Dividing the two expressions gives
\[
  \frac{V_H(\Sigma_M(H))}{V_H(\Sigma^\star(H))}
  =
  \frac{
    \frac{c}{\mathrm{tr}(HM)}
    \lambda_{\max}(H^{1/2}MH^{1/2})
  }{
    c/d
  }
  =
  \frac{d\,\lambda_{\max}(H^{1/2}MH^{1/2})}{\mathrm{tr}(HM)}.
\]
Since
\[
  \mathrm{tr}(HM)
  =
  \mathrm{tr}\!\left(H^{1/2}MH^{1/2}\right),
\]
writing
\[
  A := H^{1/2}MH^{1/2}\succ0
\]
gives
\[
  \frac{V_H(\Sigma_M(H))}{V_H(\Sigma^\star(H))}
  =
  \frac{d\,\lambda_{\max}(A)}{\mathrm{tr}(A)}.
\]
For any $A\succ0$,
\[
  \lambda_{\max}(A)
  \le
  \mathrm{tr}(A)
  \le
  d\,\lambda_{\max}(A),
\]
and hence
\[
  1
  \le
  \frac{d\,\lambda_{\max}(A)}{\mathrm{tr}(A)}
  \le
  d.
\]
The ratio equals $1$ if and only if
\[
  \mathrm{tr}(A)=d\,\lambda_{\max}(A),
\]
which holds if and only if all eigenvalues of $A$ are equal, i.e.
\[
  A=\alpha I
\]
for some $\alpha>0$. Since $A=H^{1/2}MH^{1/2}$, this is equivalent to
\[
  H^{1/2}MH^{1/2}=\alpha I,
\]
or, multiplying on the left and right by $H^{-1/2}$,
\[
  M=\alpha H^{-1}.
\]
Thus the regret is $1$ if and only if $M\propto H^{-1}$.

It remains to show that, for every fixed $M\succ0$, the supremum over $H\succ0$ is $d$.
The upper bound $d$ was just shown. To see that it is sharp, fix \(\delta\in(0,1)\) and define
\[
  A_\delta := \mathrm{diag}(1,\delta,\ldots,\delta)\succ0.
\]
We claim there exists $H_\delta\succ0$ such that
\[
  H_\delta^{1/2}M H_\delta^{1/2}=A_\delta.
\]
Indeed, let
\[
  B_\delta
  :=
  M^{-1/2}
  \left(M^{1/2}A_\delta M^{1/2}\right)^{1/2}
  M^{-1/2}
  \succ0
\]
and set
\[
  H_\delta := B_\delta^2.
\]
Then $H_\delta^{1/2}=B_\delta$ and
\[
  H_\delta^{1/2}M H_\delta^{1/2}
  =
  B_\delta M B_\delta
  =
  A_\delta.
\]
For this choice,
\[
  \frac{V_{H_\delta}(\Sigma_M(H_\delta))}
       {V_{H_\delta}(\Sigma^\star(H_\delta))}
  =
  \frac{d\,\lambda_{\max}(A_\delta)}{\mathrm{tr}(A_\delta)}
  =
  \frac{d}{1+(d-1)\delta}.
\]
Letting $\delta\downarrow0$ gives
\[
  \sup_{H\succ0}
  \frac{V_H(\Sigma_M(H))}{V_H(\Sigma^\star(H))}
  =
  d.
\]
For $d=1$, the ratio is identically $1=d$, so every target shape is optimal; the no-universal-shape conclusion requires $d\ge2$.
Thus no geometry-independent fixed covariance shape $M$ is canonical for all structured geometries.

\subsubsection{Proof of Proposition~\ref{prop:marginals_do_not_identify_coupling}}
\label{app:proof_marginals_do_not_identify_coupling}

Because $(Z_a,Z_b)$ is jointly Gaussian, any affine transformation of
$(Z_a,Z_b)$ is also Gaussian. Define the residual
\[
  R := Z_b - C\Sigma^{-1}Z_a .
\]
Then
\[
  \mathrm{Cov}(R,Z_a)
  =
  \mathrm{Cov}(Z_b,Z_a)
  -
  C\Sigma^{-1}\mathrm{Cov}(Z_a,Z_a)
  =
  C-C\Sigma^{-1}\Sigma
  =
  0.
\]
Since $(R,Z_a)$ is jointly Gaussian and uncorrelated, $R$ and $Z_a$ are independent.
Also, because the variables are centered,
\[
  \mathbb E[R]=0.
\]
Therefore
\[
  \mathbb E[Z_b\mid Z_a]
  =
  \mathbb E[C\Sigma^{-1}Z_a + R \mid Z_a]
  =
  C\Sigma^{-1}Z_a + \mathbb E[R\mid Z_a]
  =
  C\Sigma^{-1}Z_a.
\]

It remains to show that the same marginal covariance $\Sigma$ admits many different predictive maps.
For any $t\in(-1,1)$, consider the block covariance
\[
  \Gamma_t
  :=
  \begin{pmatrix}
    \Sigma & t\Sigma \\
    t\Sigma & \Sigma
  \end{pmatrix}.
\]
For any $u,v\in\mathbb R^d$,
\[
  \begin{pmatrix}u\\v\end{pmatrix}^{\!\top}
  \Gamma_t
  \begin{pmatrix}u\\v\end{pmatrix}
  =
  u^\top\Sigma u + 2t\,u^\top\Sigma v + v^\top\Sigma v.
\]
Equivalently,
\[
  u^\top\Sigma u + 2t\,u^\top\Sigma v + v^\top\Sigma v
  =
  \frac{1+t}{2}(u+v)^\top\Sigma(u+v)
  +
  \frac{1-t}{2}(u-v)^\top\Sigma(u-v).
\]
Since $\Sigma\succ0$ and $t\in(-1,1)$, this quadratic form is strictly positive
for every nonzero $(u,v)$. Hence $\Gamma_t\succ0$, so it defines a valid centered
joint Gaussian law with both marginals equal to $\Sigma$ and cross-covariance
\[
  C_t = t\Sigma.
\]
For this joint law, the Bayes predictor is
\[
  \mathbb E[Z_b\mid Z_a]
  =
  C_t\Sigma^{-1}Z_a
  =
  tZ_a.
\]
As $t$ varies over $(-1,1)$, the marginal covariance of both views remains exactly
$\Sigma$, but the predictive map varies continuously from nearly $-Z_a$ to nearly
$Z_a$. Thus even an oracle marginal covariance does not identify the view-to-view
coupling. Marginal matching alone therefore cannot determine the JEPA predictive
problem.

\subsection{Proofs for Section~\ref{sec:hamjepa_theory}}
\label{app:hamjepa_proofs}

\subsubsection{Proof of Lemma~\ref{lem:rowcol_transform}}
\label{app:proof-lem-rowcol-transform}

By definition, $\mathbf Z'=\mathbf ZH^{1/2}$, so its $i$th row is
\[
  (\mathbf Z')_{i,:} \;=\; (\mathbf Z_{i,:})H^{1/2} \;=\; z_i^\top H^{1/2}.
\]
Taking transpose gives
\[
  (\mathbf Z')_{i,:}^\top \;=\; (z_i^\top H^{1/2})^\top \;=\; (H^{1/2})^\top z_i \;=\; H^{1/2}z_i,
\]
where we used symmetry of $H^{1/2}$ (valid for SPD $H$).

\subsubsection{Proof of Theorem~\ref{thm:flow_symplectic}}
\label{app:proof-thm-flow-symplectic}

Fix $s_0\in\mathbb R^{2d_0}$ and let $s(t)=\Phi_{\phi,t}(s_0)$ solve $\dot s = J\nabla \mathcal H_\phi(s)$.
Let $A(t):=D\Phi_{\phi,t}(s_0)\in\mathbb R^{2d_0\times 2d_0}$ denote the Jacobian of the flow with respect to the initial condition.
Differentiating the ODE with respect to $s_0$ yields the variational equation
\begin{equation}
  \dot A(t) \;=\; J\,\nabla^2 \mathcal H_\phi(s(t))\,A(t),
  \qquad A(0)=I.
  \label{eq:variational}
\end{equation}
Define $M(t):=A(t)^\top J A(t)$.
Using \eqref{eq:variational} and the facts $J^\top=-J$ and $\nabla^2\mathcal H_\phi$ is symmetric, we compute
\[
\begin{aligned}
\dot M(t)
&= \dot A(t)^\top J A(t) + A(t)^\top J \dot A(t) \\
&= A(t)^\top \nabla^2\mathcal H_\phi(s(t))^\top J^\top J A(t) + A(t)^\top J J \nabla^2\mathcal H_\phi(s(t)) A(t).
\end{aligned}
\]
Since $\nabla^2\mathcal H_\phi(s(t))^\top=\nabla^2\mathcal H_\phi(s(t))$ and $J^\top=-J$ and $JJ=-I$, this becomes
\[
\begin{aligned}
\dot M(t)
&= A(t)^\top \nabla^2\mathcal H_\phi(s(t))\,(-J)J A(t) + A(t)^\top (-I)\nabla^2\mathcal H_\phi(s(t)) A(t) \\
&= A(t)^\top \nabla^2\mathcal H_\phi(s(t))\,A(t) - A(t)^\top \nabla^2\mathcal H_\phi(s(t))\,A(t) \\
&=0.
\end{aligned}
\]
Thus $M(t)$ is constant in $t$.
At $t=0$, $A(0)=I$ so $M(0)=J$, hence $M(t)=J$ for all $t$, proving \eqref{eq:flow-symplectic}.

For the determinant claim, take determinants of $A(t)^\top J A(t)=J$:
\[
\det(A(t))^2 \det(J) = \det(J)\quad\Rightarrow\quad \det(A(t))^2=1.
\]
Since $A(0)=I$ has $\det(A(0))=1$ and $\det(A(t))$ varies continuously with $t$, we must have $\det(A(t))=1$ for all $t$.

Finally, volume preservation follows from $\det(D\Phi_{\phi,t}(s))\equiv 1$ by the change-of-variables theorem: for any measurable set $U$ in the domain of $\Phi_{\phi,t}$,
\[
\mathrm{Vol}(\Phi_{\phi,t}(U))=\int_U |\det(D\Phi_{\phi,t}(s))|\,ds=\int_U 1\,ds=\mathrm{Vol}(U).
\]

\subsubsection{Proof of Corollary~\ref{cor:entropy_preservation}}
\label{app:proof-cor-entropy-preservation}

Let $Y=F(X)$ with $F$ a $C^1$ diffeomorphism.
The change-of-variables formula gives $p_Y(y)=p_X(x)\,|\det(DF(x))|^{-1}$ with $x=F^{-1}(y)$.
Then
\[
\begin{aligned}
h(Y)&=-\int p_Y(y)\log p_Y(y)\,dy \\
&=-\int p_X(x)\log\big(p_X(x)\,|\det(DF(x))|^{-1}\big)\,dx \\
&= h(X)+\mathbb E[\log|\det(DF(X))|].
\end{aligned}
\]
If $|\det(DF(\cdot))|\equiv 1$, the second term is zero, so $h(Y)=h(X)$.

\subsubsection{Proof of Theorem~\ref{thm:leapfrog_symplectic}}
\label{app:proof-thm-leapfrog-symplectic}

Write the leapfrog step as a composition of two ``kicks'' and one ``drift''.
Define for $h\in\mathbb R$:
\[
K_h(q,p) := (q,\; p - h\nabla V_\phi(q)),
\qquad
D_h(q,p) := (q + h p,\; p),
\]
where the drift uses $\nabla T(p)=p$ for $T(p)=\tfrac12\|p\|^2$.
Then the one-step leapfrog map satisfies
\[
\Psi_{\phi,\Delta t} \;=\; K_{\Delta t/2}\circ D_{\Delta t}\circ K_{\Delta t/2}.
\]

We show that both $K_h$ and $D_h$ are symplectic.
Their Jacobians have block forms
\[
\begin{aligned}
DK_h(q,p)&=
\begin{pmatrix}
I & 0\\
-h\nabla^2 V_\phi(q) & I
\end{pmatrix}
=
\begin{pmatrix}
I & 0\\
B & I
\end{pmatrix},
\quad B^\top=B,
\\
DD_h(q,p)&=
\begin{pmatrix}
I & hI\\
0 & I
\end{pmatrix}
=
\begin{pmatrix}
I & C\\
0 & I
\end{pmatrix},
\quad C^\top=C.
\end{aligned}
\]
A direct calculation shows that any matrix of the form $\begin{psmallmatrix}I&0\\B&I\end{psmallmatrix}$ with $B$ symmetric satisfies $A^\top J A=J$, and similarly for $\begin{psmallmatrix}I&C\\0&I\end{psmallmatrix}$ with $C$ symmetric.
Therefore $DK_h$ and $DD_h$ are symplectic matrices at every point, so $K_h$ and $D_h$ are symplectic maps.

The composition of symplectic maps is symplectic, hence $\Psi_{\phi,\Delta t}$ is symplectic, and iterating yields symplecticity of $\Phi_\phi^{(K)}=\Psi_{\phi,\Delta t}^{\,K}$.
The determinant claim follows as in the proof of Theorem~\ref{thm:flow_symplectic}.

\subsubsection{Proof of Corollary~\ref{cor:reversible}}
\label{app:proof-cor-reversible}

Using the decomposition $\Psi_{\Delta t}=K_{\Delta t/2}\circ D_{\Delta t}\circ K_{\Delta t/2}$, we have
$K_h^{-1}=K_{-h}$ and $D_h^{-1}=D_{-h}$ by inspection.
Hence
\[
\Psi_{\Delta t}^{-1}
=K_{\Delta t/2}^{-1}\circ D_{\Delta t}^{-1}\circ K_{\Delta t/2}^{-1}
=K_{-\Delta t/2}\circ D_{-\Delta t}\circ K_{-\Delta t/2}
=\Psi_{-\Delta t}.
\]
Taking $K$-fold compositions gives $(\Psi_{\Delta t}^{\,K})^{-1}=\Psi_{-\Delta t}^{\,K}$.

\subsubsection{Proof of Proposition~\ref{prop:hamjepa_expressivity}}
\label{app:proof_hamjepa_expressivity}

Write the phase-space state as $s=(q,p)\in\mathbb{R}^{2d_0}$ and define the two vector fields
\[
f(s) := \binom{p}{-\nabla V(q)},\qquad
f_\phi(s) := \binom{p}{-\nabla V_\phi(q)}.
\]
Let $\Phi^V_t$ denote the flow of $\dot s=f(s)$ and $\Phi^{V_\phi}_t$ the flow of $\dot s=f_\phi(s)$.

Throughout, work on the compact set $\Omega=\Omega_q\times\Omega_p$ from the proposition statement.
Since $\nabla V$ and $\nabla V_\phi$ are Lipschitz on $\Omega_q$ and $\Omega_p$ is compact, both $f$ and $f_\phi$
are Lipschitz on $\Omega$; let $L$ be any Lipschitz constant of $f_\phi$ on $\Omega$:
\[
\|f_\phi(x)-f_\phi(y)\|_2 \le L\|x-y\|_2,\qquad \forall x,y\in\Omega.
\]
Define the uniform force mismatch on $\Omega_q$:
\[
\delta := \sup_{q\in\Omega_q}\|\nabla V_\phi(q)-\nabla V(q)\|_2.
\]

\paragraph{Step 1: Continuous dependence of flows on the force field.}
Fix $s_0\in\mathcal K$ and let $s(t):=\Phi^V_t(s_0)$ and $\tilde s(t):=\Phi^{V_\phi}_t(s_0)$.
By assumption, both trajectories remain in $\Omega$ for $t\in[0,T]$.
Let $\Delta(t):=\tilde s(t)-s(t)$. Then
\begin{align*}
\dot\Delta(t)
&= f_\phi(\tilde s(t)) - f(s(t)) \\
&= \bigl(f_\phi(\tilde s(t)) - f_\phi(s(t))\bigr) + \bigl(f_\phi(s(t)) - f(s(t))\bigr).
\end{align*}
Taking norms and using Lipschitzness of $f_\phi$ on $\Omega$ plus the definition of $\delta$ gives
\[
\|\dot\Delta(t)\|_2 \le L\|\Delta(t)\|_2 + \delta.
\]
Since $\Delta(0)=0$, Gr\"onwall's inequality yields, for all $t\in[0,T]$,
\[
\|\Delta(t)\|_2 \le
\begin{cases}
\displaystyle \frac{\delta}{L}\bigl(e^{Lt}-1\bigr), & L>0,\\[6pt]
\displaystyle \delta t, & L=0.
\end{cases}
\]
In particular, at $t=T$ we obtain a uniform (over $s_0\in\mathcal K$) modeling-error bound
\[
\sup_{s_0\in\mathcal K}\|\Phi^{V_\phi}_T(s_0)-\Phi^V_T(s_0)\|_2 \le C_1\,\delta,
\qquad
C_1 :=
\begin{cases}
\frac{e^{LT}-1}{L}, & L>0,\\
T, & L=0.
\end{cases}
\]

\paragraph{Step 2: Leapfrog global error is $O(\Delta t^2)$.}
Fix $\phi$ and consider the ODE $\dot s=f_\phi(s)$.
Let $\Psi_{\phi,h}$ denote one leapfrog step with step size $h$ (Eq.~\eqref{eq:hamjepa-leapfrog}), and define the numerical iterates
\[
s_{n+1} := \Psi_{\phi,h}(s_n),\qquad s_0 \in\mathcal K,\qquad t_n := nh,
\]
so that $s_K = (\Psi_{\phi,h})^K(s_0)=\Phi^{(K)}_\phi(s_0)$ with $Kh=T$.

\emph{(2a) Local truncation error bound.}
Because $V_\phi\in C^3$ on a neighborhood of $\Omega_q$, the vector fields
\[
A(q,p):=\binom{p}{0},
\qquad
B(q,p):=\binom{0}{-\nabla V_\phi(q)}
\]
are $C^2$ on $\Omega$, and their exact flows are
\[
\exp(hA):(q,p)\mapsto(q+hp,p),
\qquad
\exp(hB):(q,p)\mapsto(q,p-h\nabla V_\phi(q)).
\]
The leapfrog step is exactly the Strang splitting:
\[
\Psi_{\phi,h}=\exp\!\left(\frac{h}{2}B\right)\exp(hA)\exp\!\left(\frac{h}{2}B\right).
\]
Standard Strang splitting theory (derivable via the BCH expansion for the Lie operators $L_A,L_B$)
implies that the local truncation error is third order:
there exists a constant $C_{\mathrm{loc}}=C_{\mathrm{loc}}(\Omega,V_\phi)$ such that
\[
\sup_{x\in\Omega}\bigl\|\Psi_{\phi,h}(x) - \Phi^{V_\phi}_h(x)\bigr\|_2 \le C_{\mathrm{loc}}\,h^3
\quad\text{for all sufficiently small }h.
\]
(Concretely, the BCH expansion gives
$\exp(\tfrac{h}{2}L_B)\exp(hL_A)\exp(\tfrac{h}{2}L_B)
= \exp\big(h(L_A+L_B) + O(h^3)\big)$
as operators on smooth test functions, and boundedness of the commutators on $\Omega$ yields the uniform $C_{\mathrm{loc}}$.)

\emph{(2b) Error recursion and discrete Gr\"onwall.}
Define the global error at step $n$ by
\[
e_n := s_n - \Phi^{V_\phi}_{t_n}(s_0).
\]
Then
\begin{align*}
e_{n+1}
&= \Psi_{\phi,h}(s_n) - \Phi^{V_\phi}_{t_{n+1}}(s_0) \\
&= \Psi_{\phi,h}(s_n) - \Phi^{V_\phi}_{h}\bigl(\Phi^{V_\phi}_{t_n}(s_0)\bigr).
\end{align*}
Add and subtract $\Phi^{V_\phi}_h(s_n)$:
\[
e_{n+1}
= \underbrace{\bigl(\Psi_{\phi,h}(s_n) - \Phi^{V_\phi}_{h}(s_n)\bigr)}_{\text{local truncation}}
+ \underbrace{\bigl(\Phi^{V_\phi}_{h}(s_n) - \Phi^{V_\phi}_{h}(\Phi^{V_\phi}_{t_n}(s_0))\bigr)}_{\text{stability}}.
\]
Taking norms, using the local error bound from (2a), and using Lipschitz stability of the flow
$\|\Phi^{V_\phi}_h(x)-\Phi^{V_\phi}_h(y)\|_2 \le e^{Lh}\|x-y\|_2$ on $\Omega$ yields
\[
\|e_{n+1}\|_2 \le C_{\mathrm{loc}}h^3 + e^{Lh}\|e_n\|_2.
\]
Iterating this recursion with $e_0=0$ gives
\[
\|e_K\|_2
\le C_{\mathrm{loc}}h^3\sum_{j=0}^{K-1} e^{Lh j}
=
C_{\mathrm{loc}}h^3 \cdot
\begin{cases}
\displaystyle \frac{e^{LKh}-1}{e^{Lh}-1}, & L>0,\\[6pt]
K, & L=0.
\end{cases}
\]
Since $Kh=T$ and $e^{Lh}-1 \ge Lh$ for $Lh\ge 0$, we obtain
\[
\|e_K\|_2 \le
\begin{cases}
\displaystyle \frac{C_{\mathrm{loc}}}{L}\bigl(e^{LT}-1\bigr)\,h^2, & L>0,\\[6pt]
\displaystyle C_{\mathrm{loc}}T\,h^2, & L=0.
\end{cases}
\]
Taking the supremum over $s_0\in\mathcal K$ (all iterates stay in $\Omega$ by assumption), we conclude that there exists
$C_2=C_2(T,\Omega,V_\phi)$ such that
\[
\sup_{s_0\in\mathcal K}\bigl\|\Phi^{(K)}_\phi(s_0)-\Phi^{V_\phi}_T(s_0)\bigr\|_2
\le C_2\,h^2.
\]

\paragraph{Step 3: Combine modeling and integrator errors.}
By the triangle inequality,
\[
\|\Phi^{(K)}_\phi(s_0)-\Phi^V_T(s_0)\|_2
\le
\|\Phi^{(K)}_\phi(s_0)-\Phi^{V_\phi}_T(s_0)\|_2
+
\|\Phi^{V_\phi}_T(s_0)-\Phi^V_T(s_0)\|_2.
\]
Taking suprema over $s_0\in\mathcal K$ and applying the bounds from Steps 1--2 yields
\[
\sup_{s_0\in\mathcal K}\|\Phi^{(K)}_\phi(s_0)-\Phi^V_T(s_0)\|_2
\le C_2 h^2 + C_1 \delta.
\]
This is the claimed inequality with $C:=\max\{C_1,C_2\}$ (or $C:=C_1+C_2$).

Finally, if the model class can make $\delta$ arbitrarily small (uniform approximation of $\nabla V$ on $\Omega_q$),
the assumed common Lipschitz bound keeps \(C_1\) uniformly bounded along the approximating family.
Choose $\phi$ so that $C_1\delta \le \varepsilon/2$, then choose $h=\Delta t$ so that $C_2h^2\le \varepsilon/2$,
and conclude
\[
\sup_{s_0\in\mathcal K}\|\Phi^{(K)}_\phi(s_0)-\Phi^V_T(s_0)\|_2 \le \varepsilon.
\]
\hfill$\square$

\subsubsection{Proof of Lemma~\ref{lem:reciprocal_svs}}
\label{app:proof-lem-reciprocal-svs}

Let $A$ be symplectic: $A^\top J A=J$.
Using $J^{-1}=-J$ and $J^\top=-J$, we obtain
\[
A^{-1} = -J A^\top J.
\]
Then
\[
(A^\top A)^{-1} = A^{-1}A^{-T} = (-J A^\top J)(-J A J) = J^{-1}(A^\top A)J,
\]
so $(A^\top A)^{-1}$ is similar to $A^\top A$ and thus has the same eigenvalues.
Therefore the eigenvalues of $A^\top A$ occur in reciprocal pairs $(\mu,1/\mu)$, and the singular values
$\sigma=\sqrt{\mu}$ occur in reciprocal pairs $(\sigma,1/\sigma)$.

\subsubsection{Proof sketch of Theorem~\ref{thm:shadow}}
\label{app:proof-thm-shadow}

This is a standard result from backward error analysis for symplectic integrators; we provide the BCH-based sketch
specialized to leapfrog.

Let $L_T$ and $L_V$ denote the Lie derivatives (Poisson bracket operators) corresponding to the Hamiltonians
$T$ and $V$, acting on smooth observables.
The exact time-$h$ flow of $T$ (resp.\ $V$) is $\exp(hL_T)$ (resp.\ $\exp(hL_V)$).
The leapfrog step can be written operator-theoretically as the symmetric composition
\[
\exp\!\Big(\tfrac{h}{2}L_V\Big)\exp(hL_T)\exp\!\Big(\tfrac{h}{2}L_V\Big).
\]
By the Baker--Campbell--Hausdorff (BCH) formula, this equals
\[
\exp\!\Big(h(L_T+L_V) + h^3 C_3 + h^5 C_5 + \cdots \Big),
\]
where $C_3,C_5,\dots$ are nested commutators of $L_T$ and $L_V$.
For Hamiltonian systems, commutators of Hamiltonian vector fields correspond to Hamiltonian vector fields of
Poisson brackets; equivalently, each $C_{2j+1}$ is itself a Lie derivative $L_{H_{2j}}$ for some scalar
function $H_{2j}$ built from iterated Poisson brackets of $T$ and $V$.
Thus the exponential above is the exact time-$h$ flow of a \emph{modified Hamiltonian}
\[
\widetilde{\mathcal H}_{\Delta t} \;=\; \mathcal H + h^2 H_2 + h^4 H_4 + \cdots
\]
at the level of formal series.
Truncating this series at order $N$ yields a modified Hamiltonian
$\widetilde{\mathcal H}^{[N]}_{\Delta t}$ whose flow matches one leapfrog step up to local error $\mathcal O(h^{N+1})$.
Under analyticity, classical results show that the truncation can be chosen so that the resulting modified energy
is nearly conserved over exponentially long times in $1/h$.
(See, e.g., standard references on geometric numerical integration for full details.)

For the stated finite-order bound, fix \(N\).
The truncated modified energy has uniformly bounded gradient on the controlled neighborhood for sufficiently small \(h\).
Its exact modified flow preserves that energy, whereas the numerical step differs from the modified flow by \(\mathcal O(h^{N+1})\).
Hence
\[
|\widetilde{\mathcal H}^{[N]}(s_{j+1})-\widetilde{\mathcal H}^{[N]}(s_j)|
\le C_1 h^{N+1}.
\]
Summing over \(n\) steps and using the \(\mathcal O(h^2)\) difference between the original and modified energies at both endpoints gives
\[
|\mathcal H_\phi(s_n)-\mathcal H_\phi(s_0)|
\le C_0h^2+C_1nh^{N+1}
=C_0h^2+C_1T_nh^N.
\]

\subsubsection{Proof of Lemma~\ref{lem:budget_not_rank}}
\label{app:proof-lem-budget-not-rank}

Take $\Sigma=\mathrm{diag}(c,0,\dots,0)$.
Then $\Sigma\succeq 0$, $\mathrm{tr}(\Sigma)=c$, and $\mathrm{rank}(\Sigma)=1$.

\subsubsection{Proof of Lemma~\ref{lem:logdet_spike}}
\label{app:proof-lem-logdet-spike}

Fix $k\ge 2$ and $\tau$.
For $\varepsilon>0$, define eigenvalues
\[
\lambda_2=\cdots=\lambda_k=\varepsilon,
\qquad
\lambda_1 = \exp(k\tau)\,\varepsilon^{-(k-1)}.
\]
Then $\prod_{i=1}^k \lambda_i = \exp(k\tau)$, hence $\frac1k\log\det(\Sigma)=\tau$.
But as $\varepsilon\to 0$, $\lambda_1\to\infty$ and $\lambda_2,\dots,\lambda_k\to 0$.
The participation ratio satisfies
\[
\mathrm{PR}(\Sigma)=\frac{(\sum_i\lambda_i)^2}{\sum_i\lambda_i^2}
\longrightarrow
\frac{\lambda_1^2}{\lambda_1^2}=1,
\]
so $\mathrm{PR}(\Sigma)$ can be made arbitrarily close to $1$ while maintaining the log-det constraint.

\subsubsection{Proof of Lemma~\ref{lem:pr_not_scale}}
\label{app:proof-lem-pr-not-scale}

Let $\lambda_i$ be the eigenvalues of $\Sigma$.
Then the eigenvalues of $a\Sigma$ are $a\lambda_i$, so
\[
\mathrm{PR}(a\Sigma)=\frac{(\sum_i a\lambda_i)^2}{\sum_i (a\lambda_i)^2}
=\frac{a^2(\sum_i \lambda_i)^2}{a^2\sum_i\lambda_i^2}
=\mathrm{PR}(\Sigma).
\]

\subsubsection{Proof of Lemma~\ref{lem:proj_logdet_rank_certificate}}
\label{app:proof_proj_logdet_rank_certificate}

Let $\Sigma_Q\succeq 0$ and $R\in\mathbb{R}^{d_0\times k}$ have full column rank.
Write $\Sigma_Q = A^\top A$ with $A:=\Sigma_Q^{1/2}$ (PSD square root).
Then
\[
\Sigma_{Y,0}=R^\top \Sigma_Q R = R^\top A^\top A R = (AR)^\top (AR).
\]
Therefore $\Sigma_{Y,0}$ is PSD and
\[
\rank(\Sigma_{Y,0}) = \rank\big((AR)^\top(AR)\big) = \rank(AR) \le \rank(A) = \rank(\Sigma_Q).
\]
If $\rank(\Sigma_Q)<k$, then $\rank(\Sigma_{Y,0})<k$, so $\Sigma_{Y,0}$ is singular. Hence $\det(\Sigma_{Y,0})=0$
and $\log\det(\Sigma_{Y,0})=-\infty$.

For the ridge statement, let $r:=\rank(\Sigma_{Y,0})$ and let $\lambda_1,\dots,\lambda_r>0$ be its nonzero eigenvalues.
Then $\Sigma_{Y,0}$ has exactly $k-r$ zero eigenvalues. Adding $\varepsilon I_k$ shifts each eigenvalue by $\varepsilon$,
so $\Sigma_Y=\Sigma_{Y,0}+\varepsilon I_k$ has eigenvalues
\[
\lambda_1+\varepsilon,\ \dots,\ \lambda_r+\varepsilon,\ \underbrace{\varepsilon,\dots,\varepsilon}_{k-r \text{ times}}.
\]
Thus
\[
\log\det(\Sigma_Y)=\sum_{i=1}^{r}\log(\lambda_i+\varepsilon) + (k-r)\log\varepsilon,
\]
which is exactly the claimed formula.
\hfill$\square$

\subsubsection{Proof of Proposition~\ref{prop:logdet_nocollapse}}
\label{app:proof-prop-logdet-nocollapse}

Let $\lambda_1\ge\cdots\ge\lambda_k>0$ be eigenvalues of $\Sigma_Y$.
From $\log\det(\Sigma_Y)\ge k\tau$ we have
\[
\prod_{i=1}^k \lambda_i \;\ge\; \exp(k\tau).
\]
Let $\lambda_{\min}=\lambda_k$.
Using $\sum_{i=1}^{k-1}\lambda_i \le \mathrm{tr}(\Sigma_Y)-\lambda_k \le m$, the AM--GM inequality yields
\[
\prod_{i=1}^{k-1}\lambda_i
\;\le\;
\Big(\frac{\sum_{i=1}^{k-1}\lambda_i}{k-1}\Big)^{k-1}
\;\le\;
\Big(\frac{m}{k-1}\Big)^{k-1}.
\]
Therefore
\[
\exp(k\tau)
\;\le\;
\prod_{i=1}^k \lambda_i
=
\lambda_k \prod_{i=1}^{k-1}\lambda_i
\;\le\;
\lambda_k \Big(\frac{m}{k-1}\Big)^{k-1},
\]
so
\[
\lambda_{\min}(\Sigma_Y)=\lambda_k
\;\ge\;
\exp(k\tau)\,\frac{(k-1)^{k-1}}{m^{k-1}}.
\]

\subsubsection{Proof of Theorem~\ref{thm:spectral_nondeg_joint}}
\label{app:proof-thm-spectral-nondeg-joint}

Let $\lambda_1\ge\cdots\ge\lambda_k>0$ be eigenvalues of $\Sigma$ and $c=\sum_i\lambda_i$.
The PR constraint implies
\[
\mathrm{PR}(\Sigma)=\frac{(\sum_i\lambda_i)^2}{\sum_i\lambda_i^2}\ge r_0
\quad\Rightarrow\quad
\sum_{i=1}^k\lambda_i^2 \le \frac{c^2}{r_0}.
\]
Since $\lambda_{\max}^2=\lambda_1^2\le\sum_i\lambda_i^2$, we get
\[
\lambda_{\max}(\Sigma)\le \frac{c}{\sqrt{r_0}}.
\]
Next, the log-det constraint implies
\[
\prod_{i=1}^k\lambda_i = \det(\Sigma) \ge \exp(k\tau).
\]
Using $\lambda_i\le\lambda_{\max}$ for $i=1,\dots,k-1$ gives
\[
\exp(k\tau)\le \lambda_{\min}\,\lambda_{\max}^{k-1},
\quad\text{so}\quad
\lambda_{\min}\ge \frac{\exp(k\tau)}{\lambda_{\max}^{k-1}}
\ge
\exp(k\tau)\left(\frac{\sqrt{r_0}}{c}\right)^{k-1}.
\]

\subsubsection{Proof of Proposition~\ref{prop:symplectic_cannot_collapse}}
\label{app:proof-prop-symplectic-cannot-collapse}

By Theorem~\ref{thm:leapfrog_symplectic}, $\Phi_\phi^{(K)}$ is a $C^1$ diffeomorphism with
$|\det(D\Phi_\phi^{(K)}(s))|=1$ for all $s$.
For any measurable $\mathcal U$ with finite volume, the change-of-variables theorem gives
\[
\mathrm{Vol}\big(\Phi_\phi^{(K)}(\mathcal U)\big)
= \int_{\mathcal U} \left|\det(D\Phi_\phi^{(K)}(s))\right|\,ds
= \int_{\mathcal U} 1\,ds
= \mathrm{Vol}(\mathcal U).
\]
Hence $\mathrm{Vol}(\mathcal U)>0$ implies $\mathrm{Vol}(\Phi_\phi^{(K)}(\mathcal U))>0$.
Therefore the predictor cannot map a positive-volume set to a zero-volume set; any collapse to a
lower-volume region must already be present in the encoder outputs.

\section{Oracle coupling and structural characterization}
\label{app:oracle_coupling}

This appendix proves the additional known-geometry results in
Section~\ref{sec:canonical_reversible_lift}.
Throughout, \(H\in\mathbb R^{d\times d}\) is symmetric positive definite and
\(\alpha=c/d>0\).
The continuous characterization concerns linear one-parameter flows.
The exact discrete Gaussian statements concern the quadratic potential
\(\tfrac12 q^\top Hq\), rather than an arbitrary learned nonlinear potential.

\subsection{Proof of Theorem~\ref{thm:canonical_linear_oracle_lift}}
\label{app:proof_canonical_linear_oracle_lift}

\paragraph{Existence of a constant generator.}
A continuous one-parameter group of linear maps is automatically differentiable in finite dimensions.
Here is a direct argument.
For sufficiently small \(a>0\), the matrix
\[
B_a:=\int_0^a F_s\,ds
\]
is invertible, since \(a^{-1}B_a\to I\) as \(a\downarrow0\).
The group property gives
\[
F_tB_a=\int_t^{t+a}F_u\,du .
\]
The right-hand side is continuously differentiable in \(t\), with derivative
\(F_{t+a}-F_t=F_t(F_a-I)\).
Therefore
\[
F_t'=F_tL,\qquad
L:=(F_a-I)B_a^{-1},
\]
and \(F_0=I\) implies \(F_t=e^{tL}\).
In particular the resulting linear flow exists for every real \(t\).

\paragraph{The oracle constraints determine the generator.}
Write
\[
L=\begin{pmatrix}A&B\\C&D\end{pmatrix},
\qquad
\Gamma_H=\alpha\begin{pmatrix}H^{-1}&0\\0&I\end{pmatrix},
\qquad
\mathsf R=\begin{pmatrix}I&0\\0&-I\end{pmatrix}.
\]
Differentiating \(\mathsf R F_t\mathsf R=F_{-t}\) at \(t=0\) yields
\[
\mathsf R L\mathsf R=-L .
\]
Its diagonal blocks give \(A=D=0\).
The kinematic identity \(\dot Q=P\), required at every state, then fixes \(B=I\).
Preservation of the centered Gaussian law is equivalent to
\[
F_t\Gamma_HF_t^\top=\Gamma_H .
\]
Differentiating this identity at zero gives
\[
0=L\Gamma_H+\Gamma_HL^\top
=\alpha
\begin{pmatrix}
0&I+H^{-1}C^\top\\
CH^{-1}+I&0
\end{pmatrix}.
\]
Since \(H\) is invertible, \(C=-H\).
This proves uniqueness of
\[
L_H=\begin{pmatrix}0&I\\-H&0\end{pmatrix}.
\]

Conversely, this generator satisfies
\[
L_H\Gamma_H+\Gamma_HL_H^\top=0,\qquad
\mathsf R L_H\mathsf R=-L_H .
\]
The first identity implies
\(\frac{d}{dt}(e^{tL_H}\Gamma_He^{tL_H^\top})=0\), so the Gaussian law is stationary.
The second implies
\(\mathsf R e^{tL_H}\mathsf R=e^{-tL_H}\), proving reversibility.
Its upper block is precisely \(\dot Q=P\).
Finally,
\[
L_H=J\begin{pmatrix}H&0\\0&I\end{pmatrix},
\qquad
\mathcal H_H(q,p)=\tfrac12q^\top Hq+\tfrac12\|p\|_2^2 .
\]
Hence the characterized flow is Hamiltonian and symplectic, although symplecticity was not imposed among the oracle constraints.
\hfill\(\square\)

\paragraph{Scope of the characterization.}
The lifted Gaussian law, its independent fixed-unit momentum, and the momentum-reversal involution are explicit assumptions.
They are stronger than equality of the two observable marginals or exchangeability of an observed view pair.
Theorem~\ref{thm:canonical_linear_oracle_lift} does not assert that these observable properties force a unique latent construction.
It also does not assert uniqueness among nonlinear vector fields.

The role of reversal can be seen directly within the linear class.
If stationarity and \(\dot Q=P\) are retained but reversal is omitted, the allowed generators are
\[
L=\begin{pmatrix}0&I\\-H&D\end{pmatrix},
\qquad D+D^\top=0 .
\]
Indeed, the same stationarity calculation fixes the lower-left block to \(-H\) and only requires the remaining block to be skew-symmetric.
For \(d\ge2\), nonzero such blocks exist.
If instead reversal is retained but the fixed-unit kinematic condition is omitted, every generator
\[
L_B=\begin{pmatrix}0&B\\-B^\top H&0\end{pmatrix},
\qquad B\in\mathbb R^{d\times d},
\]
preserves the same oracle law and is reversible.
Thus the uniqueness claim depends on the declared class and normalization.

\subsection{Proof of Corollary~\ref{cor:oracle_observable_coupling}}
\label{app:proof_oracle_observable_coupling}

Set
\[
C_t:=\cos(tH^{1/2}),\qquad S_t:=\sin(tH^{1/2}),
\]
using the spectral functional calculus for the symmetric positive definite matrix \(H\).
Solving \(\ddot Q=-HQ\) with initial velocity \(P_0\) gives
\[
Q_t=C_tQ_0+H^{-1/2}S_tP_0,
\qquad
P_t=-H^{1/2}S_tQ_0+C_tP_0.
\]
Under the oracle law, \(Q_0\sim\mathcal N(0,\alpha H^{-1})\) and
\(P_0\sim\mathcal N(0,\alpha I)\) are independent.
Conditioning on \(Q_0=q\) therefore gives
\[
Q_t\mid Q_0=q
\sim
\mathcal N\!\left(
C_tq,\,
\alpha H^{-1/2}S_t^2H^{-1/2}
\right).
\]
This includes degenerate Gaussian conditionals: in an eigenmode where
\(\sin(t\sqrt{\lambda})=0\), the conditional variance in that mode vanishes.
At other times and modes the auxiliary momentum produces a nontrivial observable conditional variance.

Write \(\Sigma=\alpha H^{-1}\).
All the displayed functions of \(H\) commute.
Using \(C_t^2+S_t^2=I\) shows
\[
\operatorname{Cov}(Q_t)
=C_t\Sigma C_t+\alpha H^{-1/2}S_t^2H^{-1/2}
=\Sigma ,
\]
and independence of \(Q_0,P_0\) gives
\[
\operatorname{Cov}(Q_t,Q_0)=C_t\Sigma .
\]
Consequently the joint centered Gaussian law of \((Q_0,Q_t)\) has covariance
\[
\begin{pmatrix}
\Sigma&\Sigma C_t\\
C_t\Sigma&\Sigma
\end{pmatrix}.
\]
Because \(C_t\Sigma=\Sigma C_t\) is symmetric, this covariance is invariant under exchanging the two blocks.
The joint law is therefore exchangeable.
Finally,
\[
M(t)=\operatorname{Cov}(Q_t,Q_0)\Sigma^{-1}=C_t,
\qquad
M''(0)=-H,
\]
which proves the curvature identity.
\hfill\(\square\)

\paragraph{What the curvature identity identifies.}
Knowing the normalized coupling \(M(t)\) on a neighborhood of zero determines \(H\) through \(-M''(0)\).
A single normalized coupling matrix \(M(t_0)\), by itself, generally does not: in one dimension,
\(\omega^2\) and \((\omega+2\pi n/t_0)^2\) give the same cosine for any positive integer \(n\).
This observation concerns coupling at one lag, not the full oracle marginal, from which
\(H=\alpha\Sigma^{-1}\) is already determined when \(\alpha\) is known.
Neither identity proves that the practical training objective identifies the unknown downstream task geometry.

\subsection{Proof of Proposition~\ref{prop:quadratic_leapfrog_oracle_coupling}}
\label{app:proof_quadratic_leapfrog_oracle_coupling}

\paragraph{The finite-step map and its invariant law.}
Let
\[
A_h:=I-\tfrac{h^2}{2}H,\qquad
R_h:=I-\tfrac{h^2}{4}H.
\]
The stability condition \(h^2\lambda_{\max}(H)<4\) makes \(R_h\) positive definite.
For the quadratic potential, a kick--drift--kick leapfrog step has matrix
\begin{align*}
F_h
&=
\begin{pmatrix}I&0\\-\frac h2H&I\end{pmatrix}
\begin{pmatrix}I&hI\\0&I\end{pmatrix}
\begin{pmatrix}I&0\\-\frac h2H&I\end{pmatrix}\\
&=
\begin{pmatrix}
A_h&hI\\
-hHR_h&A_h
\end{pmatrix}.
\end{align*}
The factors \(A_h,R_h,H\) commute and satisfy
\[
A_h^2+h^2HR_h=I .
\]
For
\(\Gamma_{H,h}=\alpha\operatorname{diag}(H^{-1},R_h)\),
direct multiplication now gives
\[
F_h\Gamma_{H,h}F_h^\top
=\alpha
\begin{pmatrix}
H^{-1}(A_h^2+h^2HR_h)&0\\
0&R_h(A_h^2+h^2HR_h)
\end{pmatrix}
=\Gamma_{H,h}.
\]
Thus every iterate preserves the centered Gaussian law.
Its density is the quadratic shadow Gibbs law
\begin{equation}
\pi_{H,h}(q,p)
\propto
\exp\!\left[
-\frac{1}{2\alpha}
\left(q^\top Hq+p^\top R_h^{-1}p\right)
\right].
\label{eq:quadratic_shadow_gibbs_law}
\end{equation}
Both quadratic forms are positive definite, so this density is normalizable.
The associated quadratic energy is exactly conserved by \(F_h\).
No assertion that this particular invariant energy generates \(F_h\) as its exact time-\(h\) flow is needed.

\paragraph{Chebyshev form of the coupling.}
Let \(\mathcal A_h=\operatorname{diag}(A_h,A_h)\).
Another direct multiplication gives
\[
F_h^2-2\mathcal A_hF_h+I_{2d}=0 .
\]
Moreover \(F_h\) commutes with \(\mathcal A_h\).
If \(B_K\) denotes the upper-left block of \(F_h^K\), then
\[
B_0=I,\qquad B_1=A_h,\qquad
B_{K+1}=2A_hB_K-B_{K-1}.
\]
This is the defining recurrence of the first-kind Chebyshev polynomials, so
\[
B_K=T_K(A_h).
\]
The upper-right block is similarly
\(hU_{K-1}(A_h)\) for \(K\ge1\), where \(U_j\) is the second-kind Chebyshev polynomial.
In particular,
\[
Q_K=T_K(A_h)Q_0+hU_{K-1}(A_h)P_0 .
\]
Under \(\pi_{H,h}\), the two initial coordinates are independent and
\(\operatorname{Cov}(Q_0)=\alpha H^{-1}\).
Therefore
\[
\operatorname{Cov}(Q_K,Q_0)
=T_K(A_h)\,\alpha H^{-1},
\]
which proves \eqref{eq:quadratic_leapfrog_chebyshev}.

The invariant marginal also gives the exact projected conditional law
\begin{equation}
Q_K\mid Q_0=q
\sim
\mathcal N\!\left(
T_K(A_h)q,\,
\alpha H^{-1}\big(I-T_K(A_h)^2\big)
\right).
\label{eq:quadratic_leapfrog_conditional}
\end{equation}
One can obtain the covariance either by Gaussian conditioning or from
\[
1-T_K(x)^2=(1-x^2)U_{K-1}(x)^2
\]
and \(I-A_h^2=h^2HR_h\).
Since the spectrum of \(A_h\) lies in \((-1,1)\), the covariance is positive semidefinite.
As in the continuous proof, its stationary Gaussian pair is exchangeable because
\(T_K(A_h)\) is symmetric and commutes with \(H^{-1}\).

\paragraph{The finite-step correction is essential.}
The preserved law \(\pi_{H,h}\) retains the original oracle content marginal but changes the momentum covariance from \(\alpha I\) to \(\alpha R_h\).
If instead one initializes with the original full oracle covariance
\(\Gamma_H=\alpha\operatorname{diag}(H^{-1},I)\), then already after one step,
\[
\operatorname{Cov}(Q_1)
=\alpha A_hH^{-1}A_h+\alpha h^2I
=\alpha H^{-1}+\frac{\alpha h^4}{4}H .
\]
For \(h>0\), this differs from \(\alpha H^{-1}\).
Thus exact preservation of the original full oracle lift is not a property of finite-step leapfrog.
An alternative invariant covariance is
\(\alpha\operatorname{diag}((HR_h)^{-1},I)\); that choice preserves the momentum normalization while modifying the content covariance.
The proposition uses \(\Gamma_{H,h}\) so that the Section~\ref{sec:price-of-isotropy-theory} content marginal stays explicit.

\paragraph{Continuous limit and structured locality.}
For an eigenvalue \(\lambda>0\) of \(H\), write
\[
\theta_h(\lambda)
:=
\arccos(1-h^2\lambda/2).
\]
Then \(\theta_h(\lambda)/h\to\sqrt{\lambda}\).
Hence, if \(h\to0\) and \(Kh\to t\),
\[
T_K(1-h^2\lambda/2)
=\cos(K\theta_h(\lambda))
\longrightarrow
\cos(t\sqrt{\lambda}).
\]
The finite-dimensional spectral decomposition gives
\[
T_K(A_h)\longrightarrow\cos(tH^{1/2}),
\qquad
R_h\longrightarrow I,
\]
and the conditional law in \eqref{eq:quadratic_leapfrog_conditional} converges to the continuous oracle conditional.

If a graph records the nonzero off-diagonal entries of \(H\), then
\((H^j)_{uv}=0\) whenever the graph distance between \(u\) and \(v\) exceeds \(j\).
This follows by expanding the matrix product: a nonzero summand requires a walk of at most \(j\) off-diagonal edges.
Since \(T_K(I-h^2H/2)\) is a polynomial in \(H\) of degree at most \(K\), its \((u,v)\) entry vanishes when that distance exceeds \(K\).
This locality statement concerns a sparse quadratic operator \(H\); it does not assert that the learned nonlinear potential in the experiments has this sparsity.
\hfill\(\square\)

\paragraph{Reported numerical oracle checks.}
Numerical checks reported during review covered 18 structured SPD cases, including non-degenerate observable conditionals, relative $H$-recovery error below $8.1\times10^{-8}$, and second-order shadow convergence. The shadow precision used was $H_h=H(I-h^2H/4)$, corresponding to the alternative invariant covariance above. These known-geometry consistency checks do not estimate the downstream operator or the price of isotropy for image representations.

\subsection{Oracle compatibility and downstream performance}
\label{app:oracle_compatibility_vs_performance}

\begin{proposition}[Oracle defect of a rotated force]
\label{prop:rotated_force_oracle_defect}
For \(U\in\mathbb R^{d\times d}\), consider
\[
\dot Q=P,\qquad \dot P=-UHQ,\qquad
L_U=\begin{pmatrix}0&I\\-UH&0\end{pmatrix}.
\]
The resulting continuous linear flow preserves
\(\mathcal N(0,\Gamma_H)\) if and only if \(U=I\).
It is nevertheless volume-preserving and reversible under \(\mathsf R\) for every \(U\).
\end{proposition}

\begin{proof}
The stationarity defect is
\[
L_U\Gamma_H+\Gamma_HL_U^\top
=\alpha
\begin{pmatrix}
0&I-U^\top\\
I-U&0
\end{pmatrix}.
\]
If the flow preserves the oracle covariance, differentiation at zero forces this matrix to vanish, hence \(U=I\).
The converse is Theorem~\ref{thm:canonical_linear_oracle_lift}.
For every \(U\), \(\operatorname{tr}(L_U)=0\), so
\(\det(e^{tL_U})=e^{t\operatorname{tr}(L_U)}=1\).
The off-diagonal block form also gives
\(\mathsf R L_U\mathsf R=-L_U\), and hence reversibility.
\end{proof}

\begin{proposition}[Equal task performance despite oracle incompatibility]
\label{prop:oracle_incompatible_equal_bayes}
Let \(d=3\), \(H=I_3\), and
\[
U=\operatorname{diag}(1,U_2),
\]
where \(U_2\) is a two-dimensional rotation with angle
\(\vartheta\in(0,\pi)\).
Initialize both the Hamiltonian and rotated-force systems from the same oracle law and define
\[
Y=\mathbf 1[Q_{0,1}>0].
\]
For every fixed \(t\), the two systems have the same Bayes-optimal classification error for predicting \(Y\) from \(Q_t\), although the rotated-force flow does not preserve the full oracle law.
\end{proposition}

\begin{proof}
In both systems the first coordinate obeys
\[
\ddot Q_1=-Q_1,\qquad
Q_{t,1}=(\cos t)Q_{0,1}+(\sin t)P_{0,1}.
\]
Its joint law with \(Y\) is therefore identical.
The remaining coordinates evolve within the separate \((2,3)\) block.
Their initial coordinates are independent of
\((Q_{0,1},P_{0,1})\), and this independence is preserved by the deterministic blockwise evolution.
Consequently,
\[
\mathbb P(Y=1\mid Q_t)
=\mathbb P(Y=1\mid Q_{t,1})
\]
in each system, and the conditional law on the right is the same in both.
The Bayes classification error is thus the same.
Since \(U\ne I\), Proposition~\ref{prop:rotated_force_oracle_defect} proves failure of full oracle stationarity.
\end{proof}

This example separates two assertions.
The oracle theorem characterizes which linear coupling satisfies all its assumptions.
It does not imply that violating an assumption must worsen every downstream task.
The example is a population statement about a specified observable task, not a claim that the training dynamics or finite-sample scores of the two learned models must coincide.
Accordingly, the matched controls restrict empirical causal attribution without contradicting the conditional oracle characterization.

\subsection{Canonical form under cotangent coordinate changes}
\label{app:cotangent_form_characterization}

\begin{proposition}[Invariant constant alternating form]
\label{prop:cotangent_canonical_form}
Let \(\Omega\in\mathbb R^{2d\times2d}\) be skew-symmetric.
Suppose that for every \(B\in GL(d)\),
\[
S_B^\top\Omega S_B=\Omega,
\qquad
S_B=
\begin{pmatrix}B&0\\0&B^{-\top}\end{pmatrix}.
\]
Then \(\Omega=\beta J\) for some \(\beta\in\mathbb R\).
The form is nondegenerate if and only if \(\beta\ne0\).
\end{proposition}

\begin{proof}
Write
\[
\Omega=
\begin{pmatrix}
E&G\\
-G^\top&F
\end{pmatrix},
\qquad E^\top=-E,\quad F^\top=-F.
\]
Taking \(B=rI\) with \(r>0\) gives
\[
S_B^\top\Omega S_B
=
\begin{pmatrix}
r^2E&G\\
-G^\top&r^{-2}F
\end{pmatrix}.
\]
Invariance for, say, \(r=2\) forces \(E=F=0\).
The upper-right block of the general invariance identity then requires
\[
B^\top G B^{-\top}=G,
\]
or \(B^\top G=GB^\top\), for every invertible \(B\).
Commutation with all invertible diagonal matrices forces \(G\) to be diagonal.
Commutation with all coordinate permutation matrices then forces its diagonal entries to agree.
Thus \(G=\beta I\) and \(\Omega=\beta J\).
Conversely \(S_B^\top J S_B=J\) by direct multiplication.
Finally, \(J\) is invertible, so \(\beta J\) is nondegenerate exactly when \(\beta\ne0\).
\end{proof}

The proposition concerns constant alternating forms and the declared cotangent action
\((q,p)\mapsto(Bq,B^{-\top}p)\).
It does not characterize arbitrary nonlinear predictors.
Nor does it make the fixed-unit kinetic normalization invariant under all such coordinate changes.
Indeed, writing \(q'=Bq\), \(p'=B^{-\top}p\) transforms
\(\tfrac12\|p\|_2^2\) into
\(\tfrac12 p'^\top BB^\top p'\), and \(\dot q=p\) into
\(\dot q'=BB^\top p'\).
The canonical two-form is preserved, while the kinetic metric changes.
The unit kinetic metric used in the oracle theorem and implementation is therefore a condition in the chosen coordinate chart.

\section{Supplementary theory: maximum entropy on phase space and symplectic factorization}

\begin{theorem}[Maximum entropy on phase space under an expected energy constraint]
\label{thm:maxent_phase_space_gibbs}
Let $\mathcal{H}:\mathbb{R}^{m}\to\mathbb{R}$ be measurable and assume there exists $\beta>0$ such that the
partition function
\[
Z(\beta):=\int_{\mathbb{R}^{m}} e^{-\beta \mathcal{H}(s)}\,ds
\]
is finite. Fix $c\in\mathbb{R}$ and consider the set of densities
\[
\mathcal{P}_c := \left\{p \text{ density on }\mathbb{R}^{m} : \int p(s)\,ds=1,\ \int \mathcal{H}(s)p(s)\,ds=c,\ h(p)>-\infty\right\},
\]
where $h(p):=-\int p\log p$ is differential entropy.
If there exists $\beta^\star>0$ such that $\mathbb{E}_{p_{\beta^\star}}[\mathcal{H}(S)]=c$ for
\[
p_{\beta}(s)=\frac{1}{Z(\beta)}e^{-\beta \mathcal{H}(s)},
\]
then the unique maximizer of $h(p)$ over $\mathcal{P}_c$ is $p^\star=p_{\beta^\star}$.
\end{theorem}

\begin{proof}
Fix $\beta>0$ with $Z(\beta)<\infty$ and define $p_\beta(s)=Z(\beta)^{-1}e^{-\beta \mathcal{H}(s)}$.
For any $p\in\mathcal{P}_c$,
\[
\mathrm{KL}(p\|p_\beta)=\int p(s)\log\frac{p(s)}{p_\beta(s)}\,ds \ge 0,
\]
with equality iff $p=p_\beta$ a.e.
Expanding $\log p_\beta(s)=-\beta\mathcal{H}(s)-\log Z(\beta)$ gives
\begin{align*}
\mathrm{KL}(p\|p_\beta)
&= \int p\log p\,ds - \int p\log p_\beta\,ds \\
&= -h(p) - \int p(s)\big(-\beta\mathcal{H}(s)-\log Z(\beta)\big)\,ds \\
&= -h(p) + \beta \underbrace{\int \mathcal{H}(s)p(s)\,ds}_{=c} + \log Z(\beta)\underbrace{\int p(s)\,ds}_{=1} \\
&= -h(p) + \beta c + \log Z(\beta).
\end{align*}
Thus $h(p)\le \beta c + \log Z(\beta)$ for all $p\in\mathcal{P}_c$, with equality iff $p=p_\beta$.

If there exists $\beta^\star$ such that $p_{\beta^\star}\in\mathcal{P}_c$ (i.e., it matches the energy constraint),
then $p_{\beta^\star}$ attains the upper bound and is therefore a maximizer.
Uniqueness follows from the equality condition for KL divergence.
\end{proof}

\begin{corollary}[Energy-based densities are invariant under Hamiltonian flow]
\label{cor:gibbs_invariant}
Let $\mathcal{H}\in C^2(\mathbb{R}^{2d_0})$ and let $\Phi_t$ be the Hamiltonian flow of
$\dot s = J\nabla \mathcal{H}(s)$, assumed to be a $C^1$ diffeomorphism for the times of interest.
If $S$ has density $p(s)=g(\mathcal{H}(s))$ for some measurable $g:\mathbb{R}\to\mathbb{R}_+$,
then for all $t$, $\Phi_t(S)\stackrel{d}{=}S$.
In particular, the Gibbs density $p(s)\propto e^{-\beta\mathcal{H}(s)}$ is invariant.
\end{corollary}

\begin{proof}
Energy is conserved along Hamiltonian trajectories: if $s(t)=\Phi_t(s_0)$, then
\[
\frac{d}{dt}\mathcal{H}(s(t))
= \nabla \mathcal{H}(s(t))^\top \dot s(t)
= \nabla \mathcal{H}(s(t))^\top J \nabla \mathcal{H}(s(t)) = 0,
\]
since $J^\top=-J$ implies $u^\top Ju=0$ for all $u$.
Thus $\mathcal{H}(\Phi_t(s))=\mathcal{H}(s)$.

Also Hamiltonian flows are volume-preserving (Liouville): $|\det(D\Phi_t(s))|=1$ for all $s,t$.
Let $Y=\Phi_t(S)$. By change of variables,
\[
p_Y(y)=p(\Phi_{-t}(y))\cdot \left|\det\big(D\Phi_{-t}(y)\big)\right|
= g(\mathcal{H}(\Phi_{-t}(y)))\cdot 1
= g(\mathcal{H}(y)) = p(y).
\]
Hence $Y$ and $S$ have the same density.
\end{proof}

\begin{proposition}[Generic kick--drift--scaling factorization of linear symplectic maps]
\label{prop:symplectic_factorization}
Let $A=\begin{psmallmatrix}a&b\\ c&d\end{psmallmatrix}\in\Sp(2d_0)$ and assume the block $d\in\mathbb{R}^{d_0\times d_0}$ is invertible.
Define
\[
B := b d^{-1},\qquad C := d^{-1}c.
\]
Then $B$ and $C$ are symmetric and
\[
A
=
\underbrace{\begin{pmatrix}I&B\\0&I\end{pmatrix}}_{\text{upper shear (drift)}}
\underbrace{\begin{pmatrix}d^{-\top}&0\\0&d\end{pmatrix}}_{\text{symplectic scaling}}
\underbrace{\begin{pmatrix}I&0\\C&I\end{pmatrix}}_{\text{lower shear (kick)}}.
\]
\end{proposition}

\begin{proof}
Since $A\in\Sp(2d_0)$, it satisfies $A^\top J A=J$, which implies the standard block identities
$b^\top d = d^\top b$ and $a^\top d - c^\top b = I$.
From $b^\top d = d^\top b$ and invertibility of $d$,
\[
B^\top = (bd^{-1})^\top = d^{-\top}b^\top = bd^{-1}=B,
\]
so $B$ is symmetric.

Similarly, $A J A^\top = J$ also holds for symplectic $A$, and its block identities include $c d^\top = d c^\top$.
Then
\[
C^\top = (d^{-1}c)^\top = c^\top d^{-\top} = d^{-1}c = C,
\]
so $C$ is symmetric.

Now multiply the three proposed factors:
\[
\begin{pmatrix}I&B\\0&I\end{pmatrix}
\begin{pmatrix}d^{-\top}&0\\0&d\end{pmatrix}
\begin{pmatrix}I&0\\C&I\end{pmatrix}
=
\begin{pmatrix}d^{-\top}+B d C & B d\\ d C & d\end{pmatrix}.
\]
The bottom-right block matches $d$.
The top-right block matches $b$ because $Bd = (bd^{-1})d=b$.
The bottom-left block matches $c$ because $dC = d(d^{-1}c)=c$.

For the top-left block, use the symplectic identity $a d^\top - b c^\top = I$ (from $A J A^\top=J$)
and right-multiply by $d^{-\top}$:
\[
a - b c^\top d^{-\top} = d^{-\top}.
\]
Using $c^\top d^{-\top} = (d^{-1}c)^\top = C^\top = C$, this becomes
\[
a = d^{-\top} + bC = d^{-\top} + (Bd)C = d^{-\top} + B d C,
\]
which matches the computed top-left block. Hence the factorization holds.
\end{proof}

\section{Further results and figures}
\label{app:res}

We refer to the method uniformly as \emph{HamJEPA}; in the code/config files, the suffix \texttt{hjepa\_mv} simply denotes the minimal multi-view implementation used in our experiments.

\subsection{Frozen-feature diagnostic summaries}
\label{app:diag_summaries}

Figures~\ref{fig:diag_grid_raw} (CIFAR-100) and~\ref{fig:diag_grid_inet} (ImageNet-100) summarize the geometry and non-parametric usefulness of \emph{frozen} encoder representations for the SIGReg baseline and HamJEPA latent states. 
Each method panel contains four diagnostics computed on \emph{raw} encoder features (i.e., no whitening or other post-processing):
\begin{enumerate}
  \item \textbf{$k$NN sweep:} top-1 accuracy as a function of neighborhood size $k$.
  \item \textbf{Covariance eigenspectrum:} top-256 eigenvalues (log-scale) of the \emph{mean-centered}
        empirical covariance.
  \item \textbf{Random-pair cosine similarities:} histogram of cosine similarities after per-vector
        $\ell_2$ normalization \emph{without} mean subtraction.
  \item \textbf{Feature norms:} histogram of $\ell_2$ feature norms.
\end{enumerate}

A key subtlety is that the covariance spectrum is computed after \emph{mean-centering}, whereas the cosine histogram is computed on normalized features \emph{without} centering. Consequently, strongly positive random-pair cosines can arise from a dominant \emph{mean direction} even when the mean-centered covariance remains high-rank. We exploit this complementarity in the discussions below.

\subsection{Discussion CIFAR-100}
\label{app:diag_discussion}

\subsubsection{30 epoch results}
These results are for the 30 epoch pretraining.

\begin{figure*}[p]
  \centering

  \begin{subfigure}[t]{0.5\textwidth}
    \centering
    \includegraphics[width=\linewidth]{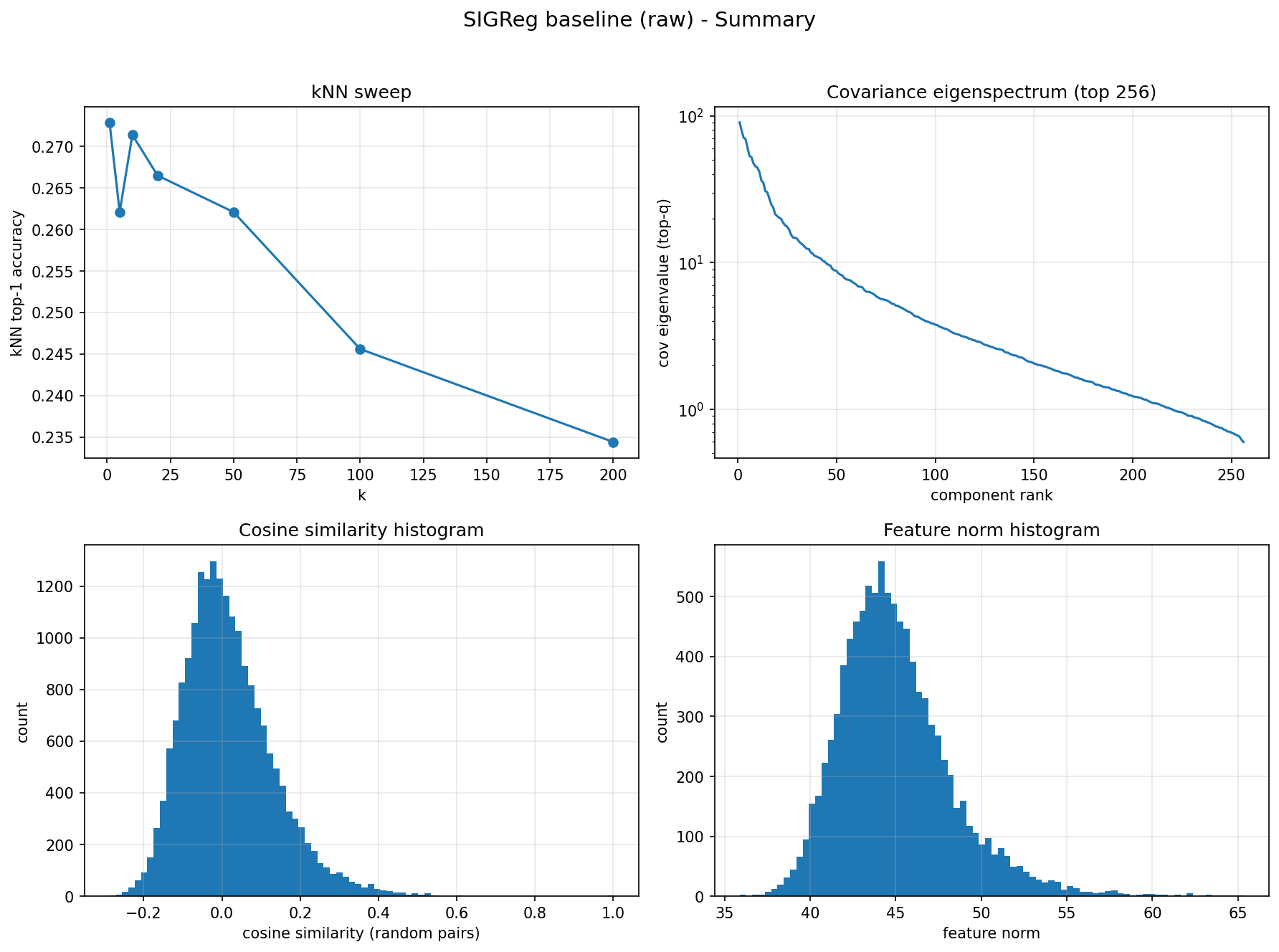}
    \caption{SIGReg (raw).}
    \label{fig:diag_sigreg_raw}
  \end{subfigure}\hfill
  \begin{subfigure}[t]{0.5\textwidth}
    \centering
    \includegraphics[width=\linewidth]{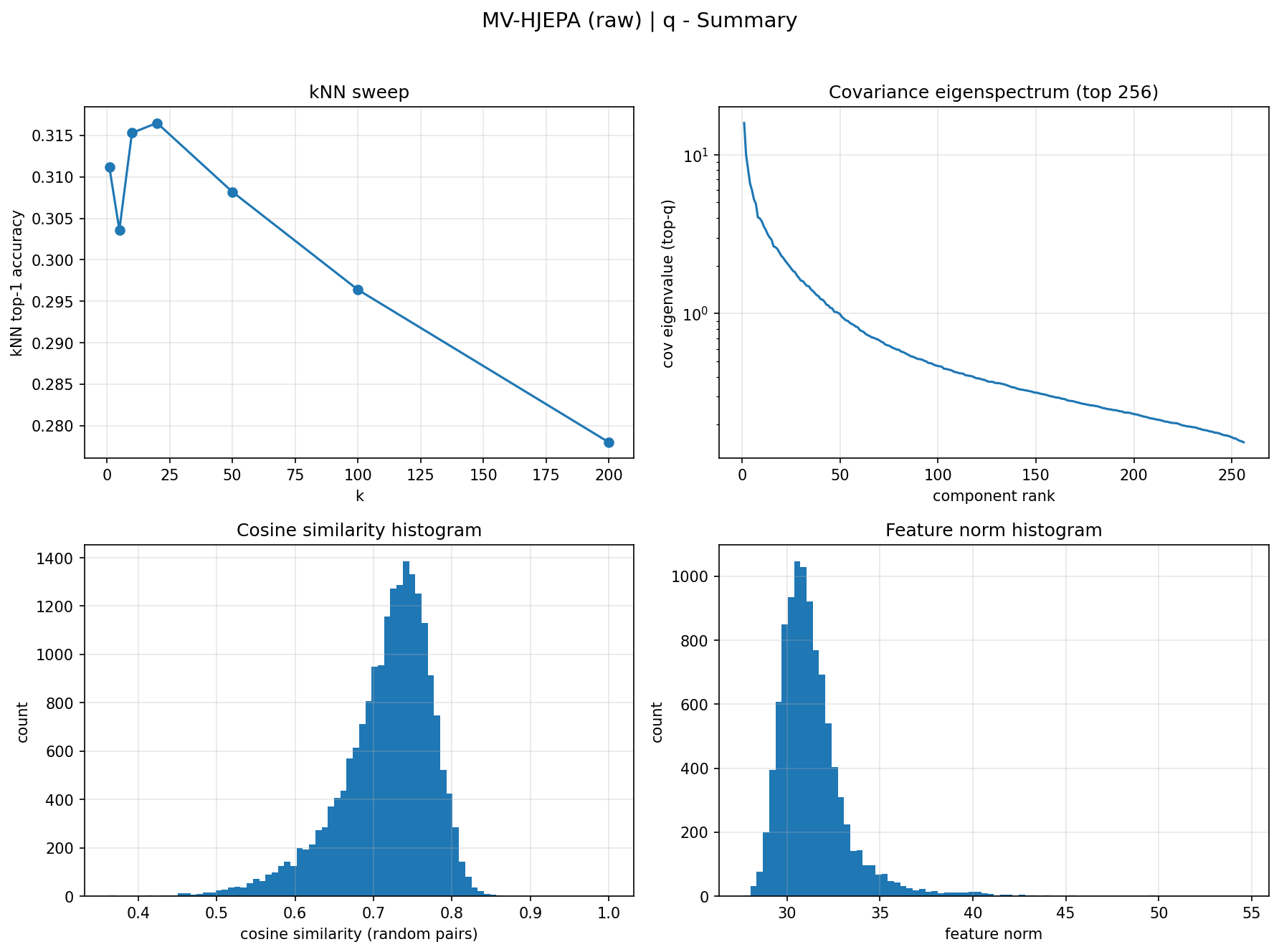}
    \caption{HamJEPA $q$ (raw).}
    \label{fig:diag_mv_hjepa_q_raw}
  \end{subfigure}

  \vspace{0.6em}

  \begin{subfigure}[t]{0.5\textwidth}
    \centering
    \includegraphics[width=\linewidth]{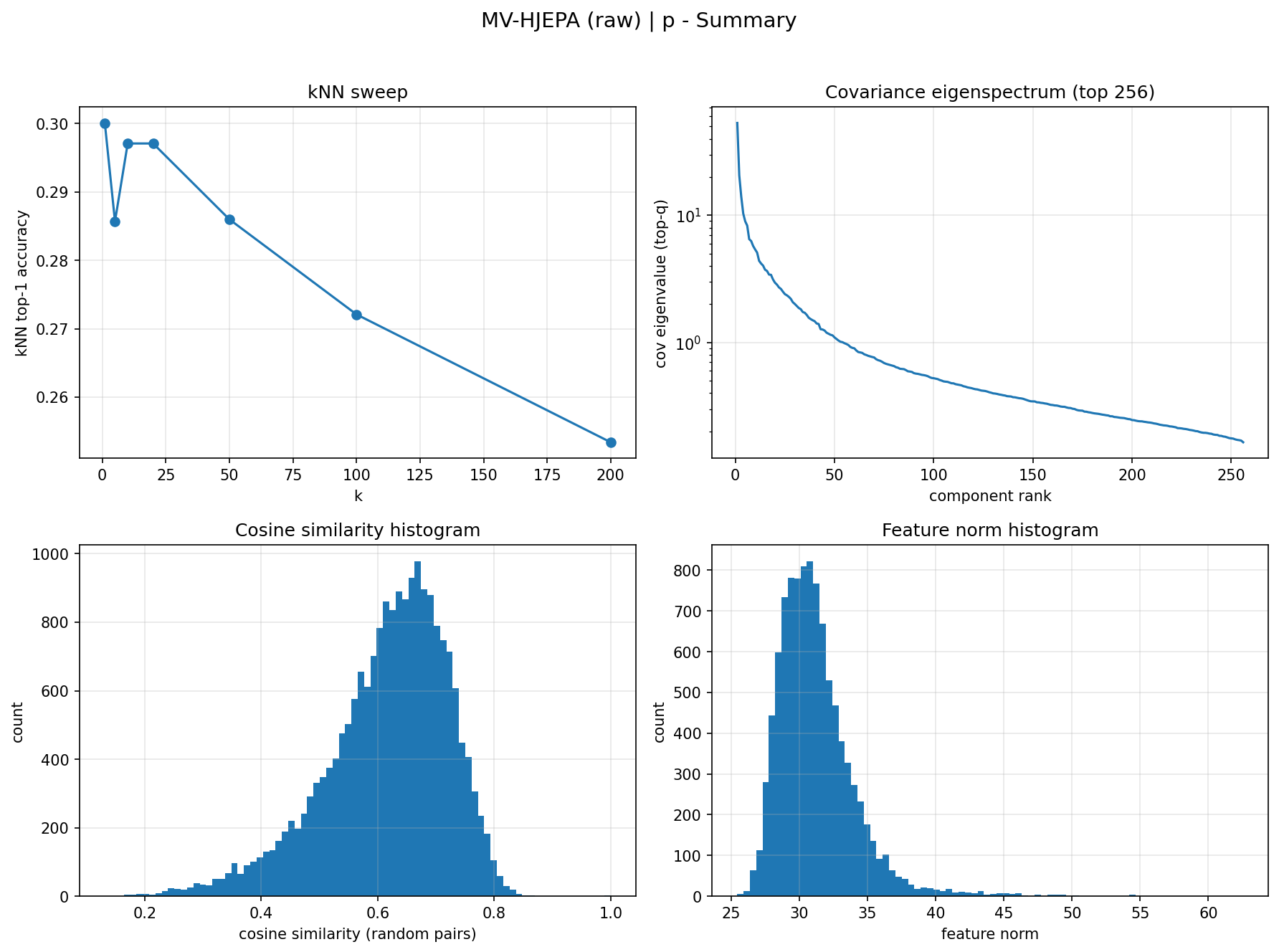}
    \caption{HamJEPA $p$ (raw).}
    \label{fig:diag_mv_hjepa_p_raw}
  \end{subfigure}\hfill
  \begin{subfigure}[t]{0.5\textwidth}
    \centering
    \includegraphics[width=\linewidth]{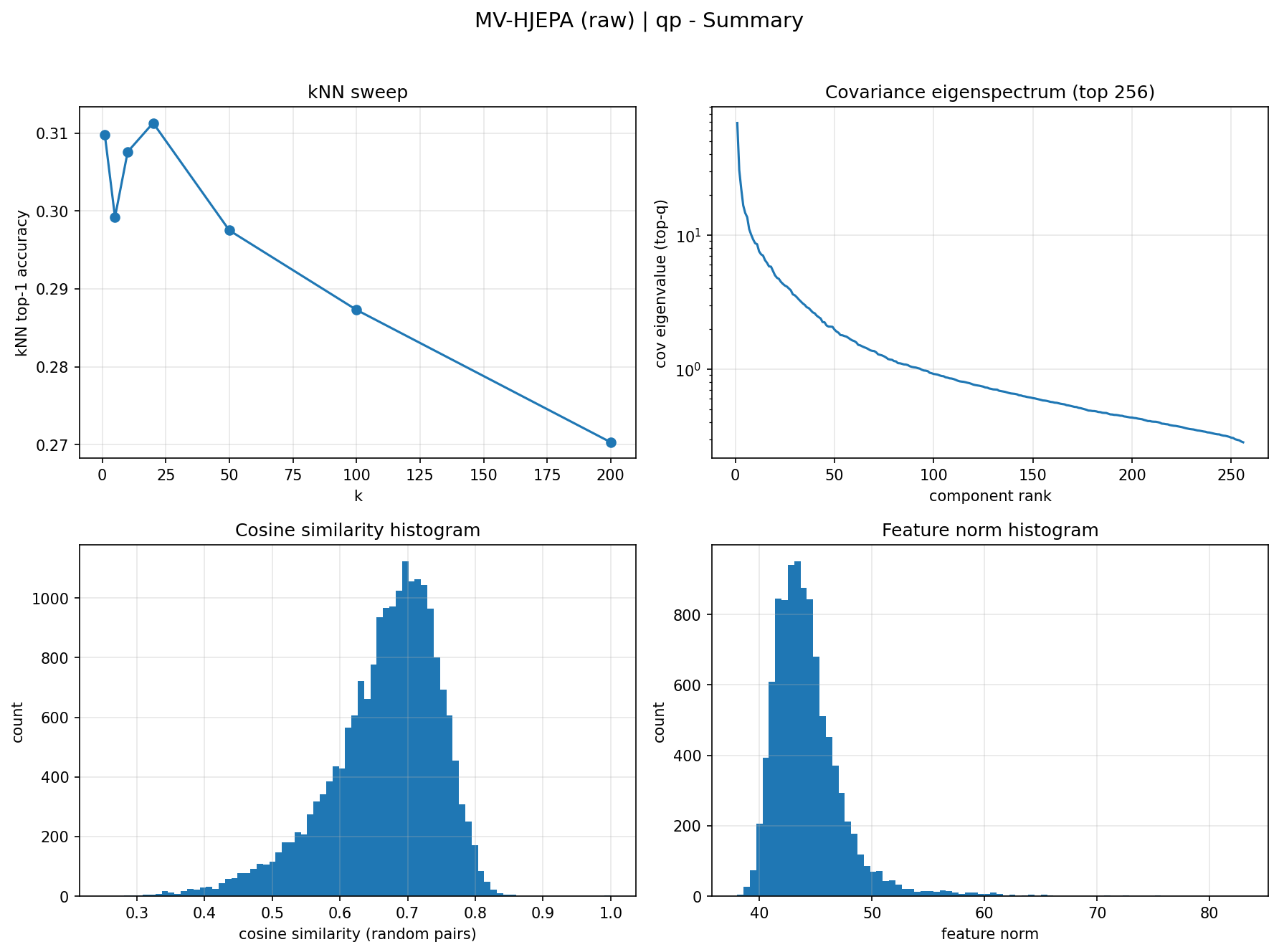}
    \caption{HamJEPA $(q,p)$ (raw).}
    \label{fig:diag_mv_hjepa_qp_raw}
  \end{subfigure}

  \caption{\textbf{Frozen-feature diagnostic summaries (raw) 30 epochs.}
  Each panel contains: (top-left) $k$NN sweep; (top-right) covariance eigenspectrum (top 256, log scale);
  (bottom-left) cosine similarity histogram (random pairs); (bottom-right) feature norm histogram.}
  \label{fig:diag_grid_raw}
\end{figure*}

\paragraph{(1) Neighborhood quality: $k$NN sweeps.}
In Fig.~\ref{fig:diag_sigreg_raw}, the SIGReg baseline peaks at small $k$ and degrades as $k$ grows, consistent with semantics being locally present but becoming less class-pure as neighborhoods expand.
Quantitatively, SIGReg peaks at $27.46 \pm 0.31\%$ (at $k{=}1$) and falls to $23.10 \pm 0.32\%$ (at $k{=}200$).

HamJEPA improves the entire sweep for all latent variants in Fig.~\ref{fig:diag_grid_raw}. Notably, the \emph{single} $q$ state (1024-D) outperforms the SIGReg baseline (2048-D) across the plotted $k$ values: $31.45 \pm 0.27\%$ at $k{=}20$ for $q$, versus $26.56 \pm 0.18\%$ for SIGReg at $k{=}20$ (a $+4.88 \pm 0.11$ point gain). The $p$ and $(q,p)$ variants achieve $29.71 \pm 0.25\%$ and $30.88 \pm 0.39\%$ at $k{=}20$, respectively.
That $(q,p)$ does not exceed $q$ suggests that, under the default weighting, $p$ contributes less additional \emph{metric-consistent} neighborhood structure than $q$, even though it remains informative for linear decoding (see below).

\paragraph{(2) Linear probe vs.\ $k$NN: separability without neighborhood purity.}
While $k$NN distinguishes the latent variants, the linear probe largely does not: SIGReg reaches
$30.43 \pm 0.30\%$ top-1, whereas HamJEPA reaches $33.95 \pm 0.24\%$, $33.07 \pm 0.36\%$ and $34.18 \pm 0.17\%$ top-1 for $q$, $p$, and $(q,p)$ respectively.
A useful interpretation is that $p$ is \emph{linearly} discriminative but induces less clean neighborhoods
under the raw metric, so its gains show up more clearly in the probe than in $k$NN.

\paragraph{(3) Covariance eigenspectrum: rank proxies and collapse checks.}
The top-right panels show mean-centered covariance spectra (top 256 eigenvalues, log-scale). A convenient summary
is the effective-rank proxy computed from these eigenvalues: SIGReg has effective rank $\approx 109.73 \pm 0.61$ (top-256),
whereas HamJEPA-$q$ increases to $125.58 \pm 3.33$. In contrast, HamJEPA-$p$ is more concentrated
(effective rank $79.46 \pm 3.99$), and $(q,p)$ lies in between ($93.10 \pm 4.00$).
Two points follow:
(i) HamJEPA-$q$ does \emph{not} achieve its downstream gain via low-rank collapse (the centered covariance is
high-rank), and
(ii) the $q/p$ split is geometrically meaningful: $q$ and $p$ allocate variance differently even when their linear
probe scores are similar.

\paragraph{(4) Cosine similarities: diagnosing mean-dominance.}
The bottom-left panels reveal the largest qualitative difference.
SIGReg’s random-pair cosines are centered near zero (mean $ 0.01699 \pm 0.00089$) with substantial negative mass,
which is consistent with its isotropy-encouraging design.

HamJEPA’s cosines shift sharply positive: mean cosine is $ 0.71561 \pm 0.00129$ for $q$, $0.61687 \pm 0.00409$ for $p$, and
$0.66613 \pm 0.00270$ for $(q,p)$.
Crucially, this strong alignment coexists with a \emph{high-rank mean-centered covariance} (especially for $q$).
This combination is most naturally explained by a large \emph{non-zero mean feature direction} that dominates
dot products when features are not explicitly centered: the covariance spectrum removes that mean direction by
construction, while the cosine histogram does not.
In other words, the CIFAR-100 HamJEPA embeddings look ``cone-like'' in raw normalized space primarily due to
mean-dominance rather than variance collapse.

\paragraph{(5) Feature norms: scale control and the $(q,p)$ energy split.}
SIGReg feature norms are concentrated around $44.891 \pm 0.090$, whereas HamJEPA’s $q$ and $p$ have smaller typical
norms around $31.358 \pm 0.017$ and $31.281 \pm 0.029$.
The concatenated $(q,p)$ norm shifts back to $44.303 \pm 0.027$, closely matching the Pythagorean combination
$\|[q;p]\|_2 \approx \sqrt{\|q\|_2^2 + \|p\|_2^2}$, which is what one expects if $q$ and $p$ maintain comparable
energy rather than one collapsing to near-zero.
This provides a simple consistency check that the phase-space parameterization retains a balanced contribution
from both coordinates.

\subsubsection{Original 80-epoch single-run results}
The following figures, table, and diagnostics retain the original single-run comparison. The replicated controls are reported separately in Appendix~\ref{app:replicated_controls}.

MLP-HJEPA keeps the phase-space encoder, channel-wise $q/p$ split, identity projector, prediction loss, anti-collapse regularizers, optimizer, batch size, and training schedule fixed, and replaces only the symplectic leapfrog map with a non-symplectic residual MLP predictor.
This ablation is intentionally strong: MLP-HJEPA reaches high linear accuracy (42.26--42.77), showing that the phase-space split and regularizers can produce linearly useful features.
However, its $k$NN@20 accuracy remains essentially at the SIGReg level (28.10--28.30 vs.\ 27.98).
By contrast, HamJEPA with the symplectic predictor reaches 34.43 $k$NN@20 and 44.59 linear accuracy using the $q$ readout.
In this single-run comparison, replacing the MLP by the Hamiltonian predictor yields a $+6.33$ point $k$NN@20 gain over MLP-HJEPA-$q$ and a $+6.45$ point gain over full-dimensional SIGReg at the same 80-epoch budget. The matched controls below separate this predictor-family comparison from claims about individual structural properties.

By design, the CIFAR-100 pretraining curves tend to plateau because the experiment is set up as a minimal, controlled training regime, whose goal is to isolate representational geometry effects rather than to chase absolute accuracy with an aggressively optimized recipe. CIFAR-100 is small and low-resolution, so the encoder typically learns the dominant semantic invariances early; after that point, additional epochs mainly yield diminishing returns (small refinements in neighborhood structure or feature scaling) that often do not translate into large gains in linear-probe/(k)NN top-1 under a fixed evaluation protocol. In other words, once both methods reach a “good-enough” feature space, the remaining headroom is mostly governed by recipe choices we intentionally avoid here (stronger augmentation, longer schedules, larger models, distillation/EMA, more hyperparameter search), and their absence limits the external validity of these results.

\begin{figure*}[p]
  \centering

  \begin{subfigure}[t]{0.5\textwidth}
    \centering
    \includegraphics[width=\linewidth]{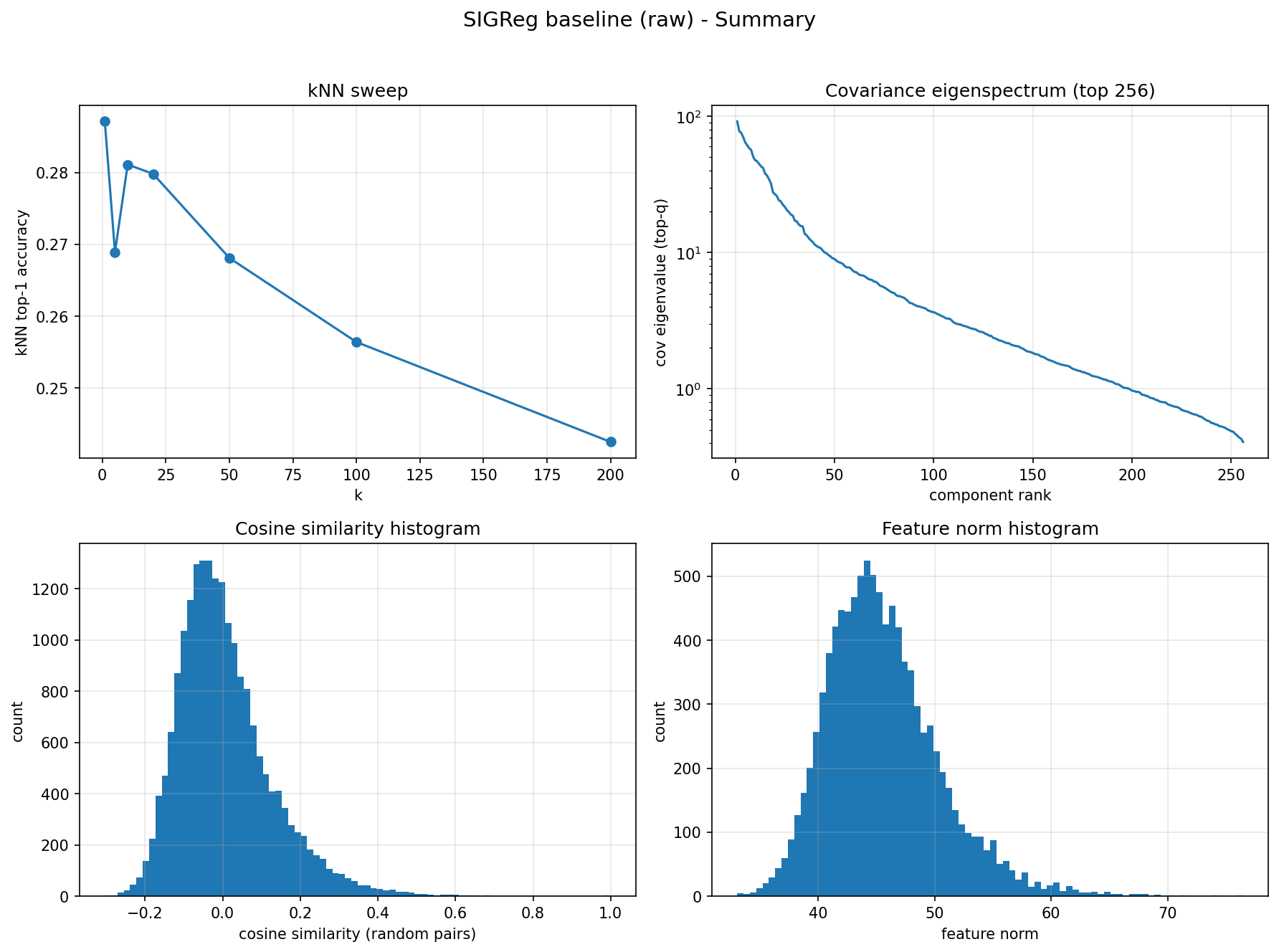}
    \caption{SIGReg (raw).}
    \label{fig:diag_sigreg_raw_80}
  \end{subfigure}\hfill
  \begin{subfigure}[t]{0.5\textwidth}
    \centering
    \includegraphics[width=\linewidth]{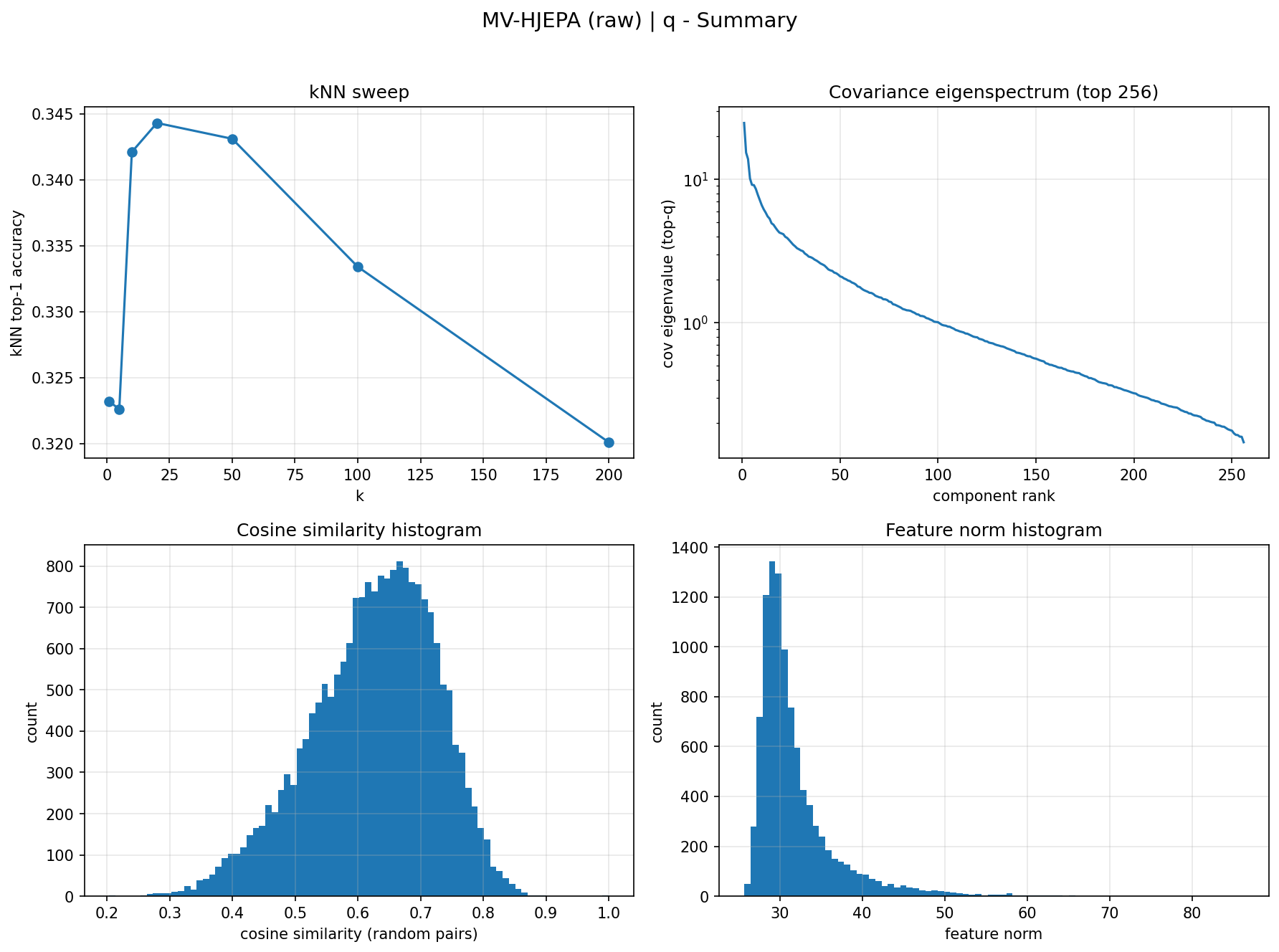}
    \caption{HamJEPA $q$ (raw).}
    \label{fig:diag_mv_hjepa_q_raw_80}
  \end{subfigure}

  \vspace{0.6em}

  \begin{subfigure}[t]{0.5\textwidth}
    \centering
    \includegraphics[width=\linewidth]{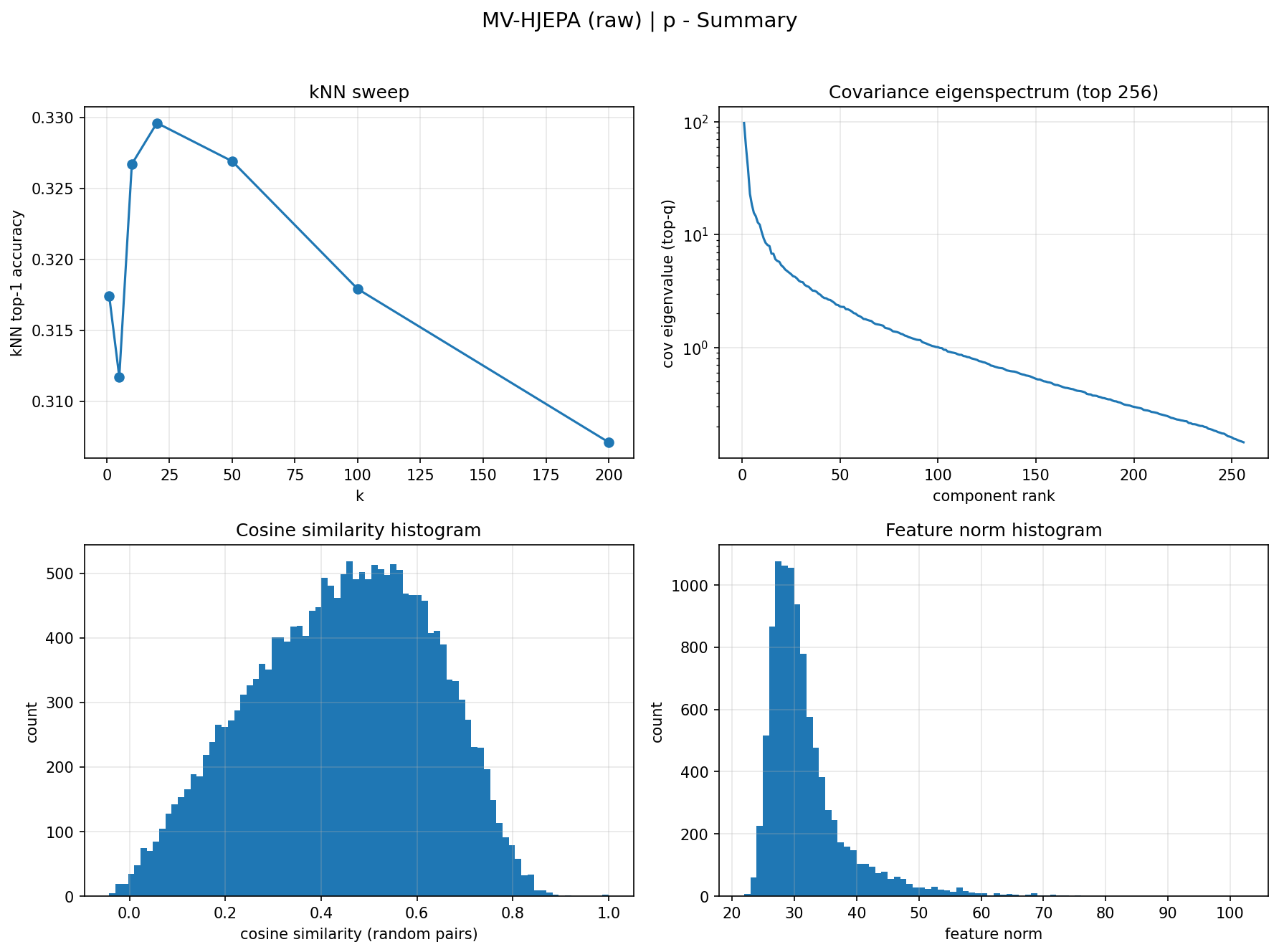}
    \caption{HamJEPA $p$ (raw).}
    \label{fig:diag_mv_hjepa_p_raw_80}
  \end{subfigure}\hfill
  \begin{subfigure}[t]{0.5\textwidth}
    \centering
    \includegraphics[width=\linewidth]{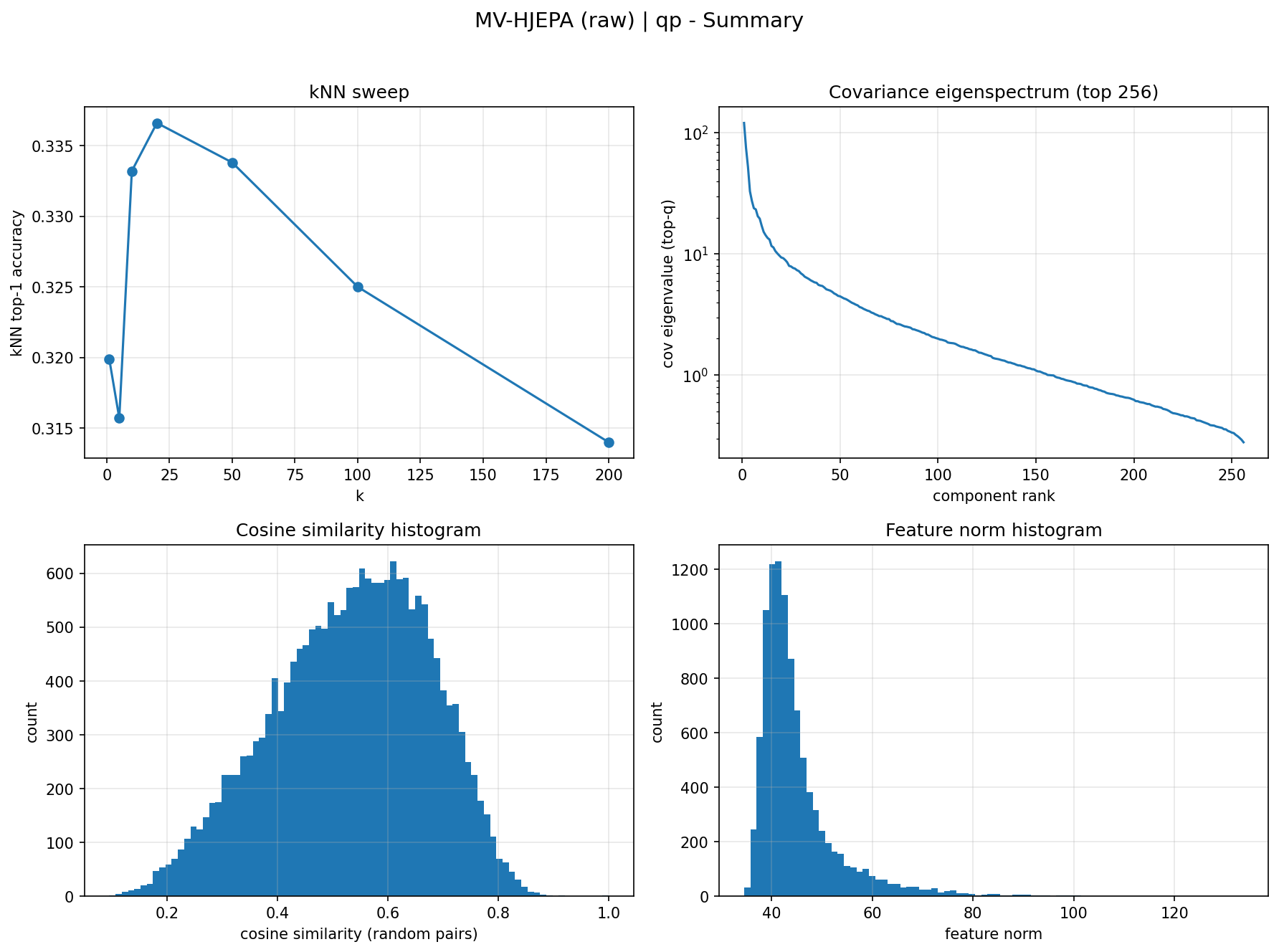}
    \caption{HamJEPA $(q,p)$ (raw).}
    \label{fig:diag_mv_hjepa_qp_raw_80}
  \end{subfigure}

  \caption{\textbf{CIFAR-100 frozen-feature diagnostic summaries (raw), original single-run 80-epoch results.}
  Each panel contains: (top-left) $k$NN sweep; (top-right) mean-centered covariance eigenspectrum (top 256, log scale);
  (bottom-left) random-pair cosine similarity histogram; (bottom-right) feature norm histogram.}
  \label{fig:diag_grid_raw_80}
\end{figure*}

\begin{table}[H]
\centering
\small
\begin{tabular}{llccc}
\toprule
Method & Variant & Linear probe (\%) & $k$NN@20 (\%) & $k$NN best (\%) \\
\midrule
SIGReg & single & 33.95 & 27.98 & 28.71 @ $k{=}1$ \\
\midrule
MLP-HJEPA (non-symplectic) & $q$ & 42.26 & 28.10 & 28.10 @ $k{=}20$ \\
MLP-HJEPA (non-symplectic) & $p$ & 41.90 & 28.30 & 28.30 @ $k{=}20$ \\
MLP-HJEPA (non-symplectic) & $(q,p)$ & 42.77 & 28.15 & 28.15 @ $k{=}20$ \\
\midrule
HamJEPA (symplectic) & $q$ & 44.59 & 34.43 & 34.43 @ $k{=}20$ \\
HamJEPA (symplectic) & $p$ & 44.44 & 32.96 & 32.96 @ $k{=}20$ \\
HamJEPA (symplectic) & $(q,p)$ & 44.52 & 33.66 & 33.66 @ $k{=}20$ \\
\bottomrule
\end{tabular}
\caption{\textbf{CIFAR-100, original single-run 80-epoch predictor-family ablation.}
MLP-HJEPA keeps the phase-space encoder, channel-wise $q/p$ split, identity projector, prediction loss, anti-collapse regularizers, optimizer, batch size, and 80-epoch schedule fixed, and replaces only the symplectic leapfrog predictor by a non-symplectic residual MLP predictor.}
\label{tab:cifar80_predictor_ablation}
\end{table}

\paragraph{Predictor-family ablation: what changes and what stays fixed.}
In addition to SIGReg and symplectic HamJEPA, we evaluate a non-symplectic predictor ablation, \emph{MLP-HJEPA}, that keeps the phase-space encoder, channel-wise $q/p$ split, identity projector, prediction loss, anti-collapse regularizers, optimizer, batch size, and training schedule fixed, and replaces only the leapfrog Hamiltonian rollout by a generic residual MLP rollout predictor. Because Fig.~\ref{fig:diag_grid_raw_80} visualizes only SIGReg and symplectic HamJEPA, we summarize the MLP ablation numerically in Table~\ref{tab:cifar80_predictor_ablation}. This comparison changes the predictor family and does not separately isolate symplecticity, reversibility, or volume preservation.

\paragraph{(1) Neighborhood quality: $k$NN sweeps and local semantic structure.}
At 80 epochs, SIGReg reaches $27.98\%$ at $k$NN@20 and peaks at $28.71\%$ over the full sweep (at $k{=}1$). It peaks at small $k$ and degrades as $k$ grows, consistent with semantics being locally present but becoming less class-pure as neighborhoods expand. The non-symplectic MLP-HJEPA ablation barely changes this picture: its best $k$NN@20 is only $28.30\%$ (from the $p$ readout), a gain of just $+0.32$ points over SIGReg, and its best-over-$k$ score is in fact slightly \emph{below} SIGReg ($28.30\%$ vs.\ $28.71\%$). By contrast, symplectic HamJEPA is qualitatively different: the $q$ readout reaches $34.43\%$ at $k$NN@20 and also attains the best-over-$k$ value of $34.43\%$ (at $k{=}20$), namely gains of $+6.45$ points at $k$NN@20 and $+5.72$ points best-over-$k$ relative to SIGReg. 
Notably, the \emph{single} $q$ state (1024-D) outperforms the SIGReg baseline (2048-D) across the plotted $k$ values.
Best-vs.-best, the symplectic predictor outperforms the MLP predictor by $+6.13$ points at $k$NN@20. This is evidence of a predictor-family difference in neighborhood quality; the matched controls in Appendix~\ref{app:replicated_controls} do not identify symplecticity as its cause.

\paragraph{(2) Linear probe: phase-space structure and predictor-family comparison.}
The linear-probe comparison is more nuanced. MLP-HJEPA already substantially outperforms SIGReg, reaching $42.77\%$ top-1 in its best variant ($(q,p)$) versus $33.95\%$ for SIGReg, a gain of $+8.82$ points. Symplectic HamJEPA pushes this further to $44.59\%$ (best $q$ readout), i.e.\ another $+1.82$ points over MLP-HJEPA and $+10.64$ over SIGReg. This shows that the \emph{phase-space parameterization plus encoder-side anti-collapse bridge} can give linearly useful features with a non-symplectic predictor; it does not isolate the contribution of each component. However, once nonparametric neighborhood quality is examined, the predictor family becomes decisive: the MLP predictor narrows much of the probe gap, but essentially none of the $k$NN gap.

\paragraph{(3) $q/p$ decomposition: structured phase space versus generic feature blocks.}
Within MLP-HJEPA, the three readouts are almost indistinguishable: linear probe is $42.26\%$/$41.90\%$/$42.77\%$ for $q$/$p$/$ (q,p)$, and $k$NN@20 is $28.10\%$/$28.30\%$/$28.15\%$. In this MLP comparison, the two blocks have similar downstream scores. This contrasts with symplectic HamJEPA, where $q$ is the cleanest neighborhood readout, $p$ is slightly weaker, and $(q,p)$ remains close but does not dominate. These readout differences are descriptive and do not establish that symplecticity causes functional differentiation.

\paragraph{(4) Mean-centered covariance geometry: high-rank symplectic $q$ versus concentrated MLP states.}
The scalar covariance diagnostics sharpen this distinction. SIGReg has effective rank $\approx 99.80$ and participation ratio $\approx 58.32$ on the top-256 eigenspectrum. Symplectic HamJEPA-$q$ is substantially broader still, with effective rank $\approx 128.98$ and participation ratio $\approx 70.34$, confirming that its downstream gains are \emph{not} coming from collapse onto a narrow subspace. By contrast, MLP-HJEPA is much more spectrally concentrated: its effective rank is only $\approx 42.06$ for $q$, $\approx 43.95$ for $p$, and $\approx 42.85$ for $(q,p)$, with participation ratios only $\approx 23.65$--$24.99$. So although the MLP ablation yields good linear-probe accuracy, it does so with a much more concentrated mean-centered covariance geometry than symplectic HamJEPA. Put differently: the MLP predictor can produce a linearly usable representation, but it does not organize a broad, high-rank content geometry in the way the symplectic predictor does.

\paragraph{(5) Cosine similarities and feature norms: the predictor changes geometry, not just scale.}
The random-pair cosine statistics again place MLP-HJEPA much closer to SIGReg than to symplectic HamJEPA. SIGReg has cosine mean $\approx 0.006$ with standard deviation $\approx 0.121$, while MLP-HJEPA remains near zero as well (cosine means $\approx 0.042$ for $q$, $\approx 0.053$ for $p$, and $\approx 0.047$ for $(q,p)$, all with standard deviation around $0.19$). Symplectic HamJEPA, in contrast, exhibits strongly positive cosine means ($\approx 0.625$ for $q$, $\approx 0.443$ for $p$, and $\approx 0.534$ for $(q,p)$), reflecting the markedly more aligned raw geometry visible in Fig.~\ref{fig:diag_grid_raw_80}. Importantly, this difference is \emph{not} explained by trivial scale effects: the feature norms of MLP-HJEPA are nearly the same as those of symplectic HamJEPA (means $\approx 32.21$, $31.89$, and $45.32$ for MLP-HJEPA versus $\approx 31.81$, $31.88$, and $45.06$ for symplectic HamJEPA). Thus the predictor family is primarily changing \emph{how} variance is organized and aligned, not simply how large the features are.

\paragraph{Takeaway (CIFAR-100, 80 epochs).}
The MLP ablation shows that the phase-space encoder, $q/p$ split, and non-isotropic anti-collapse regularization can produce linearly useful representations with a non-symplectic predictor. In this original comparison, HamJEPA produces much better $k$NN structure, a broader and less concentrated mean-centered covariance for the content coordinate, and a more distinctive raw embedding geometry. The replicated force-rotation and non-reversible controls in Appendix~\ref{app:replicated_controls} retain the gain, so neither symplecticity nor reversibility is identified as individually necessary.

\subsection{Discussion ImageNet-100}
\label{app:diag_discussion_imagenet}

\paragraph{(1) Frozen-feature quality: $k$NN and probe.}
On ImageNet-100, HamJEPA again improves non-parametric and linear evaluation.
In Fig.~\ref{fig:diag_grid_inet}, the dimension-matched comparison (2048-D) is SIGReg $(q,p)$ versus HamJEPA $(q,p)$.
At the final checkpoint, SIGReg $(q,p)$ reaches a best $k$NN top-1 of $21.7\%$ (peak near $k{=}50$),
whereas HamJEPA $(q,p)$ reaches $25.9\%$ (a $+4.2$ point gain).
For linear probing, SIGReg $(q,p)$ reaches $25.5\%$ top-1, whereas HamJEPA $(q,p)$ reaches $32.0\%$ (a $+6.5$ point gain).

Figure~\ref{fig:imagenet_epoch_dynamics} additionally provides a strict 1024-D comparison on $q$:
at epoch 45, HamJEPA-$q$ improves linear probe by $+7.52$ points and best $k$NN by $+4.90$ points over SIGReg-$q$
(as annotated).

\paragraph{(2) Covariance rank: HamJEPA spreads variance across many directions.}
The covariance spectra (top-right) and the effective-rank proxy show a stark contrast:
SIGReg-$q$ and SIGReg-$p$ have low effective rank ($\approx 38$--$39$ on the top-256 eigenvalues), whereas
HamJEPA-$q$ and HamJEPA-$p$ are much higher ($\approx 95$--$98$).
Thus, on ImageNet-100, HamJEPA’s gains are associated with \emph{more} variance being spread across principal
directions (higher rank), not less.

\paragraph{(3) Cosine similarities: mild global alignment without extreme cone structure.}
Unlike CIFAR-100, ImageNet-100 does \emph{not} show an extreme cone effect for HamJEPA.
SIGReg’s random-pair cosines remain centered near zero (mean $\approx 0$), while HamJEPA shifts to moderately
positive cosines (mean $\approx 0.27$--$0.28$ depending on $q/p/(q,p)$).
Given the simultaneously high effective rank of the mean-centered covariance, this again suggests a shared
direction and/or non-zero mean, but the effect is substantially weaker than on CIFAR-100.

\paragraph{(4) $q/p$ decomposition: complementarity for SIGReg, redundancy for HamJEPA.}
Figure~\ref{fig:imagenet_qp_decomp} makes the structural difference explicit.
For SIGReg+tokens, $(q,p)$ is consistently best throughout training, indicating that $q$ and $p$ halves are not
individually sufficient but become competitive when concatenated.
For HamJEPA, $q$ and $p$ are individually as strong as $(q,p)$ (both for probe and $k$NN), suggesting that both
coordinates carry similar semantic content and that concatenation adds little.

\paragraph{(5) Dynamics: accuracy tracks geometry for HamJEPA, not for SIGReg.}
In Fig.~\ref{fig:imagenet_epoch_dynamics}, HamJEPA’s effective rank increases over training while probe accuracy
improves, yielding a strong positive association in the probe-versus-rank plot (annotated $\rho$).
SIGReg shows the opposite: effective rank decreases while accuracy rises slightly, yielding a negative association.
This supports the interpretation that HamJEPA’s ImageNet gains co-occur with a progressively less concentrated
(mean-centered) covariance structure.

\paragraph{Takeaway (ImageNet-100).}
On ImageNet-100, HamJEPA improves $k$NN and linear probing while producing substantially higher-rank mean-centered
covariance than SIGReg and only moderately more positive cosine alignment. In contrast to SIGReg+tokens, HamJEPA
makes both $q$ and $p$ individually strong, with little benefit from concatenation.

\begin{figure*}[!t]
  \centering

  \begin{subfigure}[t]{0.49\textwidth}
    \centering
    \includegraphics[width=\linewidth]{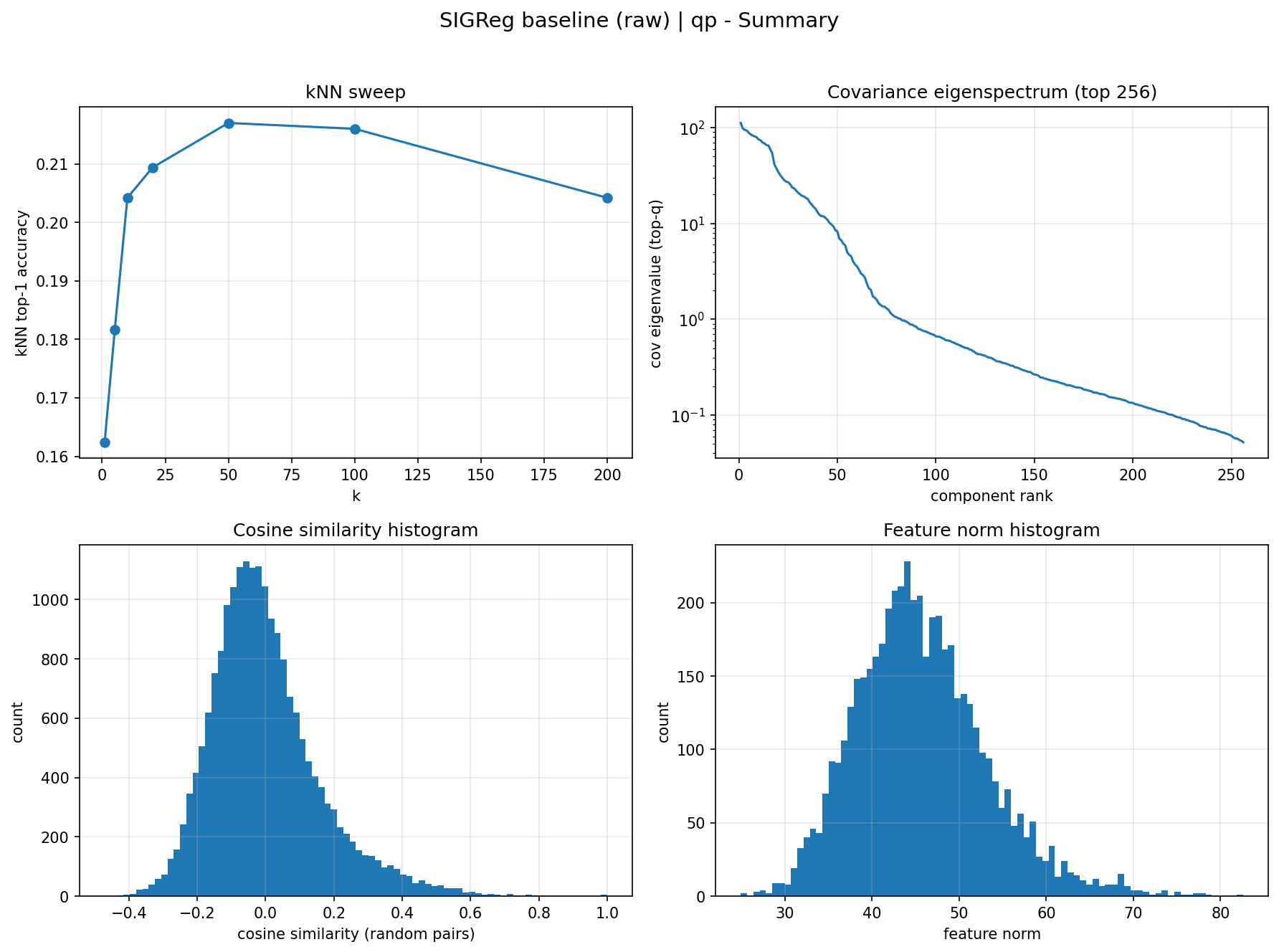}
    \caption{SIGReg (raw).}
    \label{fig:diag_sigreg_raw_inet}
  \end{subfigure}\hfill
  \begin{subfigure}[t]{0.49\textwidth}
    \centering
    \includegraphics[width=\linewidth]{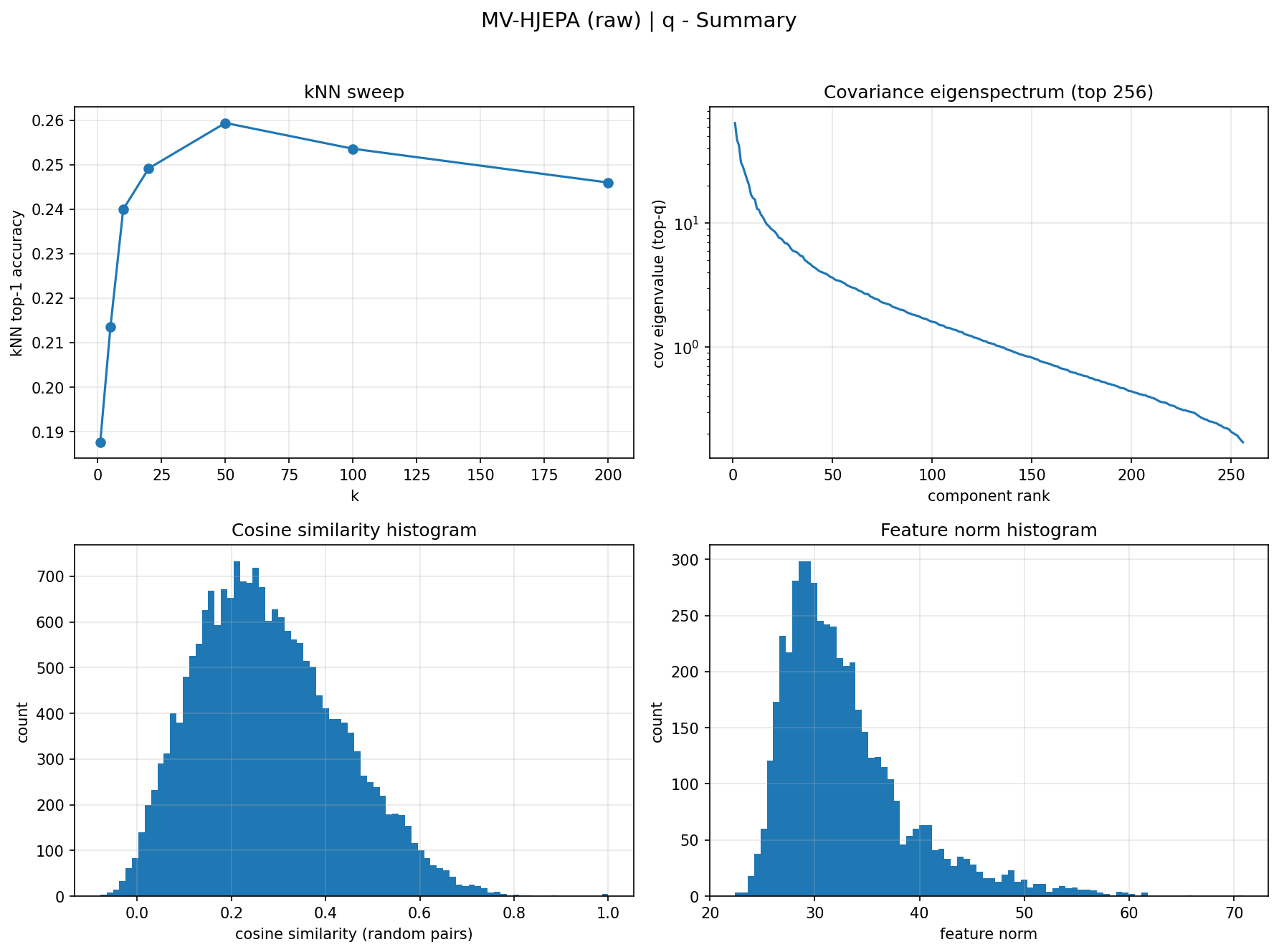}
    \caption{HamJEPA $q$ (raw).}
    \label{fig:diag_mv_hjepa_q_raw_inet}
  \end{subfigure}

  \vspace{0.6em}

  \begin{subfigure}[t]{0.49\textwidth}
    \centering
    \includegraphics[width=\linewidth]{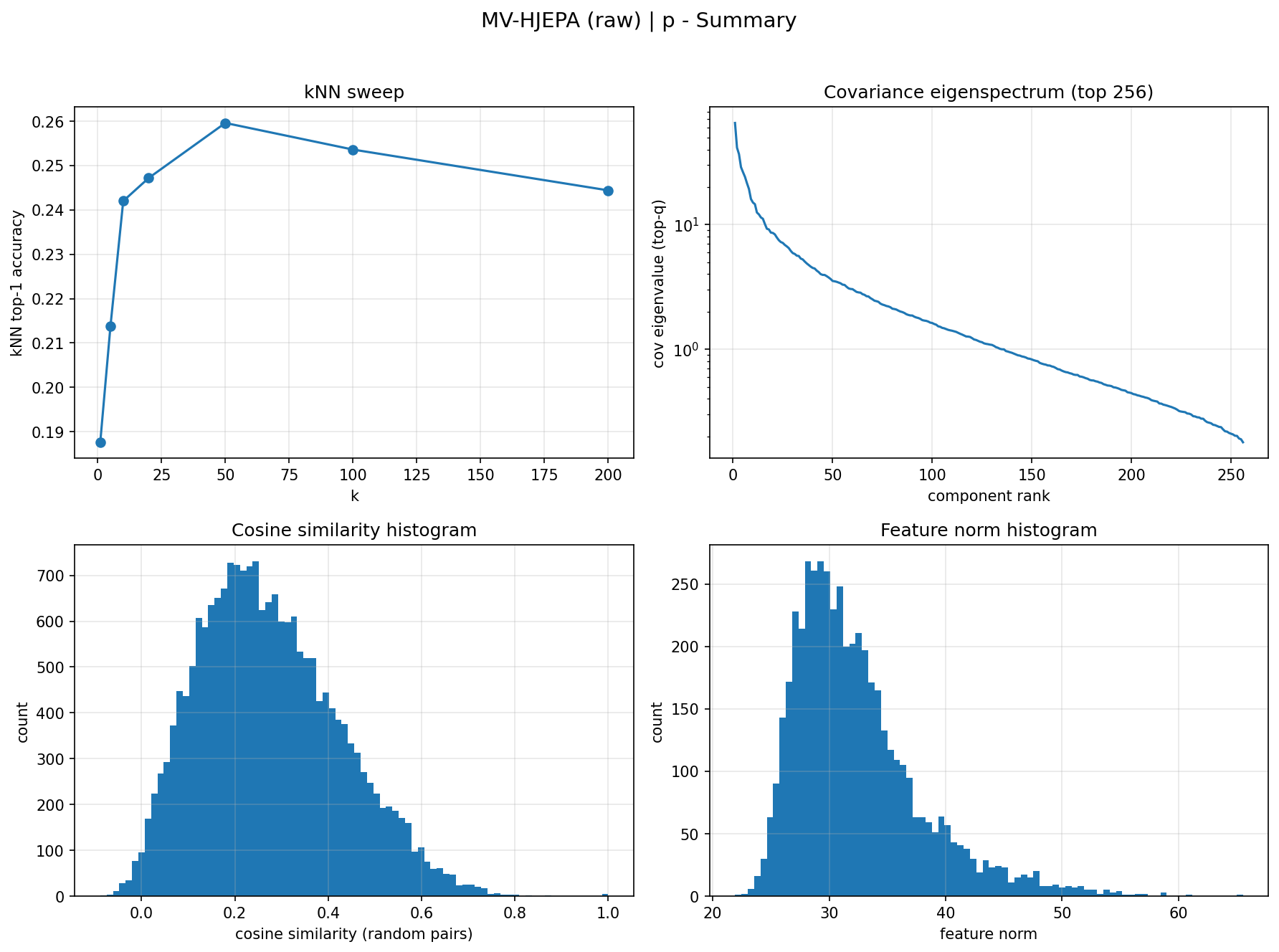}
    \caption{HamJEPA $p$ (raw).}
    \label{fig:diag_mv_hjepa_p_raw_inet}
  \end{subfigure}\hfill
  \begin{subfigure}[t]{0.49\textwidth}
    \centering
    \includegraphics[width=\linewidth]{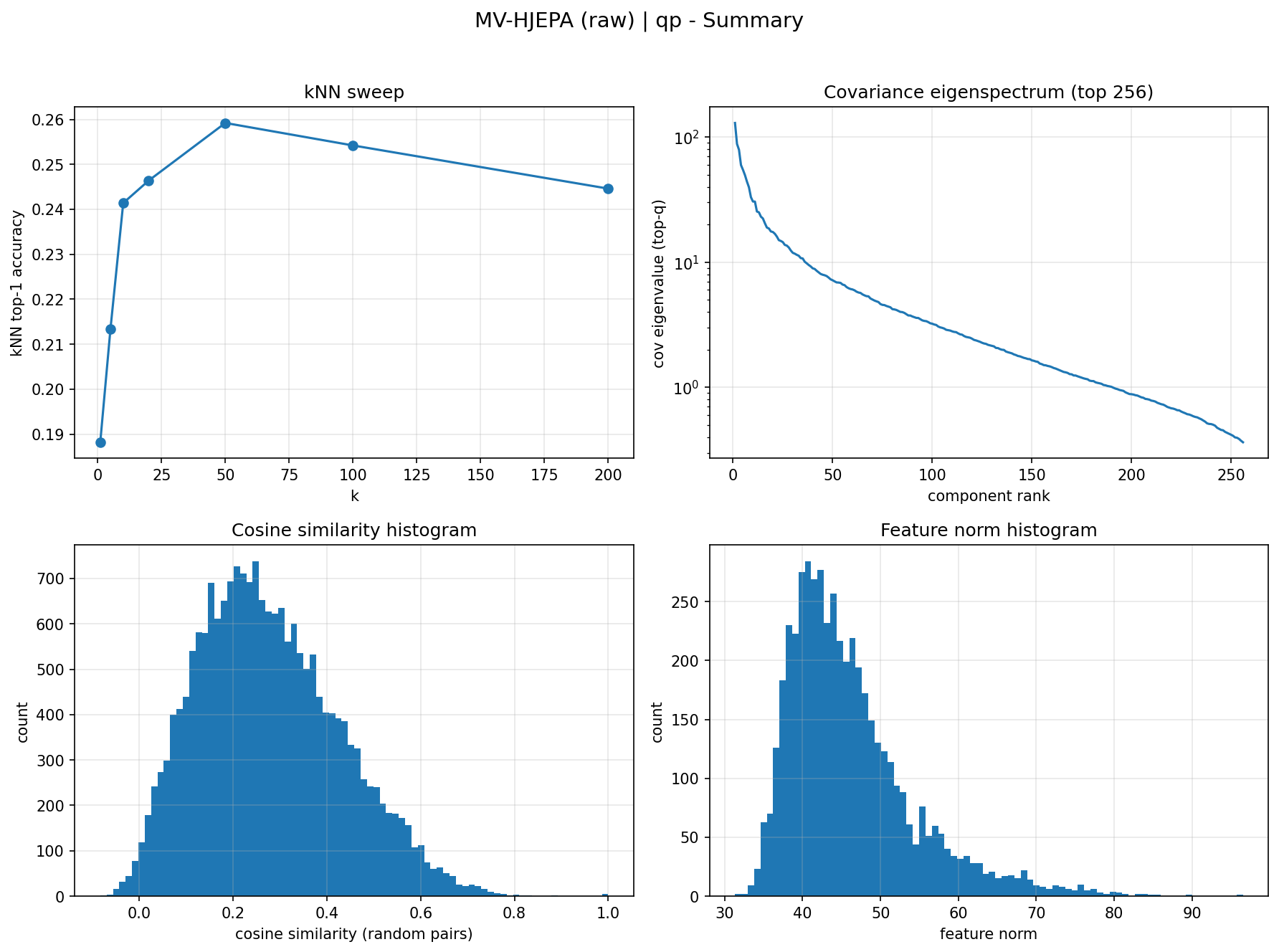}
    \caption{HamJEPA $(q,p)$ (raw).}
    \label{fig:diag_mv_hjepa_qp_raw_inet}
  \end{subfigure}

  \caption{\textbf{ImageNet-100 frozen-feature diagnostic summaries (raw).}
  Each panel contains: (top-left) $k$NN sweep; (top-right) covariance eigenspectrum (top 256, log scale);
  (bottom-left) cosine similarity histogram (random pairs); (bottom-right) feature norm histogram.}
  \label{fig:diag_grid_inet}
\end{figure*}

\FloatBarrier

\begin{figure*}[p]
  \centering
  \begin{subfigure}[t]{0.49\textwidth}
    \centering
    \includegraphics[width=\linewidth]{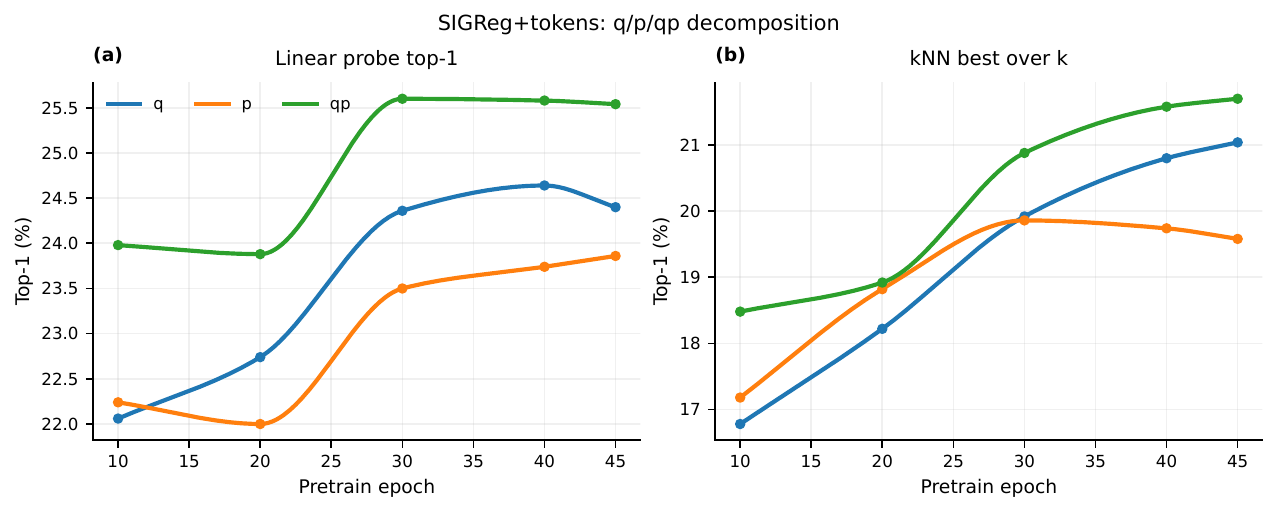}
    \caption{SIGReg (tokens).}
    \label{fig:imagenet_qp_decomp_sigreg}
  \end{subfigure}\hfill
  \begin{subfigure}[t]{0.49\textwidth}
    \centering
    \includegraphics[width=\linewidth]{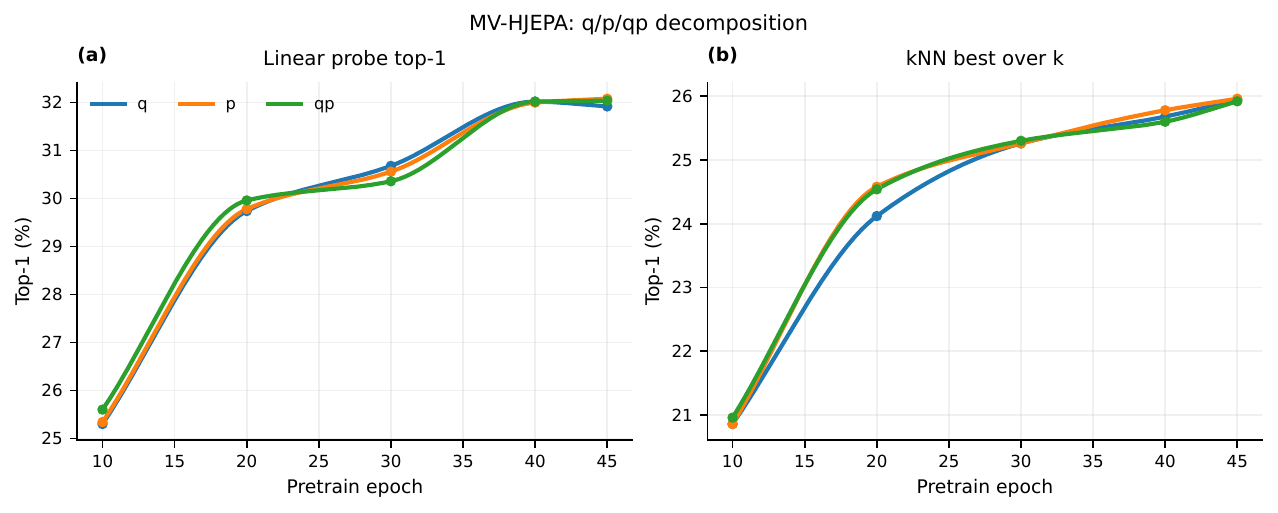}
    \caption{HamJEPA.}
    \label{fig:imagenet_qp_decomp_hjepa}
  \end{subfigure}

  \caption{\textbf{$q/p$ decomposition across pretraining (ImageNet-100).}
  We evaluate frozen features extracted at sparse pretraining checkpoints using (left) a linear probe top-1 and (right) the best $k$NN top-1 across the evaluated $k$ values.
  For SIGReg+tokens, concatenating $(q,p)$ is consistently better than either half alone, indicating that the split halves are complementary (and/or that doubling feature dimension helps). For HamJEPA, $q$, $p$, and $(q,p)$ track each other closely, suggesting that both phase-space coordinates are individually informative and largely redundant under the symplectic objective. (Left) SIGReg tokens. (Right) HamJEPA.}
  \label{fig:imagenet_qp_decomp}
\end{figure*}

\begin{figure*}[p]
  \centering
  \includegraphics[width=0.7\textwidth]{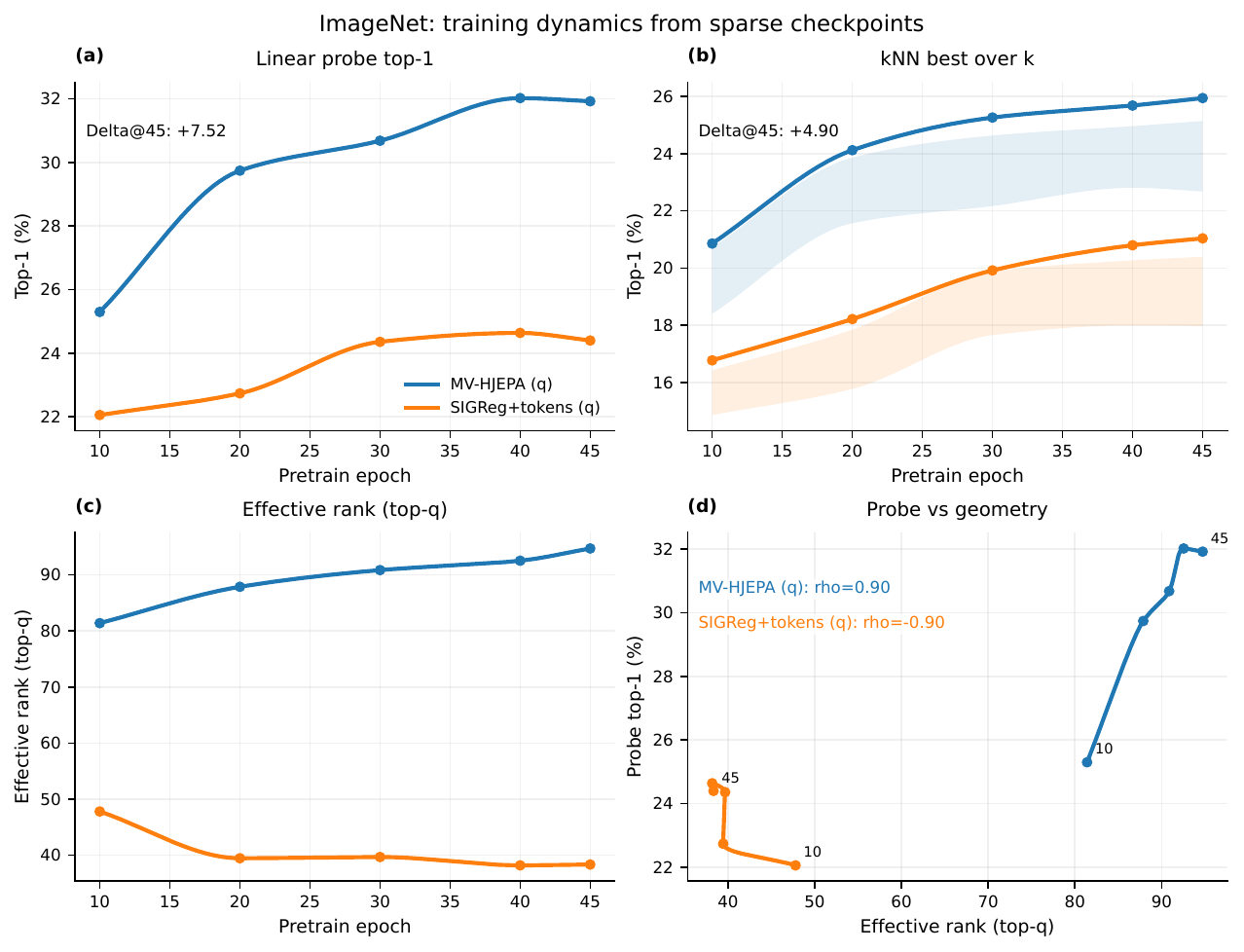}
  \caption{\textbf{ImageNet-100 training dynamics from sparse checkpoints (geometry vs.\ accuracy).}
  Using the $q$ representation at a set of pretraining epochs, we plot:
  (a) linear probe top-1; (b) best $k$NN top-1 over $k$ (shaded region indicates the spread over the $k$ sweep);
  (c) effective-rank proxy computed from the top-256 eigenvalues of the mean-centered $q$ covariance;
  (d) probe accuracy versus effective rank (points annotated by epoch), with the plotted correlation coefficient $\rho$.
  HamJEPA improves steadily and shows increasing effective rank, while SIGReg+tokens improves modestly while its
  effective rank decreases; the deltas at epoch 45 are annotated in panels (a,b).}
  \label{fig:imagenet_epoch_dynamics}
\end{figure*}
\FloatBarrier

\subsection{Replicated controls, metric robustness, and sensitivity}
\label{app:replicated_controls}

\paragraph{Protocol and independent repetitions.}
The replicated CIFAR-100 comparison uses two independent 80-epoch pretraining seeds (42 and 43) under the controlled headless token protocol. Each linear score averages five probe optimizations on each frozen representation; uncertainty across pretraining seeds is distinct from probe optimization variability. The $q$ readout has 1024 dimensions for every method. For SIGReg this is the first representation block, not a position coordinate. SIGReg-full uses all 2048 dimensions. The original three-seed 30-epoch experiment and original single-run 80-epoch experiment remain separate cohorts. All results below are descriptive; two pretraining seeds do not support a claim of formal significance.

\paragraph{Evaluation definitions for the new CIFAR-100 experiments.}
The replicated comparisons and sensitivity study use cached pre-projector features from the frozen encoder in evaluation mode. Evaluation uses all 50,000 training images and 10,000 test images, with tensor conversion and channel normalization using mean $(0.5071,0.4867,0.4408)$ and standard deviation $(0.2675,0.2565,0.2761)$, without evaluation augmentation. The stored feature layout defines $q$ as the first 1024 coordinates, $p$ as the last 1024, and $qp$ as all 2048.

All four neighborhood evaluations use uniform class voting at $k\in\{1,5,10,20,50,100,200\}$, without temperature weighting. Class-vote ties select the lowest class index. Equal-distance neighbor ordering follows \texttt{torch.topk} and is not guaranteed to be identical across hardware. For a feature vector $x$, define $\nu(x)=x/\max(\|x\|_2,10^{-12})$. The four modes are:
\begin{itemize}
  \item \textbf{Raw cosine:} compare $\nu(x)$ by dot product.
  \item \textbf{Centered cosine:} subtract the training-feature mean $\mu$ from both training and test features, then compare $\nu(x-\mu)$ by dot product.
  \item \textbf{Euclidean:} rank neighbors by squared Euclidean distance on the original features, without centering, whitening, LayerNorm, or row normalization.
  \item \textbf{Whitened cosine:} subtract the same training mean, apply the whitening transform defined below, then normalize each transformed feature before computing dot products.
\end{itemize}

For whitening, let $N_{\mathrm{tr}}=50{,}000$ and let $d$ denote the selected readout dimension. Fit only on training features:
\[
\begin{aligned}
\mu &= \frac{1}{N_{\mathrm{tr}}}\sum_{i=1}^{N_{\mathrm{tr}}}x_i,
& \Sigma &= \frac{1}{N_{\mathrm{tr}}-1}\sum_{i=1}^{N_{\mathrm{tr}}}(x_i-\mu)(x_i-\mu)^\top,\\
\rho &=10^{-3}\frac{\operatorname{tr}(\Sigma)}{d},
& LL^\top&=\Sigma+\rho I.
\end{aligned}
\]
The whitened feature is $\nu(L^{-1}(x-\mu))$; for row-stacked features, the corresponding transform is $X_cL^{-\top}$ before row normalization. The implemented dense path forms and symmetrizes the covariance in fp32, performs Cholesky factorization in fp64, and casts the factor to fp32 for the feature transform. Every recovered production result uses the stated ridge without adjustment. The ridge is fixed by this formula; test features and labels do not fit or tune preprocessing.

\paragraph{Probe optimization and uncertainty.}
For the replicated core and control runs, the probe is a learned LayerNorm with affine scale and bias and $\varepsilon=10^{-5}$, followed by a linear classifier with 100 outputs. It is trained on the original cached features, without the $k$NN centering or whitening transforms, using cross-entropy and AdamW for 50 epochs, batch size 256, learning rate $10^{-3}$, weight decay $10^{-6}$, shuffled training batches, and cosine learning-rate decay at each optimization step. The reported accuracy is from the final probe epoch. Probe seeds are $0,1,2,3,4$ for each checkpoint.

Each checkpoint's five probe accuracies are averaged first. Reported uncertainty is then the sample standard deviation across the two independent pretraining seeds, using denominator $n-1$ with $n=2$. The ten probe optimizations are not ten independent pretraining replications. For these tables, means, sample standard deviations, and improvements are computed from saved values before rounding; each reported statistic is rounded once to two decimal places using round-half-up.

\paragraph{Full-spectrum covariance diagnostics.}
The new effective-rank and participation-ratio results use all training features of the selected readout, centered by their training mean, without LayerNorm, row normalization, whitening, or ridge. If $s_j$ are all singular values of the centered feature matrix, define $\lambda_j=s_j^2/(N_{\mathrm{tr}}-1)$ and retain positive eigenvalues. With $w_j=\lambda_j/\sum_\ell\lambda_\ell$,
\[
r_{\mathrm{eff}}=\exp\!\left(-\sum_j w_j\log w_j\right),
\qquad
\mathrm{PR}=\frac{\left(\sum_j\lambda_j\right)^2}{\sum_j\lambda_j^2}.
\]
An empty positive spectrum gives zero for both statistics. These full-spectrum results are distinct from the historical top-256 PCA estimates computed on at most 5,000 samples.

\subsubsection{Replicated performance and the SIGReg readout correction}
\label{app:replicated_metrics}

\begin{table}[H]
\centering
\small
\begin{tabular}{lrrrrr}
\toprule
Method / readout & Raw & Centered & Euclidean & Whitened & Linear \\
\midrule
SIGReg-$q$ & 25.61 & 25.62 & 27.45 & 27.68 & 32.22 \\
SIGReg-full & 27.50 & 27.53 & 29.23 & 29.63 & 34.21 \\
MLP-HJEPA-$q$ & 27.50 & 27.48 & 27.66 & 36.12 & 41.47 \\
VPR-NS-$q$ & 34.59 & 34.33 & 33.97 & 39.01 & 44.18 \\
Non-reversible VPR-NS-$q$ & 34.93 & 34.89 & 33.90 & 39.18 & 44.22 \\
HamJEPA-$q$ & 34.74 & 34.59 & 34.37 & 39.26 & 43.88 \\
\bottomrule
\end{tabular}
\caption{\textbf{Replicated CIFAR-100 accuracy after 80 epochs.} The first four columns are $k$NN@20 with raw cosine, centered cosine, Euclidean, and whitened evaluation. Values are means in percent over two pretraining seeds; linear scores average five probes per seed.}
\label{tab:cifar80_all_metrics}
\end{table}

\begin{table}[H]
\centering
\small
\begin{tabular}{lrrrrr}
\toprule
Method / readout & Raw & Centered & Euclidean & Whitened & Linear \\
\midrule
SIGReg-$q$ & 0.11 & 0.35 & 0.03 & 0.13 & 0.23 \\
SIGReg-full & 0.66 & 0.81 & 0.44 & 0.31 & 0.55 \\
MLP-HJEPA-$q$ & 0.57 & 0.48 & 0.73 & 0.50 & 0.36 \\
VPR-NS-$q$ & 0.11 & 0.37 & 0.05 & 0.29 & 0.03 \\
Non-reversible VPR-NS-$q$ & 0.11 & 0.15 & 0.08 & 0.49 & 0.17 \\
HamJEPA-$q$ & 0.28 & 0.65 & 0.30 & 0.35 & 0.15 \\
\bottomrule
\end{tabular}
\caption{\textbf{Sample standard deviations across the two pretraining seeds.} Values are in percentage points, with corresponding means in Table~\ref{tab:cifar80_all_metrics}. Each linear observation is the mean of five probes for one frozen representation.}
\label{tab:cifar80_reported_uncertainty}
\end{table}

The original SIGReg values 27.98 raw $k$NN@20 and 33.95 linear used the full 2048-dimensional representation. The 25.61 and 32.22 values use its 1024-dimensional first block. Fresh full-state runs give 27.97/34.60 (raw $k$NN@20/linear) for training seed 42 and 27.03/33.82 for seed 43, with means $27.50\pm0.66$ and $34.21\pm0.55$. Thus the first-half and full-state comparisons must not be interchanged. Relative to the fresh full-state SIGReg means, HamJEPA-$q$ gains 7.24 raw $k$NN@20 and 9.67 linear points. Its centered, Euclidean, and whitened gains over the full representation are respectively 7.07, 5.14, and 9.63 points.

The repeated original value 33.95 has independent origins: the three HamJEPA 30-epoch probe values 34.13, 33.68, and 34.03 average to 33.9467, while SIGReg's 80-epoch result came from a separate stored evaluation. Five fresh probes of the original SIGReg representation give $33.90\pm0.07$; this is probe optimization variability, not a five-seed pretraining result.

\paragraph{Metric and mean-direction checks.}
HamJEPA exceeds dimension-matched SIGReg and the residual MLP in all 56 $q$-readout comparisons across four neighborhood geometries, seven values of $k$, and two pretraining seeds. These correlated sweep points are robustness checks, not independent significance tests. Centering retains approximately 98.3\% of the raw HamJEPA--SIGReg gap. The HamJEPA random-pair cosine mean falls from approximately 0.62 to approximately 0.002 after centering, while centered $k$NN@20 remains 34.59\%. The seedwise centered gaps over SIGReg-$q$ are 8.27 and 9.68 points; the corresponding Euclidean gaps are 6.73 and 7.11. The improvement therefore survives removal of the common mean direction and the evaluated changes of metric, and the 1024-dimensional $q$ representation also remains ahead of full-dimensional SIGReg.

\subsubsection{Matched non-symplectic and non-reversible controls}
\label{app:matched_predictor_controls}

\paragraph{Reversible force-rotation control.}
VPR-NS uses the identical separable potential and 328,449 predictor parameters as HamJEPA, with the same $K=2$, $h=0.05$, residual schedule, matching loss, anti-collapse terms, optimizer, batch size, precision, and training schedule. It replaces each potential-force vector $\nabla V_\phi(q)$ by $R\nabla V_\phi(q)$, where $R$ applies a fixed 30-degree rotation to adjacent coordinate pairs, chosen before downstream evaluation. This preserves force norms and gives force alignment $\cos(\pi/6)\approx0.866$. The kick--drift--kick map remains invertible, reversible, and determinant-one, but is generically non-symplectic; its measured relative symplectic defect is approximately 0.01, versus numerical zero for HamJEPA.

HamJEPA--VPR-NS centered-$k$NN differences are $-0.46$ and $+0.98$ points across the two seeds. The linear means are 43.88 and 44.18. These results do not establish a HamJEPA advantage over this control or identify symplecticity as necessary for the benchmark gain.

\paragraph{Non-reversible splitting control.}
Starting from VPR-NS, the later control changes the palindromic second-order splitting to a predeclared non-palindromic second-order splitting. Architecture, force, horizon, number of force evaluations, and training settings are unchanged. Define the signed macro-step $\tau=\mathrm{direction}\times Kh$, with $K=2$ and $h=0.05$, and the stages
\[
\mathsf{K}_b^\tau(q,p)=\bigl(q,p-b\tau R\nabla V_\phi(q)\bigr),
\qquad
\mathsf{D}_a^\tau(q,p)=\bigl(q+a\tau p,p\bigr).
\]
The implemented application order, first applied first, is
\[
\mathsf{K}_{0.20}^\tau,\quad
\mathsf{D}_{0.40}^\tau,\quad
\mathsf{K}_{0.25}^\tau,\quad
\mathsf{K}_{0.25}^\tau,\quad
\mathsf{D}_{0.60}^\tau,\quad
\mathsf{K}_{0.30}^\tau.
\]
The two middle kicks are computed separately, giving four potential-gradient evaluations and two drifts, as in the reversible $K=2$ control. Combining the middle kicks mathematically gives $(a_1,a_2)=(0.4,0.6)$ and $(b_1,b_2,b_3)=(0.2,0.5,0.3)$, which satisfy $a_1+a_2=b_1+b_2+b_3=1$ and $b_1+a_2b_2=1/2$, the second-order consistency condition. The fixed rotation acts on each adjacent pair as
\[
\begin{pmatrix}u_{2j-1}\\u_{2j}\end{pmatrix}
\longmapsto
\begin{pmatrix}
\cos(\pi/6)&-\sin(\pi/6)\\
\sin(\pi/6)&\cos(\pi/6)
\end{pmatrix}
\begin{pmatrix}u_{2j-1}\\u_{2j}\end{pmatrix},
\qquad j=1,\ldots,512.
\]
Negative direction keeps this application order and negates $\tau$. The exact inverse instead reverses the stage order and negates $\tau$.

The map remains invertible, volume-preserving, generically nonsymplectic, and second-order accurate. Its displacement is matched in the structural fixture. In the eight-dimensional float64 structural test, its determinant error is $2.22\times10^{-16}$. The unnormalized negative-step cycle norm over a $13\times8$ fixture batch increases from numerical precision for reversible VPR-NS to $1.17\times10^{-7}$. The trained-checkpoint reversibility defects, normalized by forward displacement as defined in Appendix~\ref{app:control_dynamics}, are 0.0439 and 0.0165. The corresponding state-normalized cycle errors are approximately 0.00356 and 0.00134; the exact reverse-stage inverse has mean state-normalized errors $2.19\times10^{-8}$ and $2.05\times10^{-8}$. Its $34.89\pm0.15$ centered $q$-$k$NN@20 and $44.22\pm0.17$ linear performance retain the gain. Thus reversibility is not identified as individually necessary either. The experiments do not separately identify volume preservation as necessary.

\paragraph{Scope of the component evidence.}
The tested constrained phase-space updates outperform SIGReg and the generic residual MLP in this protocol. The auxiliary-matching intervention below supports a role for coupling $p$ during training, while the MLP shows that this supervision alone is insufficient for the best native neighborhoods. This does not isolate one uniquely causal invariant. The oracle characterization selects the Hamiltonian lift under its stated joint assumptions; it does not assert that every violation of an assumption must reduce every downstream score.

\subsubsection{Rollout, potential, and matching sensitivity}
\label{app:control_sensitivity}

Nine predeclared conditions use pretraining seed 42 and the first 30 epochs of the same 80-epoch schedule. They are not continuations or replications of the separate original 30-epoch schedule. The default uses $K=2$, $h=0.05$, residual scale 0.02, and $q+0.2p$ matching. Each condition uses one 50-epoch probe with optimization seed 0. No uncertainty across pretraining seeds is estimated for this single-seed sensitivity study.

\begin{table}[H]
\centering
\small
\begin{tabular}{lrrr}
\toprule
Setting & Centered $k$NN@20 & Whitened $k$NN@20 & Linear \\
\midrule
$K=1,\ h=0.10,\ Kh=0.10$ & 32.56 & 34.71 & 38.91 \\
$K=2,\ h=0.05,\ Kh=0.10$ (default) & 32.33 & 33.78 & 38.09 \\
$K=4,\ h=0.025,\ Kh=0.10$ & 32.67 & 34.26 & 37.72 \\
$K=2,\ h=0.025$ & 32.31 & 34.10 & 37.10 \\
$K=2,\ h=0.10$ & 31.51 & 33.01 & 36.46 \\
Residual scale 0 & 32.01 & 33.79 & 38.05 \\
Residual scale 0.05 & 33.47 & 34.18 & 37.76 \\
Pure $q$ matching ($p$ weight 0) & 30.01 & 30.68 & 34.55 \\
Full-state matching & 32.78 & 34.71 & 38.28 \\
\bottomrule
\end{tabular}
\caption{\textbf{Single-seed sensitivity at epoch 30 of a common 80-epoch schedule.} Accuracy in percent. The first three rows hold the rollout horizon fixed; the next two also change the horizon.}
\label{tab:cifar_sensitivity_completed}
\end{table}

At fixed horizon, $K=1,2,4$ span 0.34 centered-$k$NN points. Halving the default step at fixed $K$ changes centered $k$NN by $-0.02$; doubling it also doubles the horizon and weakens all three metrics. Setting the learned residual to zero leaves a linear-symplectic quadratic predictor within 0.32 centered-$k$NN and 0.04 linear points of default. Increasing residual scale improves centered $k$NN while slightly reducing the probe, so its effect is not uniform across metrics.

Setting only \texttt{p\_weight=0} removes auxiliary momentum matching and loses 2.32 centered-$k$NN and 3.54 linear points. Full-state matching remains close to default. The reported covariance diagnostics move in the opposite direction under pure $q$ matching: effective rank rises from 248.71 to 306.23 and participation ratio from 77.48 to 93.83. Thus larger marginal rank is insufficient to explain the accuracy improvement. These are single-seed component trends; hidden width/depth and every anti-collapse threshold were not varied.

\subsubsection{Distance from identity and learned dynamics}
\label{app:control_dynamics}

\paragraph{Diagnostic definitions.}
For each of $N=5{,}000$ paired augmentations, let $z_i$ be the initial state, $y_i$ its paired-view target, $\widehat z_i=F_+(z_i)$ the prediction, and $c_i=F_-(\widehat z_i)$ the cycle obtained using the same stage order with negative direction. With $\varepsilon=10^{-12}$, the mean relative displacement and state-normalized cycle error are
\[
D_{\mathrm{rel}}=\frac1N\sum_{i=1}^N
\frac{\|\widehat z_i-z_i\|_2}{\max(\|z_i\|_2,\varepsilon)},
\qquad
C_{\mathrm{state}}=\frac1N\sum_{i=1}^N
\frac{\|c_i-z_i\|_2}{\max(\|z_i\|_2,\varepsilon)}.
\]
The reported trained-checkpoint non-reversibility defect instead normalizes each cycle error by its forward displacement:
\[
C_{\mathrm{increment}}=\frac1N\sum_{i=1}^N
\frac{\|c_i-z_i\|_2}{\|\widehat z_i-z_i\|_2+\varepsilon}.
\]
Exact-inverse cycle errors use $F_+^{-1}(\widehat z_i)$ in place of $c_i$ in $C_{\mathrm{state}}$. Full-state diagnostics use all 2048 coordinates; block diagnostics apply the same formulas to the relevant 1024 coordinates.

Identity and learned prediction MSE are first averaged over coordinates and then over pairs. Their fractional improvement is $(\overline{\mathrm{MSE}}_{\mathrm{id}}-\overline{\mathrm{MSE}}_{\mathrm{pred}})/\max(\overline{\mathrm{MSE}}_{\mathrm{id}},\varepsilon)$. Increment alignment is the cosine between $\widehat z_i-z_i$ and $y_i-z_i$, using $\varepsilon=10^{-12}$.

For $\mathcal H_\phi(q,p)=\tfrac12\|p\|_2^2+V_\phi(q)$, the reported mean absolute relative energy drift is
\[
\frac1N\sum_{i=1}^N
\frac{|\mathcal H_\phi(\widehat z_i)-\mathcal H_\phi(z_i)|}
     {\max(|\mathcal H_\phi(z_i)|,\varepsilon)}.
\]
Energy uses coordinate sums, and this relative statistic depends on the energy gauge. The residual/base-force ratio averages $\|s\nabla f_\phi(q_i)\|_2/\max(\|\alpha q_i\|_2,\varepsilon)$ over pairs, where $V_\phi(q)=\tfrac12\alpha\|q\|_2^2+s f_\phi(q)$. Diagnostic distributional standard deviations use denominator $N$; the accuracy uncertainties across training seeds instead use sample standard deviations.

The separate float64 structural tests use the full Jacobian $J_F$ in a small-dimensional fixture. With $\Omega=\left(\begin{smallmatrix}0&I\\-I&0\end{smallmatrix}\right)$, the relative symplectic defect is $\|J_F^\top\Omega J_F-\Omega\|_F/\|\Omega\|_F$, and the determinant error is $|\det J_F-1|$. These fixture measurements are not full-Jacobian tests of the trained 2048-dimensional predictors.

Diagnostics on 5,000 paired augmentations per checkpoint give approximately 8.3\% full-state displacement for HamJEPA and 8.0\% for reversible VPR-NS. Both use $p$ functionally and invert with cycle error approximately $2\times10^{-8}$. HamJEPA's mean absolute energy drift is 0.145\%, compared with 0.488\% for VPR-NS, and its learned residual force is roughly five times the quadratic force. Thus the trained predictor is neither the inactive identity nor merely its fixed quadratic base.

HamJEPA nevertheless does not reduce final paired-view MSE relative to identity. The residual MLP improves that MSE by approximately 29\% while remaining 7.11 centered-$k$NN and 2.41 linear points below HamJEPA. Zeroing $p$ improves HamJEPA/VPR $q$-prediction MSE, whereas shuffling or randomizing $p$ changes the prediction substantially. These diagnostics distinguish representation quality from direct paired-view regression fit. They support interpreting the predictor as a geometric training constraint and do not establish a physical trajectory that transports each augmentation pointwise to the other.

\subsubsection{Matched training cost}
\label{app:control_cost}

\begin{table}[H]
\centering
\small
\begin{tabular}{lrrr}
\toprule
Method & Seconds / step & Augmented views / s & Peak allocated (GiB) \\
\midrule
SIGReg & 0.664 & 1542 & 2.457 \\
MLP-HJEPA & 0.808 & 1268 & 2.459 \\
VPR-NS & 0.838 & 1222 & 2.459 \\
HamJEPA & 1.696 & 604 & 2.459 \\
\bottomrule
\end{tabular}
\caption{\textbf{Matched exact-configuration A100 preflight benchmark on real CIFAR-100 batches.} Batch size is 512 images with two views per image; the throughput column counts augmented views. Each entry averages two cold preflight observations, using zero warmup steps and one timed step per method and seed configuration. These preflight measurements are distinct from the approximate original full-run timings in Appendix~\ref{app:compute}.}
\label{tab:matched_training_cost}
\end{table}

For HamJEPA, the cold preflight step time is approximately $2.55\times$ that of SIGReg, with essentially unchanged peak allocated memory; the zero-warmup, single-step procedure does not support a reliable steady-state runtime ratio. It adds 328,449 predictor parameters. The predictor is used only during pretraining and is discarded for frozen-encoder inference.

The benchmark reconstructs newly initialized temporary model and optimizer instances and uses a fixed real augmented batch transferred to the GPU before timing. Data loading and augmentation preparation are excluded. The timed full step includes the encoder/projector, objective and active regularizers, backward pass, configured gradient clipping, optimizer, and scheduler. Timing uses CUDA synchronization before and after each step, with peak allocated-memory statistics reset before timing.

\paragraph{Remaining evaluation scope.}
ImageNet-100 remains a single 45-epoch pretraining seed. Production learned-marginal HamSIGReg, standard projector/EMA/multicrop recipes, VICReg/Barlow Twins/KerJEPA, larger backbones, and ImageNet-1k were not completed. The controlled results do not establish that the gain survives those settings or that the theoretical price of isotropy equals the measured accuracy difference.

\subsection{Further figures}

Figures~\ref{fig:lejepa_style_projection_cartoon}--\ref{fig:lejepa_style_directional_discrepancy}
provide an intuition for a LeJEPA-inspired \emph{sliced-projection} diagnostic: rather than comparing rollout
distributions only in aggregate, we compare many 1D projections (a Cram\'er--Wold viewpoint) to expose
direction-dependent geometric distortion.

\begin{figure*}[!b]
  \centering
  \includegraphics[width=0.9\textwidth]{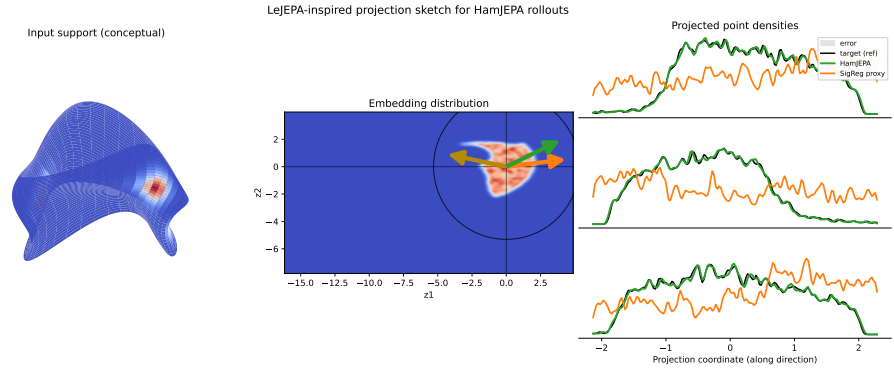}
  \caption{\textbf{LeJEPA-inspired sliced-projection diagnostic for HamJEPA rollouts (cartoon).}
  Left: conceptual support of the input distribution (manifold-structured data).
  Middle: an example embedding distribution in a 2D slice with several projection directions $u(\theta)$.
  Right: estimated 1D projected densities along representative directions; black denotes the reference/target,
  green denotes HamJEPA rollouts, orange denotes a SIGReg proxy baseline, and the shaded band visualizes the
  per-direction projection error.
  Matching many 1D projections (Cram\'er--Wold/sliced view) makes geometric distortion from non-structure-preserving
  dynamics visible, while symplectic rollouts remain close to the reference across directions.}
  \label{fig:lejepa_style_projection_cartoon}
\end{figure*}

\begin{figure*}[!b]
  \centering
  \includegraphics[width=0.9\textwidth]{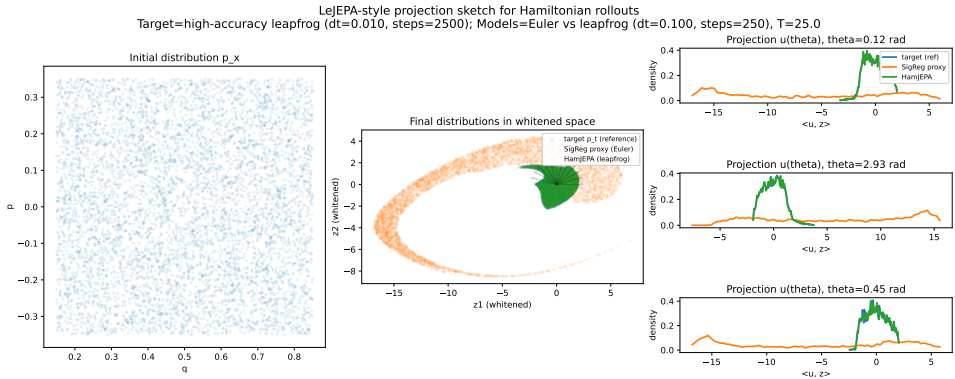}
  \caption{\textbf{LeJEPA-style projection sketch on a toy Hamiltonian transport problem.}
  Left: initial phase-space distribution $p_x$ over $(q,p)$.
  Middle: final distributions in whitened coordinates after rolling out to a fixed horizon $T$;
  the reference target is generated by a high-accuracy leapfrog integrator, while the two model curves compare
  a coarse Euler rollout (SIGReg proxy) against a coarse leapfrog rollout (HamJEPA).
  Right: 1D densities of projections $\langle u(\theta), z\rangle$ for representative angles $\theta$,
  showing that leapfrog-based rollouts track the reference distribution across directions, whereas Euler produces
  large direction-dependent mismatch.}
  \label{fig:lejepa_style_projection_sketch}
\end{figure*}

\begin{figure*}[!t]
  \centering
  \includegraphics[width=0.9\textwidth]{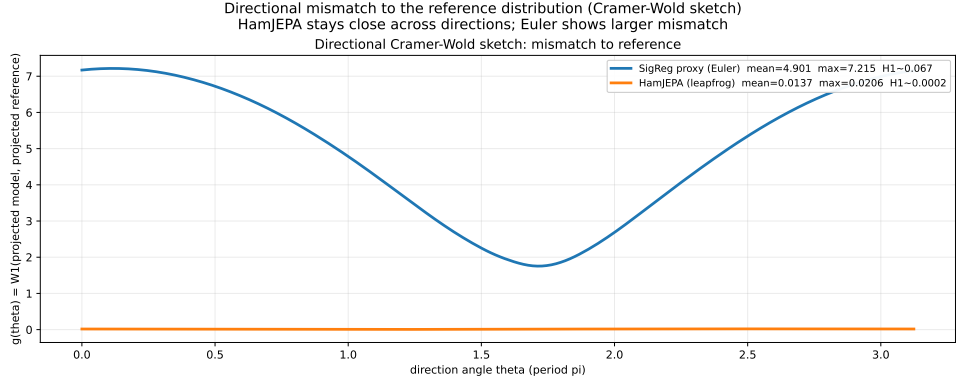}
  \caption{\textbf{Directional Cram\'er--Wold discrepancy (sliced mismatch) versus projection angle.}
  For each direction angle $\theta$ (period $\pi$), we compute a 1D mismatch
  $g(\theta)=W_1(\langle u(\theta),Z_{\text{model}}\rangle,\langle u(\theta),Z_{\text{ref}}\rangle)$ between the
  projected model and reference distributions.
  HamJEPA (leapfrog) stays uniformly close across directions, while the Euler/SIGReg proxy exhibits large,
  anisotropic deviation; the legend reports summary statistics over $\theta$ (mean/max and the aggregate score used
  in the plot).}
  \label{fig:lejepa_style_directional_discrepancy}
\end{figure*}
\FloatBarrier

\section{A visual guide to HamJEPA}
\label{app:visual_hamjepa}

This appendix gives a diagrammatic explanation of HamJEPA for readers who want a more visual account than the equation-level presentation in the main text. The figures here are explanatory schematics rather than empirical results. Their purpose is to separate the method into four conceptual layers: (i) the overall two-view prediction pipeline, (ii) how the encoder constructs the phase-space state $s=[q;p]$, (iii) how the predictor evolves that state by a symplectic Hamiltonian rollout, and (iv) why encoder-side anti-collapse regularization is still required even when the predictor itself is symplectic.

\FloatBarrier

\begin{figure}[H]
\centering
\resizebox{0.97\textwidth}{!}{%
\begin{tikzpicture}[x=1cm,y=1cm]

  \node[hambox, minimum width=1.25cm, minimum height=1.25cm] (img) at (0,0) {image\\$x$};
  \coordinate (branch) at (1.15,0);

  \node[hamsmall, minimum width=1.35cm] (taua) at (2.45,1.0) {$\tau_a$};
  \node[hamsmall, minimum width=1.35cm] (taub) at (2.45,-1.0) {$\tau_b$};

  \node[hambox, minimum width=1.55cm, minimum height=0.88cm] (va) at (4.55,1.0) {view $v_a$};
  \node[hambox, minimum width=1.55cm, minimum height=0.88cm] (vb) at (4.55,-1.0) {view $v_b$};

  \node[hambox, minimum width=1.8cm, minimum height=2.85cm, fill=hamlight] (enc) at (7.15,0) {shared\\encoder\\$E_\theta$};

  \draw[qfill, rounded corners=2pt, line width=0.9pt] (9.35,1.55) rectangle +(1.55,0.72);
  \draw[pfill, rounded corners=2pt, line width=0.9pt] (9.35,0.73) rectangle +(1.55,0.72);
  \node[hammath] at (10.125,1.91) {\hm{q_a}};
  \node[hammath] at (10.125,1.09) {\hm{p_a}};
  \node[hammath] at (10.125,2.62) {\hm{s_a}};

  \draw[qfill, rounded corners=2pt, line width=0.9pt] (9.35,-0.73) rectangle +(1.55,0.72);
  \draw[pfill, rounded corners=2pt, line width=0.9pt] (9.35,-1.55) rectangle +(1.55,0.72);
  \node[hammath] at (10.125,-0.37) {\hm{q_b}};
  \node[hammath] at (10.125,-1.19) {\hm{p_b}};
  \node[hammath] at (10.125,-1.90) {\hm{s_b}};

  \node[hambox, minimum width=2.55cm, minimum height=1.18cm, fill=hamlight] (flow) at (13.15,1.45) {symplectic flow\\\hm{\Phi_\phi^{K}}};

  \draw[qfill, rounded corners=2pt, line width=0.9pt] (15.75,1.91) rectangle +(1.55,0.72);
  \draw[pfill, rounded corners=2pt, line width=0.9pt] (15.75,1.09) rectangle +(1.55,0.72);
  \node[hammath] at (16.525,2.27) {\hm{\hat q_b}};
  \node[hammath] at (16.525,1.45) {\hm{\hat p_b}};
  \node[hammath] at (16.525,2.98) {\hm{\hat s_b}};

  \node[hambox, minimum width=2.15cm, minimum height=0.72cm] (loss) at (13.15,-1.15) {prediction loss};
  \node[hambox, regfill, minimum width=4.9cm, minimum height=0.88cm] (regs) at (12.35,-2.95) {encoder-side anti-collapse regularizers};

  \coordinate (saeast)   at (10.90,1.45);
  \coordinate (sbout)    at (10.90,-1.19);   
  \coordinate (shatleft) at (15.75,1.45);
  \coordinate (shatdown) at (16.525,1.09);

  \coordinate (lossL) at ($(loss.north)+(-0.48,0)$);
  \coordinate (lossR) at ($(loss.north)+(0.48,0)$);

  \coordinate (regL) at ($(regs.north)+(-1.10,0)$);
  \coordinate (regR) at ($(regs.north)+(1.10,0)$);

  \draw[hamarrow] (img.east) -- (branch);
  \draw[hamarrow] (branch) |- (taua.west);
  \draw[hamarrow] (branch) |- (taub.west);

  \draw[hamarrow] (taua.east) -- (va.west);
  \draw[hamarrow] (taub.east) -- (vb.west);

  \draw[hamarrow] (va.east) -- ($(enc.west)+(0,0.58)$);
  \draw[hamarrow] (vb.east) -- ($(enc.west)+(0,-0.58)$);

  \draw[hamarrow] ($(enc.east)+(0,0.58)$) -- (9.35,1.55);
  \draw[hamarrow] ($(enc.east)+(0,-0.58)$) -- (9.35,-0.73);

  \draw[hamarrow] (saeast) -- (flow.west);
  \draw[hamarrow] (flow.east) -- (shatleft);

  \draw[hamsoftarrow] (sbout) -- (loss.west);
  \draw[hamsoftarrow] (shatdown) -- ++(0,-0.30) -| (lossR);

  \draw[hamsoftarrow] (10.125,0.73)  -- ++(0,-0.42) -| (regL);
  \draw[hamsoftarrow] (10.125,-1.55) -- ++(0,-0.22) -| (regR);

\end{tikzpicture}%
}
\caption{\textbf{HamJEPA at the highest level.}
Two augmented views of the same image are encoded by a shared encoder into phase-space states $s_a=[q_a;p_a]$ and $s_b=[q_b;p_b]$. The predictor is not an unconstrained regression head: it is a structured Hamiltonian rollout $\Phi_\phi^{K}$ that evolves $s_a$ into a predicted target state $\hat s_b$. Training combines a prediction loss between $\hat s_b$ and $s_b$ with separate encoder-side anti-collapse regularizers applied across the batch of states.}
\label{fig:hamjepa_visual_overview}
\end{figure}
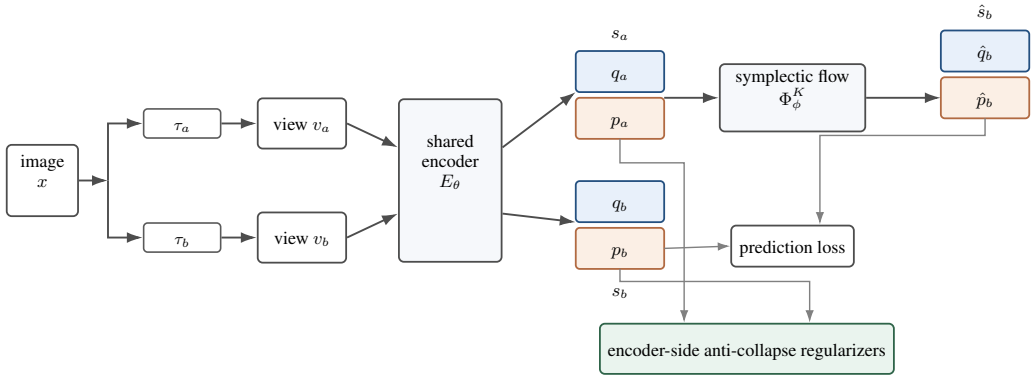


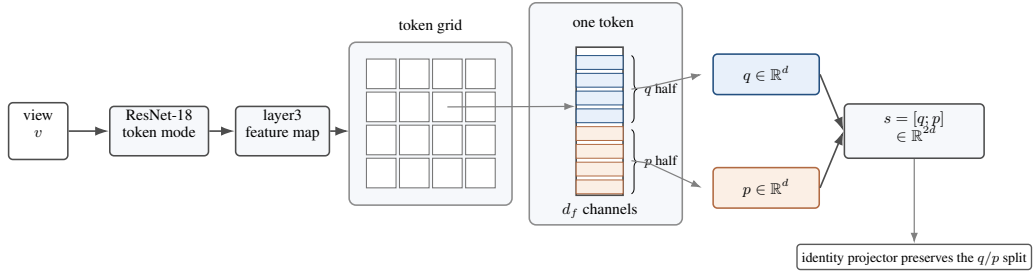
\begin{figure}[H]
\centering
\resizebox{0.97\textwidth}{!}{%
\begin{tikzpicture}[x=1cm,y=1cm]

  \node[hambox, minimum width=1.2cm, minimum height=1.15cm] (view) at (0,0) {view\\$v$};
  \node[hambox, minimum width=2.0cm, minimum height=1.0cm, fill=hamlight] (resnet) at (2.45,0) {ResNet-18\\token mode};
  \node[hambox, minimum width=1.9cm, minimum height=1.0cm, fill=hamlight] (feat) at (4.95,0) {layer3\\feature map};

  \begin{scope}[shift={(6.65,-1.15)}]
    \draw[hamgroup] (-0.35,-0.35) rectangle (2.95,2.95);
    \foreach \i in {0,1,2,3}{
      \foreach \j in {0,1,2,3}{
        \draw[line width=0.6pt, draw=black!55, fill=white]
          (0.67*\i,0.67*\j) rectangle +(0.60,0.60);
      }
    }
    \node[hammath] at (1.3,3.32) {token grid};

    \coordinate (gridtoken) at (1.64,1.64);
  \end{scope}

  \begin{scope}[shift={(10.3,-1.55)}]
    \draw[hamgroup] (-0.35,-0.35) rectangle (2.75,4.15);

    \node[hammath] at (1.15,3.78) {one token};

    \draw[line width=0.95pt, draw=black!70, fill=white] (0.60,0.25) rectangle (1.55,3.25);

    \coordinate (tokenstackin) at (0.58,2.05);

    \foreach \yy in {0.28,0.64,1.00,1.36}{
      \draw[pfill, line width=0.45pt] (0.60,\yy) rectangle (1.55,\yy+0.28);
    }

    \foreach \yy in {1.72,2.08,2.44,2.80}{
      \draw[qfill, line width=0.45pt] (0.60,\yy) rectangle (1.55,\yy+0.28);
    }

    \draw[decorate,decoration={brace,amplitude=3pt}] (1.72,3.08) -- (1.72,1.72)
      node[midway,right=4pt,font=\footnotesize] {$q$ half};
    \draw[decorate,decoration={brace,amplitude=3pt}] (1.72,1.58) -- (1.72,0.28)
      node[midway,right=4pt,font=\footnotesize] {$p$ half};

    \node[hammath] at (1.08,-0.02) {\hm{d_f\ \text{channels}}};

    \coordinate (qhalfout) at (1.72,2.40);
    \coordinate (phalfout) at (1.72,0.95);
  \end{scope}

  \draw[hamsoftarrow] (gridtoken) -- (tokenstackin);

  \node[hambox, qfill, minimum width=2.15cm, minimum height=0.82cm] (qvec) at (14.75,1.15) {\hm{q\in\mathbb{R}^{d}}};
  \node[hambox, pfill, minimum width=2.15cm, minimum height=0.82cm] (pvec) at (14.75,-1.15) {\hm{p\in\mathbb{R}^{d}}};

  \node[
    hambox,
    fill=hamlight,
    text width=2.55cm,
    minimum height=1.10cm,
    align=center
  ] (state) at (17.75,0)
  {%
    \raisebox{0.10cm}{%
      \begin{tabular}{@{}c@{}}
        \hm{s=[q;p]}\\[-1pt]
        \hm{\in\mathbb{R}^{2d}}
      \end{tabular}%
    }%
  };

  \node[hamsmall, minimum width=3.35cm] (idp) at (17.75,-2.55)
  {identity projector preserves the $q/p$ split};

  \draw[hamarrow] (view.east) -- (resnet.west);
  \draw[hamarrow] (resnet.east) -- (feat.west);
  \draw[hamarrow] (feat.east) -- (6.30,0);

  \draw[hamsoftarrow] (qhalfout) -- ++(0.35,0.05) -- (13.45,1.15);
  \draw[hamsoftarrow] (phalfout) -- ++(0.35,-0.05) -- (13.45,-1.15);

  \draw[hamarrow] (qvec.east) -- (state.west);
  \draw[hamarrow] (pvec.east) -- (state.west);

  \draw[hamsoftarrow] (state.south) -- (idp.north);

\end{tikzpicture}%
}
\caption{\textbf{How the encoder constructs the phase-space state.}
HamJEPA is not just taking an arbitrary vector and renaming its halves as $q$ and $p$. In the implementation, the encoder runs in token mode: an intermediate ResNet feature map is converted into a token grid, the \emph{channel stack of each token} is split into a $q$ half and a $p$ half, and those halves are flattened across the spatial grid to form $q\in\mathbb{R}^{d}$ and $p\in\mathbb{R}^{d}$. These are then concatenated into $s=[q;p]\in\mathbb{R}^{2d}$. This is why HamJEPA uses an identity projector: an unconstrained MLP after the encoder could mix the two coordinate blocks.}
\label{fig:hamjepa_visual_encoder}
\end{figure}


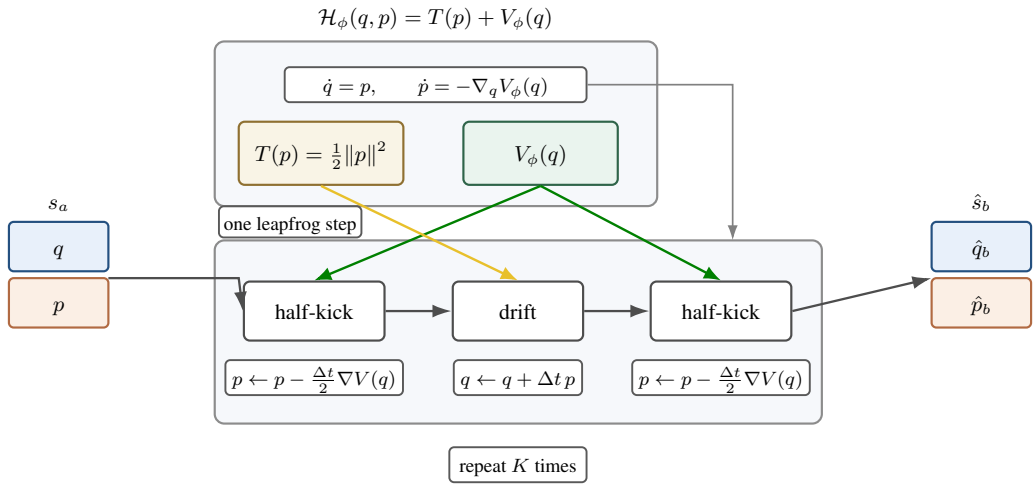
\begin{figure}[H]
\centering
\resizebox{0.97\textwidth}{!}{%
\begin{tikzpicture}[x=1cm,y=1cm]

  \draw[qfill, rounded corners=2pt, line width=0.9pt] (0.0,1.0) rectangle +(1.45,0.72);
  \draw[pfill, rounded corners=2pt, line width=0.9pt] (0.0,0.18) rectangle +(1.45,0.72);
  \node[hammath] at (0.725,1.36) {\hm{q}};
  \node[hammath] at (0.725,0.54) {\hm{p}};
  \node[hammath] at (0.725,2.03) {\hm{s_a}};

  \draw[hamgroup] (3.0,1.95) rectangle (9.45,4.35);
  \node[hammath] at (6.22,4.72) {\hm{\mathcal H_\phi(q,p)=T(p)+V_\phi(q)}};

  \node[hamsmall, minimum width=4.4cm] (dyn) at (6.22,3.72)
  {\hm{\dot q=p,\qquad \dot p=-\nabla_q V_\phi(q)}};

  \node[hambox, goldfill, minimum width=2.4cm, minimum height=0.92cm] (Tbox) at (4.55,2.72)
  {\hm{T(p)=\tfrac12\|p\|^2}};
  \node[hambox, regfill, minimum width=2.25cm, minimum height=0.92cm] (Vbox) at (7.75,2.72)
  {\hm{V_\phi(q)}};

  \draw[hamgroup] (3.0,-1.22) rectangle (11.85,1.45);
  \node[hamsmall, fill=white, inner sep=2pt] at (4.10,1.72) {one leapfrog step};

  \node[hambox, minimum width=2.05cm, minimum height=0.86cm] (kick1) at (4.45,0.42) {half-kick};
  \node[hambox, minimum width=1.90cm, minimum height=0.86cm] (drift) at (7.42,0.42) {drift};
  \node[hambox, minimum width=2.05cm, minimum height=0.86cm] (kick2) at (10.38,0.42) {half-kick};

  \node[hamsmall] at (4.45,-0.55) {\hm{p\leftarrow p-\tfrac{\Delta t}{2}\nabla V(q)}};
  \node[hamsmall] at (7.42,-0.55) {\hm{q\leftarrow q+\Delta t\,p}};
  \node[hamsmall] at (10.38,-0.55) {\hm{p\leftarrow p-\tfrac{\Delta t}{2}\nabla V(q)}};

  \node[hamsmall, minimum width=2.0cm] at (7.42,-1.82) {repeat $K$ times};

  \draw[qfill, rounded corners=2pt, line width=0.9pt] (13.45,1.0) rectangle +(1.45,0.72);
  \draw[pfill, rounded corners=2pt, line width=0.9pt] (13.45,0.18) rectangle +(1.45,0.72);
  \node[hammath] at (14.175,1.36) {\hm{\hat q_b}};
  \node[hammath] at (14.175,0.54) {\hm{\hat p_b}};
  \node[hammath] at (14.175,2.03) {\hm{\hat s_b}};

  \draw[hamarrow] (1.45,0.90) -- (3.35,0.90) -- (kick1.west);
  \draw[hamarrow] (kick1.east) -- (drift.west);
  \draw[hamarrow] (drift.east) -- (kick2.west);
  \draw[hamarrow] (kick2.east) -- (13.45,0.90);

  \draw[-{Latex[length=2.5mm,width=2mm]}, line width=1.0pt, draw=green!50!black]
    (Vbox.south) -- (kick1.north);
  \draw[-{Latex[length=2.5mm,width=2mm]}, line width=1.0pt, draw=green!50!black]
    (Vbox.south) -- (kick2.north);
  \draw[-{Latex[length=2.5mm,width=2mm]}, line width=1.0pt, draw=yellow!65!orange!90!black]
    (Tbox.south) -- (drift.north);

  \coordinate (dynelbow) at (10.55,3.72);
  \coordinate (dynhook)  at (10.55,1.45);
  \draw[hamsoftarrow] (dyn.east) -- (dynelbow) -- (dynhook);

\end{tikzpicture}%
}
\caption{\textbf{What the predictor is doing.}
HamJEPA does not use an arbitrary predictor $z_a\mapsto \hat z_b$. Instead, it rolls the state forward with a learned separable Hamiltonian $\mathcal H_\phi(q,p)=T(p)+V_\phi(q)$ and a leapfrog integrator. The potential $V_\phi(q)$ drives the kick updates through $\nabla_qV_\phi(q)$, while the kinetic term $T(p)=\tfrac12\|p\|^2$ yields the drift update $q\leftarrow q+\Delta t\,p$. Repeating this structured step $K$ times produces the predicted target state $\hat s_b=\Phi_\phi^{K}(s_a)$.}
\label{fig:hamjepa_visual_predictor}
\end{figure}

\begin{figure}[H]
\centering
\resizebox{0.97\textwidth}{!}{%
\begin{tikzpicture}[x=1cm,y=1cm]
  \node[hammath] at (2.5,3.0) {generic learned map};
  \draw[hamgroup] (0.4,0.4) rectangle (4.6,2.6);
  \fill[hamlightblue, draw=hamblue!70!black, line width=0.9pt]
    plot[smooth cycle, tension=0.9]
    coordinates {(1.0,1.2) (1.5,2.1) (2.7,2.25) (3.8,1.65) (3.45,0.9) (2.2,0.65)};
  \draw[hamarrow] (4.95,1.5) -- (6.1,1.5);
  \draw[hamgroup] (6.4,0.4) rectangle (10.6,2.6);
  \fill[hamlightred, draw=hamred!70!black, line width=0.9pt]
    plot[smooth cycle, tension=0.9]
    coordinates {(7.1,1.15) (7.35,1.55) (9.6,1.72) (9.7,1.43) (8.2,1.28)};
  \node[hamsmall, warnfill, minimum width=2.2cm] at (8.5,-0.25) {can arbitrarily squash local volume};

  \node[hammath] at (14.2,3.0) {symplectic map};
  \draw[hamgroup] (12.1,0.4) rectangle (16.3,2.6);
  \fill[hamlightblue, draw=hamblue!70!black, line width=0.9pt]
    plot[smooth cycle, tension=0.9]
    coordinates {(12.7,1.2) (13.2,2.1) (14.4,2.25) (15.5,1.65) (15.15,0.9) (13.9,0.65)};
  \draw[hamarrow] (16.65,1.5) -- (17.8,1.5);
  \draw[hamgroup] (18.1,0.4) rectangle (22.3,2.6);
  \fill[hamlightgreen, draw=hamgreen!60!black, line width=0.9pt]
    plot[smooth cycle, tension=0.9]
    coordinates {(18.5,1.05) (19.2,2.15) (21.1,2.15) (21.8,1.4) (20.55,0.65) (19.0,0.72)};
  \node[hamsmall, regfill, minimum width=2.6cm] at (20.2,-0.25) {$|\det D\Phi|=1$: bends and shears, but preserves phase-space volume};
\end{tikzpicture}%
}
\caption{\textbf{What “symplectic” means geometrically.}
A generic learned map can bend, stretch, and arbitrarily collapse local geometry.
A symplectic map can still deform the state space, but it preserves total phase-space volume locally, i.e. $|\det D\Phi|=1$.
This gives local invertibility and ensures that the predictor itself is not the source of volume collapse.}
\label{fig:hamjepa_visual_symplectic}
\end{figure}
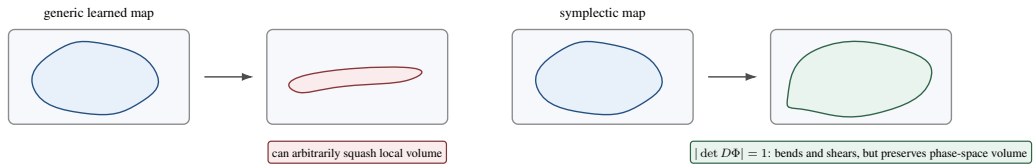

Symplecticity constrains the predictor, but it does \emph{not} by itself prevent encoder collapse. A collapsed encoder could still map many different views to almost the same latent state, and a perfectly symplectic predictor would then only move around an already-degenerate state cloud. The next two figures therefore separate these two roles: first, why predictor structure alone is insufficient; second, how the encoder-side regularizers keep the latent space alive.

\begin{figure}[H]
\centering
\resizebox{0.97\textwidth}{!}{%
\begin{tikzpicture}[x=1cm,y=1cm]
  \node[hammath] at (2.3,2.75) {many different views};
  \foreach \x/\y in {0.8/1.6,1.5/2.0,2.4/1.8,3.0/1.3,1.2/1.0,2.2/0.9,3.2/2.0} {
    \node[hambox, minimum width=0.8cm, minimum height=0.65cm] at (\x,\y) {};
  }

  \node[hambox, minimum width=1.7cm, minimum height=1.1cm, fill=hamlight] (enc2) at (6.0,1.45) {shared\\encoder};

  \draw[hamgroup] (8.3,0.7) rectangle (11.2,2.2);
  \fill[warnfill, draw=hamred!70!black] (9.65,1.45) circle (0.16);
  \fill[warnfill, draw=hamred!70!black] (9.55,1.35) circle (0.16);
  \fill[warnfill, draw=hamred!70!black] (9.72,1.55) circle (0.16);
  \fill[warnfill, draw=hamred!70!black] (9.62,1.50) circle (0.16);
  \node[hammath] at (9.75,2.65) {collapsed states};

  \node[hambox, minimum width=2.3cm, minimum height=1.0cm, fill=hamlight] (flow2) at (13.7,1.45) {symplectic\\predictor};

  \draw[hamgroup] (16.2,0.7) rectangle (19.4,2.2);
  \fill[pfill] (17.6,1.35) ellipse (0.65 and 0.28);
  \node[hammath] at (17.8,2.65) {still low-volume};

  \node[hamsmall, warnfill, minimum width=5.4cm] at (10.2,-0.2) {predictor structure cannot rescue an already-collapsed encoder output cloud};

  \draw[hamarrow] (3.7,1.45) -- (enc2.west);
  \draw[hamarrow] (enc2.east) -- (8.3,1.45);
  \draw[hamarrow] (11.2,1.45) -- (flow2.west);
  \draw[hamarrow] (flow2.east) -- (16.2,1.45);
\end{tikzpicture}%
}
\caption{\textbf{Why the predictor alone is not enough.}
Even if the predictor is perfectly symplectic, the encoder could still map many different views to nearly the same latent state.
In that situation the rollout acts on an already-collapsed representation.
This is why HamJEPA needs a second component beyond the predictor: encoder-side anti-collapse regularization.}
\label{fig:hamjepa_visual_predictor_not_enough}
\end{figure}
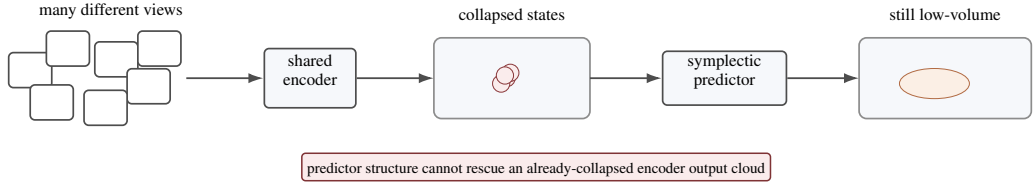

\paragraph{Why the predictor alone is not enough.}
The symplectic constraint belongs to the predictor map, not to the encoder. This means it prevents the rollout from crushing phase-space volume, but it does \emph{not} stop the encoder itself from mapping many different views to nearly the same latent state. A collapsed encoder would therefore still give a degenerate state cloud, even with a perfectly symplectic predictor. The next figures isolate the role of each encoder-side anti-collapse regularizer.

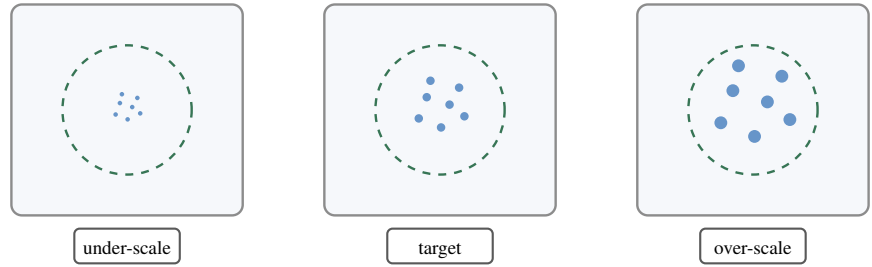
\begin{figure}[H]
\centering
\begin{tikzpicture}[scale=0.9, every node/.style={transform shape}, x=1cm,y=1cm]

  \foreach \cx/\lab/\sc in {0/under-scale/0.42,4.6/target/0.78,9.2/over-scale/1.18}{
    \draw[hamgroup] (\cx-1.7,-1.55) rectangle (\cx+1.7,1.55);
    \draw[dashed, draw=hamgreen!70!black, line width=0.9pt] (\cx,0) circle (0.95);

    \begin{scope}[shift={(\cx,0)}, scale=\sc]
      \foreach \x/\y in {0.18/0.10,0.36/0.42,-0.25/0.24,-0.40/-0.16,0.02/-0.33,0.46/-0.12,-0.18/0.55}{
        \fill[hamblue!75] (\x,\y) circle (0.08);
      }
    \end{scope}

    \node[hamsmall, minimum width=1.55cm] at (\cx,-2.0) {\lab};
  }

\end{tikzpicture}%
\caption{\textbf{Energy budget.}
The energy budget keeps the average per-dimension scale of $q$ and $p$ near a target range. It penalizes both under-scaled states (which risk collapse toward zero) and over-scaled states (which can destabilize dynamics), so the latent cloud stays at a controlled overall radius.}
\label{fig:hamjepa_visual_energy_budget}
\end{figure}

\begin{figure}[H]
\centering
\begin{tikzpicture}[scale=0.9, every node/.style={transform shape}, x=1cm,y=1cm]

  \draw[hamgroup] (-1.9,-1.6) rectangle (1.9,1.7);
  \draw[black!55,->] (-1.35,-1.1) -- (-0.15,-1.1);
  \draw[black!55,->] (-1.35,-1.1) -- (-1.35,0.7);
  \node[font=\footnotesize] at (-0.05,-1.35) {$x_1$};
  \node[font=\footnotesize] at (-1.60,0.85) {$x_2$};

  \foreach \yy in {-0.8,-0.5,-0.2,0.1,0.4,0.7}{
    \fill[hamred!75] (-0.75,\yy) circle (0.08);
  }
  \node[hamsmall, warnfill, minimum width=1.8cm] at (0,-2.05) {dead coordinate};

  \draw[hamgroup] (3.0,-1.6) rectangle (6.8,1.7);
  \draw[black!55,->] (3.55,-1.1) -- (4.75,-1.1);
  \draw[black!55,->] (3.55,-1.1) -- (3.55,0.7);
  \node[font=\footnotesize] at (4.85,-1.35) {$x_1$};
  \node[font=\footnotesize] at (3.30,0.85) {$x_2$};

  \foreach \xx/\yy in {4.0/-0.6,4.5/-0.2,5.0/0.2,4.2/0.45,5.3/-0.4,5.55/0.55}{
    \fill[hamgreen!70!black] (\xx,\yy) circle (0.08);
  }
  \node[hamsmall, regfill, minimum width=1.8cm] at (4.9,-2.05) {all coordinates active};

\end{tikzpicture}%
\caption{\textbf{Variance floor.}
When satisfied, the variance floor prevents individual coordinates from becoming nearly constant across the batch. Visually, it revives dead axes: each coordinate retains a minimum amount of variation. This does not by itself certify full-rank covariance.}
\label{fig:hamjepa_visual_variance_floor}
\end{figure}

\begin{figure}[H]
\centering
\begin{tikzpicture}[scale=0.95, every node/.style={transform shape}, x=1cm,y=1cm]

  \draw[hamgroup] (-1.9,-1.55) rectangle (1.9,1.65);
  \foreach \x/\y in {-0.8/-0.5,-0.5/0.2,-0.1/-0.2,0.2/0.5,0.55/-0.35,0.85/0.25,0.1/0.0}{
    \fill[hamgreen!70!black] (\x,\y) circle (0.08);
  }
  \node[hamsmall, regfill, minimum width=1.9cm] at (0,-2.0) {healthy volume};

  \draw[hamgroup] (3.0,-1.55) rectangle (6.8,1.65);
  \foreach \x/\y in {3.45/-0.15,3.95/-0.05,4.45/0.00,4.95/0.07,5.45/0.15,5.95/0.20,6.35/0.26}{
    \fill[hamred!75] (\x,\y) circle (0.08);
  }
  \node[hamsmall, warnfill, minimum width=1.9cm] at (4.9,-2.0) {thin sheet};

  \draw[hamgroup] (7.9,-1.55) rectangle (11.7,1.65);
  \foreach \yy in {-0.9,-0.55,-0.2,0.15,0.5,0.85}{
    \fill[hamred!75] (9.8,\yy) circle (0.08);
  }
  \node[hamsmall, warnfill, minimum width=1.9cm] at (9.8,-2.0) {line-like collapse};

\end{tikzpicture}%
\caption{\textbf{Projected log-det floor.}
When satisfied, the projected log-det floor bounds the determinant of the regularized projected covariance away from zero. It does not require isotropy, but it does force that covariance to retain nontrivial volume. Unregularized rank protection additionally requires a trace bound and a lower eigenvalue bound exceeding the covariance ridge (Proposition~\ref{prop:logdet_nocollapse}).}
\label{fig:hamjepa_visual_logdet_floor}
\end{figure}

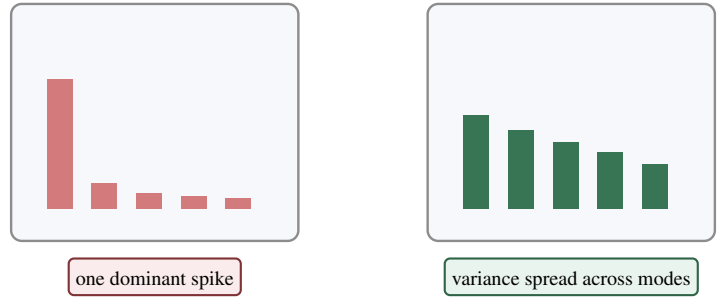
\begin{figure}[H]
\centering
\begin{tikzpicture}[scale=0.95, every node/.style={transform shape}, x=1cm,y=1cm]

  \draw[hamgroup] (-1.8,-1.45) rectangle (2.2,1.85);
  \foreach \i/\h in {0/1.8,1/0.35,2/0.22,3/0.18,4/0.14}{
    \fill[hamred!75] (-1.3+0.62*\i,-1.0) rectangle +(0.36,\h);
  }
  \node[hamsmall, warnfill, minimum width=2.15cm] at (0.2,-1.95) {one dominant spike};

  \draw[hamgroup] (4.0,-1.45) rectangle (8.0,1.85);
  \foreach \i/\h in {0/1.30,1/1.10,2/0.92,3/0.78,4/0.62}{
    \fill[hamgreen!70!black] (4.5+0.62*\i,-1.0) rectangle +(0.36,\h);
  }
  \node[hamsmall, regfill, minimum width=2.15cm] at (6.0,-1.95) {variance spread across modes};

\end{tikzpicture}%
\caption{\textbf{Participation-ratio floor and top-eigenvalue control.}
These penalties discourage the representation from surviving only through one huge dominant direction. The participation-ratio floor encourages variance to remain spread across multiple modes, while the top-eigenvalue penalty discourages a single spike from taking over the spectrum.}
\label{fig:hamjepa_visual_pr_floor}
\end{figure}

\begin{figure}[H]
\centering
\begin{tikzpicture}[scale=0.95, every node/.style={transform shape}, x=1cm,y=1cm]

  \draw[hamgroup] (-1.8,-1.55) rectangle (2.2,1.75);
  \draw[black!50,->] (-1.2,0) -- (1.6,0);
  \draw[black!50,->] (0,-1.1) -- (0,1.2);
  \foreach \x/\y in {-0.55/-0.25,-0.32/0.28,0.10/-0.12,0.35/0.25,0.62/-0.28,0.15/0.55}{
    \fill[hamgreen!70!black] (\x,\y) circle (0.08);
  }
  \node[hamsmall, regfill, minimum width=1.8cm] at (0.2,-2.0) {mean near zero};

  \draw[hamgroup] (4.0,-1.55) rectangle (8.0,1.75);
  \draw[black!50,->] (4.6,0) -- (7.4,0);
  \draw[black!50,->] (5.8,-1.1) -- (5.8,1.2);

  \foreach \x/\y in {6.20/0.25,6.45/0.78,6.85/0.38,7.05/0.75,7.32/0.20,6.72/1.02}{
    \fill[hamred!75] (\x,\y) circle (0.08);
  }
  \draw[hamarrow, draw=hamred!80!black] (5.8,0) -- (6.75,0.58);
  \node[hammath] at (6.15,0.82) {\hm{\mu}};
  \node[hamsmall, warnfill, minimum width=1.9cm] at (6.0,-2.0) {large offset direction};

\end{tikzpicture}%
\caption{\textbf{Mean penalty.}
A large global mean can create a dominant offset direction that makes random-pair cosine similarities look artificially cone-like. The mean penalty discourages this by keeping the batch mean closer to zero, so the geometry is not dominated by a single global shift.}
\label{fig:hamjepa_visual_mean_penalty}
\end{figure}
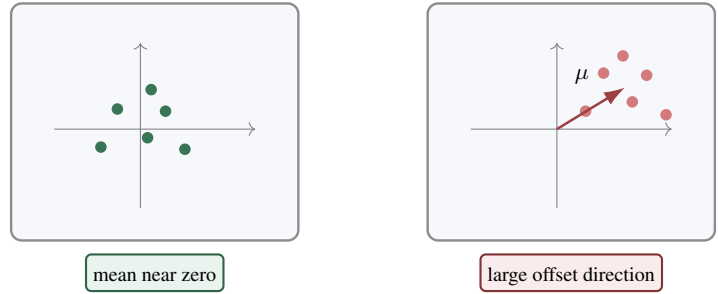

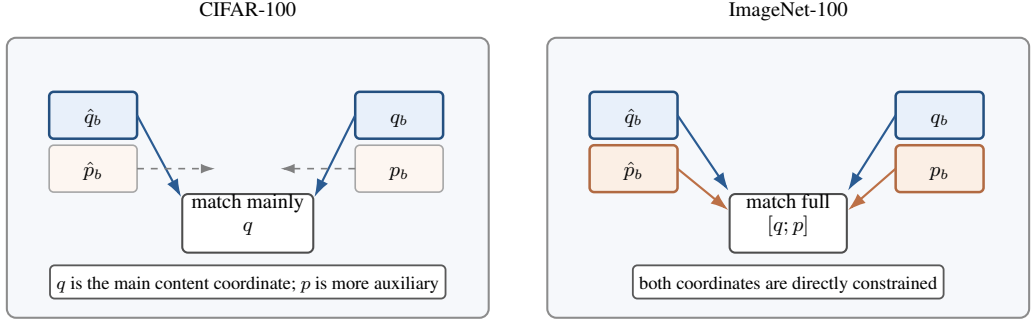
\begin{figure}[H]
\centering
\resizebox{0.97\textwidth}{!}{%
\begin{tikzpicture}[x=1cm,y=1cm]

  \draw[hamgroup] (-0.1,-1.8) rectangle (7.3,2.5);
  \node[hammath] at (3.6,2.95) {CIFAR-100};

  \draw[qfill, rounded corners=2pt, line width=1.1pt] (0.55,0.95) rectangle +(1.35,0.72);
  \draw[pfill, rounded corners=2pt, line width=0.7pt, draw=black!35, fill=hamlightorange!55] (0.55,0.13) rectangle +(1.35,0.72);
  \node[hammath] at (1.225,1.31) {\hm{\hat q_b}};
  \node[hammath] at (1.225,0.49) {\hm{\hat p_b}};

  \draw[qfill, rounded corners=2pt, line width=1.1pt] (5.25,0.95) rectangle +(1.35,0.72);
  \draw[pfill, rounded corners=2pt, line width=0.7pt, draw=black!35, fill=hamlightorange!55] (5.25,0.13) rectangle +(1.35,0.72);
  \node[hammath] at (5.925,1.31) {\hm{q_b}};
  \node[hammath] at (5.925,0.49) {\hm{p_b}};

  \node[hambox, minimum width=1.8cm, minimum height=0.9cm] (lossc) at (3.6,-0.35) {match\\\hm{q+0.2p}};
  \draw[hamarrow, draw=hamblue!80!black] (1.9,1.31) -- (lossc.west |- 3.6,0.05);
  \draw[hamarrow, draw=hamblue!80!black] (5.25,1.31) -- (lossc.east |- 3.6,0.05);

  \draw[hamsoftarrow, dashed] (1.9,0.49) -- (3.1,0.49);
  \draw[hamsoftarrow, dashed] (5.25,0.49) -- (4.1,0.49);

  \node[hamsmall, minimum width=3.7cm] at (3.6,-1.25) {$q$ is the main content coordinate; $p$ is more auxiliary};

  \draw[hamgroup] (8.2,-1.8) rectangle (15.6,2.5);
  \node[hammath] at (11.9,2.95) {ImageNet-100};

  \draw[qfill, rounded corners=2pt, line width=1.1pt] (8.85,0.95) rectangle +(1.35,0.72);
  \draw[pfill, rounded corners=2pt, line width=1.1pt] (8.85,0.13) rectangle +(1.35,0.72);
  \node[hammath] at (9.525,1.31) {\hm{\hat q_b}};
  \node[hammath] at (9.525,0.49) {\hm{\hat p_b}};

  \draw[qfill, rounded corners=2pt, line width=1.1pt] (13.55,0.95) rectangle +(1.35,0.72);
  \draw[pfill, rounded corners=2pt, line width=1.1pt] (13.55,0.13) rectangle +(1.35,0.72);
  \node[hammath] at (14.225,1.31) {\hm{q_b}};
  \node[hammath] at (14.225,0.49) {\hm{p_b}};

  \node[hambox, minimum width=1.8cm, minimum height=0.9cm] (lossi) at (11.9,-0.35) {match full\\\hm{[q;p]}};

  \draw[hamarrow, draw=hamblue!80!black] (10.2,1.31) -- (lossi.west |- 11.9,0.15);
  \draw[hamarrow, draw=hamblue!80!black] (13.55,1.31) -- (lossi.east |- 11.9,0.15);
  \draw[hamarrow, draw=hamorange!85!black] (10.2,0.49) -- (lossi.west |- 11.9,-0.15);
  \draw[hamarrow, draw=hamorange!85!black] (13.55,0.49) -- (lossi.east |- 11.9,-0.15);

  \node[hamsmall, minimum width=3.7cm] at (11.9,-1.25) {both coordinates are directly constrained};

\end{tikzpicture}%
}
\caption{\textbf{Why CIFAR-100 and ImageNet-100 use different matching choices.}
On CIFAR-100 the training objective emphasizes the content coordinate $q$ with an auxiliary $0.2p$ matching term; $q$ is the cleanest downstream readout and $p$ mainly supports the dynamics. On ImageNet-100 we match the full state $[q;p]$, which directly constrains both coordinates, and $p$ is much more informative.}
\label{fig:hamjepa_visual_cifar_imagenet}
\end{figure}

\paragraph{How to read the sequence.}
Taken together, Figs.~\ref{fig:hamjepa_visual_overview}--\ref{fig:hamjepa_visual_cifar_imagenet} decompose the full method into three conceptual layers: the encoder constructs a structured phase-space state, the predictor evolves that state with a symplectic Hamiltonian rollout, and the regularizers ensure that the encoder supplies a non-degenerate cloud of states for those dynamics to act on. The distinctive point is therefore twofold: HamJEPA changes both the \emph{representation parameterization} (from a single embedding to a phase-space state) and the \emph{predictor family} (from arbitrary regression to Hamiltonian dynamics), while avoiding the need to force the encoder marginal toward Euclidean isotropy.

\section{Implementation Details and Algorithms}
\label{app:impl}

In this section (as in the code) we refer to \emph{HamJEPA} as \emph{MV-HJEPA}, denoting the minimal multi-view implementation used in our experiments.

\paragraph{Source provenance and archive scope.}
\label{app:source_provenance}
The public code repository is linked in the introduction. The replicated CIFAR-100 experiments, sensitivity study, and common evaluator use source revision \texttt{894311e0126ac9786267ed42796640718bfd6e91}. The non-reversible control and additional cycle diagnostics use \texttt{4af1d4548153d2724f0854514c555f6f64d22b3e}. The accompanying supplementary reproducibility archive records these source snapshots, configurations, persisted execution commands, saved evaluation outputs, file hashes, and external runtime helpers verified against recorded hashes. The archive excludes datasets, checkpoints, and feature tensors; checkpoint hashes are retained. These archived revisions identify the scientific source for the new experiments.

The recovered historical snapshot \texttt{5089d6e01902c3a843de24778ae6fb69da929086} supports the documented ImageNet-100 $[+1,+1]$ view directions, full-state matching, and fixed residual scale $0.02$, but is not a certified producing revision: the inspected checkpoint contains no source-revision metadata. Its exact source-to-checkpoint provenance remains unresolved.

\paragraph{Asset provenance and terms.}
CIFAR-100 was created by Alex Krizhevsky, Vinod Nair, and Geoffrey Hinton and is described in \citet{Krizhevsky2009CIFAR}. Its official distribution page, \url{https://www.cs.toronto.edu/~kriz/cifar.html}, requests citation of that report but does not state a separate named dataset license. ImageNet-100 is a subset of ImageNet \citep{Deng2009ImageNet}; ImageNet access is subject to its non-commercial research and educational terms at \url{https://www.image-net.org/download.php}. The implementation uses PyTorch \citep{Paszke2019PyTorch} and torchvision; their BSD 3-Clause terms and copyright notices are provided at \url{https://github.com/pytorch/pytorch/blob/main/LICENSE} and \url{https://github.com/pytorch/vision/blob/main/LICENSE}. The SIGReg implementation includes code adapted from LeJEPA~\citep{lejepa}, with Creative Commons Attribution--NonCommercial 4.0 International (CC BY-NC 4.0) terms documented in the public repository's license and attribution notice at \url{https://github.com/R02AJ/HamJEPA-HamSIGReg/blob/main/LICENSE}. No dataset images are redistributed.

\subsection{Compute resources}
\label{app:compute}

All reported experiments were run on NVIDIA A100 GPUs.
CIFAR-100 pretraining took approximately $40$ seconds per epoch, corresponding to roughly
$20$ minutes for a $30$-epoch run and roughly $53$ minutes for an $80$-epoch run.
For the historical $45$-epoch ImageNet-100 runs, the recorded wall times of the training cells were approximately $15.43$ hours for HamJEPA and $5.45$ hours for SIGReg. These notebook durations include the full cell execution and do not isolate GPU compute time.
Downstream evaluation was also run on A100 GPUs and required under $5$ minutes for CIFAR-100
and under $15$ minutes for ImageNet-100 per checkpoint/readout evaluation.
These timings are approximate wall-clock measurements for the original runs; peak GPU memory was not separately profiled for those runs. The matched cold preflight benchmark in Appendix~\ref{app:control_cost} reports single-step time and peak allocated memory.

\subsection{Two training modes: MV-HJEPA vs. baseline (LeJEPA + regularizer)}
Our training scripts are \emph{mode-switched} by the presence of a \texttt{hjepa:} block in the YAML:
\begin{itemize}
  \item \textbf{MV-HJEPA mode} (\texttt{hjepa} present): we treat the representation as a phase-space state
  $z=[q;p]\in\mathbb{R}^{2d}$ and train a Hamiltonian flow predictor that maps between two global views.
  We additionally apply scale/diversity regularizers on \emph{all} views.
  \item \textbf{Baseline mode} (\texttt{hjepa} absent): we use the \emph{LeJEPA} prediction objective where all
  views predict the mean of the global views, and we optionally add a distributional regularizer
  such as SIGReg.
\end{itemize}

\subsection{Data pipeline and multi-view batching (CIFAR100 vs. ImageNet)}
\paragraph{Multi-crop dataset contract.}
Each dataset example yields a list of $V$ augmented views (multi-crop), plus labels:
\[
(\{x^{(v)}\}_{v=1}^V, y, y_{\text{coarse}}).
\]
In our configurations, we set \texttt{num\_global\_views=2} and \texttt{num\_local\_views=0}, so $V=2$.

\paragraph{CIFAR100MultiCrop.}
\texttt{eval/datasets/cifar100\_multicrop.py} applies the full augmentation \emph{in the worker} and returns
\emph{float tensors already normalized} (using \texttt{ToTensor()} + \texttt{Normalize}).
Notable defaults:
\begin{itemize}
  \item \texttt{out\_size=32} (CIFAR resolution)
  \item blur uses a small kernel (\texttt{kernel\_size=3})
\end{itemize}

\paragraph{ImageNetMultiCrop.}
\texttt{eval/datasets/imagenet\_multicrop.py} applies the same high-level augmentation pattern, but returns
\emph{uint8 tensors} using \texttt{PILToTensor()} to reduce dataloader IPC bandwidth. We then convert to float and
normalize \emph{on device} inside the ImageNet training loop.
Notable defaults:
\begin{itemize}
  \item \texttt{out\_size=224}
  \item blur uses a larger kernel (\texttt{kernel\_size=23}) appropriate for higher resolution
  \item global crop scale defaults match the YAML (\texttt{global\_scale=[0.4,1.0]})
\end{itemize}

\paragraph{Multi-view batching optimization.}
For computational efficiency, we concatenate views along the batch dimension and run a \emph{single}
encoder+projector forward:
\[
x_{\text{cat}} = \mathrm{concat}\left(x^{(1)},\dots,x^{(V)}\right)\in\mathbb{R}^{(VB)\times C\times H\times W},
\qquad
z_{\text{cat}} = P(E(x_{\text{cat}}))\in\mathbb{R}^{(VB)\times K}.
\]
We then reshape:
\[
z_{\text{views}} \in \mathbb{R}^{V\times B\times K},
\qquad
z_{\text{views}}=\mathrm{reshape}(z_{\text{cat}},(V,B,K)).
\]
This is exactly how both \texttt{train\_cifar\_hamjepa.py} and \texttt{train\_imagenet\_hamjepa.py} operate.

\subsection{Encoder and token-state parameterization}
\paragraph{ResNet18 backbone with dataset-specific stem.}
We use \texttt{torchvision.models.resnet18(weights=None)} and alter the stem depending on the dataset:
\begin{itemize}
  \item \textbf{CIFAR stem} (\texttt{encoder\_stem=cifar}): \texttt{conv1} is replaced by a $3\times3$ stride-1 conv and
  \texttt{maxpool} is removed.
  \item \textbf{ImageNet stem} (\texttt{encoder\_stem=imagenet}): keep the default $7\times7$ stride-2 conv and maxpool.
\end{itemize}

\paragraph{Token mode (used in all shown configs).}
With \texttt{encoder\_mode=tokens}, we expose an intermediate feature map at a chosen ResNet layer
(\texttt{token\_layer} in \{layer2, layer3, layer4\}), optionally pool to a fixed spatial grid (\texttt{token\_hw}), project its channels with a $1\times1$ conv and batch normalization to
\texttt{token\_d\_f}, then flatten.

Concretely, if the token feature map has shape \texttt{[B, token\_d\_f, h, w]}, the flattened representation is
\[
z \in \mathbb{R}^{B\times (h w\,\texttt{token\_d\_f})}.
\]
This is why \texttt{embed\_dim} must match $hw\,\texttt{token\_d\_f}$:
\begin{itemize}
  \item CIFAR100: \texttt{layer3} naturally yields an $8\times8$ map, and with \texttt{token\_d\_f=32} we get
  $8\cdot 8\cdot 32 = 2048$.
  \item ImageNet: \texttt{layer3} yields $14\times14$ pre-pool; we set \texttt{token\_hw=8} so adaptive pooling yields
  $8\times8$, again giving $2048$.
\end{itemize}

The resulting token projection and flattening define the trained encoder output. In the reported runs this output comes from \texttt{layer3}; \texttt{layer4} is not on the loss path and is not trained as a later version of that representation. Here ``headless'' refers to the absence of a separate learned projector after this encoder output.

\paragraph{Projector.}
We support either an MLP projector or an identity projector. In MV-HJEPA we \emph{require}
\texttt{projector\_type=identity} and enforce it in code, because we interpret the representation channels
as structured phase-space coordinates and do not want a learned MLP to arbitrarily mix them.

\paragraph{Phase-space split: $z=[q;p]$.}
In MV-HJEPA we require an even representation dimension $K=2d$, and interpret:
\[
z = \begin{bmatrix} q \\ p \end{bmatrix}\in\mathbb{R}^{2d},
\qquad
q\in\mathbb{R}^{d},\; p\in\mathbb{R}^{d}.
\]
This split is implemented \emph{channel-wise per token} in the encoder when \texttt{split\_qp=true}:
the token channel dimension \texttt{token\_d\_f} is split into equal halves (first half $\to q$, second half $\to p$)
before flattening and concatenation.

\subsection{Hamiltonian flow predictor (MV-HJEPA)}
\paragraph{Hamiltonian dynamics.}
We predict one view's state from another by integrating Hamiltonian dynamics in latent phase space.
Given a Hamiltonian $H(q,p)$, the vector field is:
\[
\frac{dq}{dt} = \frac{\partial H}{\partial p},\qquad
\frac{dp}{dt} = -\frac{\partial H}{\partial q}.
\]
In code, $\partial H/\partial q$ and $\partial H/\partial p$ are obtained via automatic differentiation
(\texttt{torch.autograd.grad(create\_graph=True)}), so gradients propagate through the integrator.

\paragraph{Hamiltonian parameterizations.}
\texttt{hamjepa/hamiltonian.py} supports:
\begin{enumerate}
  \item \textbf{Quadratic}: $H(q,p)=\tfrac12(\|q\|^2+\|p\|^2)$ (no learnable parameters).
  \item \textbf{Learnable (residual MLP)}:
  \[
  H_\theta(q,p)=\tfrac12(\|q\|^2+\|p\|^2) + s\, f_\theta([q,p]),
  \]
  where $s$ $(\texttt{residual\_scale})$ is a stabilizing residual multiplier.
  \item \textbf{Separable (used for MV-HJEPA)}:
  \[
  H_\phi(q,p)=T(p)+V_\phi(q), \qquad T(p)=\tfrac12\|p\|^2,
  \qquad
  V_\phi(q)=\tfrac12\alpha\|q\|^2 + s\, f_\phi(q),
  \]
  where $\alpha$ $(\texttt{base\_coeff})$ is a fixed quadratic base coefficient and $s$ $(\texttt{residual\_scale})$ is a residual scaling factor.
\end{enumerate}
The separable form enables an explicit velocity-Verlet/leapfrog update that is symplectic.

\paragraph{Residual-scale warmup.}
In CIFAR-100 MV-HJEPA, we linearly warm up \texttt{residual\_scale} over \texttt{residual\_scale\_warmup\_epochs}, setting $s_e=0.02\min(1,e/10)$ at epoch $e\ge1$.
Early in training the potential is close to purely quadratic (stable), and the learned residual
is introduced gradually. The ImageNet-100 training loop does not apply this residual-scale schedule: its potential uses the fixed multiplier $s=0.02$ despite the warmup field in the configuration. Both datasets use the separately configured learning-rate warmup.

\paragraph{Step size and fp32 integration.}
We integrate for \texttt{steps=S} steps with step size \texttt{dt} (fixed in MV-HJEPA; \texttt{learn\_dt=false}).
For mixed precision runs (bf16), we optionally do the dynamics computations in fp32 and cast back
(\texttt{integrate\_fp32=true}) to stabilize higher-order autodiff through the integrator.

\subsection{Symplectic integrators}
We implement:
\begin{itemize}
  \item \textbf{Symplectic Euler} for general $H(q,p)$.
  \item \textbf{Leapfrog / velocity-Verlet}:
    \begin{itemize}
      \item a general variant using autodiff for $\nabla_q H$ and $\nabla_p H$
      \item a specialized separable variant (used by MV-HJEPA when $H=T+V$) using only $\nabla V(q)$
    \end{itemize}
\end{itemize}

For separable $H(q,p)=\tfrac12\|p\|^2+V(q)$, the specialized update is:
\[
p_{1/2} = p - \tfrac{\Delta t}{2}\nabla V(q),\quad
q \leftarrow q + \Delta t\,p_{1/2},\quad
p \leftarrow p_{1/2} - \tfrac{\Delta t}{2}\nabla V(q).
\]

\subsection{MV-HJEPA prediction loss}
Given two global views' states $z_a=[q_a;p_a]$ and $z_b=[q_b;p_b]$, we predict:
\[
\hat z_b = \mathrm{Flow}_H(z_a;\Delta t,S),
\]
and compute a MSE matching term either on $q$ with an auxiliary $p$ term or the full $(q,p)$ state:
\[
\mathcal{L}_{\text{match}}=
\begin{cases}
\mathrm{MSE}(\hat q_b,q_b) + w_p\,\mathrm{MSE}(\hat p_b,p_b) & \text{(match = $q$)}\\
\mathrm{MSE}(\hat z_b,z_b) & \text{(match = $qp$)}.
\end{cases}
\]
We stop-gradient the target branch (\texttt{detach\_target=true}); encoder-side regularizers discourage trivial collapse. CIFAR-100 uses $w_p=0.2$, while ImageNet-100 uses \texttt{match=qp}.

Optionally, we add an energy-consistency term that encourages the flow to preserve the input energy:
\[
\mathcal{L}_{\text{energy}}=\mathrm{MSE}\!\big(H(q_a,p_a),\,H(\hat q_b,\hat p_b)\big),
\]
and the full prediction loss is:
\[
\mathcal{L}_{\text{pred}} = \mathcal{L}_{\text{match}} + \lambda_E \mathcal{L}_{\text{energy}}.
\]

\paragraph{Bidirectional training.}
If enabled, we symmetrize by predicting in both directions. 
\[
\mathcal{L}_{\text{pred}} \leftarrow \tfrac12\left(\mathcal{L}_{\text{pred}}(z_0,z_1) + \mathcal{L}_{\text{pred}}(z_1,z_0)\right).
\]
In the CIFAR-100 implementation,
the reverse direction uses reverse-time integration (negated $\Delta t$), which is supported by the predictor's
\texttt{direction} flag. In the ImageNet-100 implementation, the loss module is unidirectional and the loop averages two calls with swapped view inputs; both calls use positive $\Delta t$.

\subsection{Exact spectral terms in the objective}
\label{app:spectral_objective}
The following definitions complete the objective in Section~\ref{sec:hamjepa-objective}.

\paragraph{(4) Projected log-det volume floor.}
A scale budget alone can be satisfied by concentrating variance in a low-dimensional subspace.
To prevent rank collapse without forcing isotropy, we impose a projected log-det floor on the content batch.
Let $Q\in\mathbb R^{B\times d_0}$ contain centered content codes:
$Q_i = q_i - \frac1B\sum_{j=1}^B q_j$.
Sample a random projection $R\in\mathbb R^{d_0\times k}$ with orthonormalized columns ($k\ll d_0$) and form $Y=QR\in\mathbb R^{B\times k}$.
Define the regularized projected covariance
\begin{equation}
\Sigma_Y := \frac{1}{B-1}Y^\top Y + \varepsilon I_k,
\qquad \varepsilon>0,
\label{eq:hamjepa-SigmaY}
\end{equation}
and per-dimension log-volume
\begin{equation}
\ell_{\text{vol}} := \frac{1}{k}\log\det(\Sigma_Y).
\end{equation}
We use the hinge penalty
\begin{equation}
\mathcal L_{\text{vol}}
:= \big[\max(0,\;\tau-\ell_{\text{vol}})\big]^2,
\qquad \tau\in\mathbb R.
\label{eq:hamjepa-Lvol}
\end{equation}

\paragraph{(5) Participation ratio floor (discourages ``one-spike'' solutions).}
Log-det alone can be satisfied with one very large eigenvalue and many tiny ones when scale is not tightly controlled, or when minibatch effects dominate.
To penalize the ``one-spike'' regime, we add a participation ratio (PR) floor on $\Sigma_Y$:
\begin{equation}
\mathrm{PR}(\Sigma_Y)
:= \frac{\big(\mathrm{tr}(\Sigma_Y)\big)^2}{\mathrm{tr}(\Sigma_Y^2)}
= \frac{\big(\sum_{i=1}^k \lambda_i\big)^2}{\sum_{i=1}^k \lambda_i^2}
\in [1,k],
\label{eq:hamjepa-PR}
\end{equation}
where $\lambda_1,\dots,\lambda_k$ are the eigenvalues of $\Sigma_Y$.
We penalize violations of a target effective rank $r_0\in(1,k]$:
\begin{equation}
\mathcal L_{\text{pr}}
:= \big[\max(0,\;r_0-\mathrm{PR}(\Sigma_Y))\big]^2.
\label{eq:hamjepa-Lpr}
\end{equation}
For each block $x\in\{q,p\}$, the implementation uses the normalized PR hinge, $k^{-2}\mathcal L_{\text{pr}}^x$, with $r_0^x=k\rho_x$, and adds the top-eigenvalue-fraction penalty
\[
\mathcal L_{\text{eig}}^x
=\left[\max\left(0,\frac{\lambda_{\max}(\Sigma_Y^x)}{\operatorname{tr}(\Sigma_Y^x)}-\gamma_x\right)\right]^2.
\]

\subsection{Regularizers (applied across all views)}
We compute regularizers on the concatenation of all views' states:
\[
z_{\text{all}}\in\mathbb{R}^{(VB)\times 2d}.
\]

\paragraph{Phase-space energy budget (scale control).}
We stabilize scale in fixed units by enforcing per-coordinate second moments:
\[
\mathcal{L}_{\text{budget}}
=
\left(\mathbb{E}[q^2]-q_{\text{tgt}}\right)^2
+
\left(\mathbb{E}[p^2]-p_{\text{tgt}}\right)^2,
\]
where expectations are taken over batch and coordinates in fp32 and can be Distributed Data Parallel (DDP)-synchronized.

\paragraph{Variance floor (anti-collapse).}
Let $\sigma_i$ be the batch standard deviation of dimension $i$ (fp32 stats). We apply:
\[
\mathcal{L}_{\text{var}}=\frac{1}{d}\sum_{i=1}^d\max(0,\sigma_{\min}-\sigma_i)^2,
\]
on both $q$ and $p$ for CIFAR-100 and on $q$ only for ImageNet-100. The standard deviation is computed from the population batch variance with additive $\varepsilon=10^{-4}$ inside the square root.

\paragraph{Projected log-determinant floor (spectral diversity).}
We enforce a lower bound on the log-determinant of a randomly projected covariance.
Given $x\in\mathbb{R}^{B\times d}$, sample an orthonormal projection $R\in\mathbb{R}^{d\times k}$,
compute $y=(x-\bar x)R$, then:
\[
C = \mathrm{Cov}(y) + \varepsilon I,
\qquad
\mathcal{L}_{\log\det}=\max\Big(0,\tau - \tfrac1k\log\det(C)\Big)^2.
\]
We add constraints on the normalized participation ratio $\mathrm{PR}/k$, where $\mathrm{PR}=\mathrm{tr}(C)^2/\mathrm{tr}(C^2)$,
and the top-eigenvalue fraction $\lambda_{\max}/\mathrm{tr}(C)$. Each block's spectral bundle is the sum of the log-det hinge, $\max(0,\rho-\mathrm{PR}/k)^2$, and $\max(0,\lambda_{\max}/\mathrm{tr}(C)-\gamma)^2$, with unit internal weights. The scalar \texttt{lambda\_logdet} multiplies the sum of the $q$ and $p$ bundles.

\paragraph{Mean penalty.}
We penalize non-zero feature means using the squared mean (fp32):
\[
\mathcal{L}_{\text{mean}} = \frac{1}{2d}\|\mathbb{E}[z]\|_2^2.
\]

\paragraph{Active configuration values.}
Both datasets use two leapfrog steps, $\Delta t=0.05$, quadratic base coefficient $0.1$, residual target scale $0.02$, and a residual MLP of width $256$ and depth $2$. Both use $q_{\mathrm{tgt}}=p_{\mathrm{tgt}}=1$, standard-deviation floor $0.2$, projected dimension $k=128$, covariance ridge $10^{-4}$, and projection refresh interval $16$.
\begin{center}
\begin{tabular}{lcc}
\toprule
Setting & CIFAR-100 & ImageNet-100 \\
\midrule
Matching & $q+0.2p$ & full $(q,p)$ \\
Return-view step sign & $-\Delta t$ & $+\Delta t$ \\
Variance blocks & $q,p$ & $q$ \\
$(\lambda_{\mathrm{budget}},\lambda_{\mathrm{var}},\lambda_{\log\det},\lambda_{\mathrm{mean}})$ & $(0.2,0.5,0.5,0.05)$ & $(0.2,1,1,0.05)$ \\
q/p log-det floors & $-1.5,-1.5$ & $-1.2,-1.2$ \\
q/p normalized PR floors & $0.12,0.16$ & $0.25,0.25$ \\
q/p eigmax-fraction ceilings & $0.35,0.30$ & $0.10,0.10$ \\
Energy-consistency weight & $0$ & $0$ \\
\bottomrule
\end{tabular}
\end{center}

\paragraph{Frozen-feature evaluation.}
The evaluation scripts extract features with the encoder frozen and in evaluation mode. CIFAR-100 uses tensor conversion and normalization; ImageNet uses resize, center crop, tensor conversion, and normalization. The reported ``linear probe'' consists of a learned LayerNorm followed by a linear classifier on these frozen features, trained with cross-entropy and AdamW; it is not an affine classifier directly on the raw features. The reference probe defaults are $50$ epochs, batch size $256$, learning rate $10^{-3}$, weight decay $10^{-6}$, and cosine learning-rate decay. Raw cosine $k$NN uses normalized frozen features and uniform voting.

\subsection{Training step summary}
\paragraph{MV-HJEPA.}
At each step:
\begin{enumerate}
  \item Build $V$ views per image, concatenate them, and compute $z_{\text{views}}\in\mathbb{R}^{V\times B\times 2d}$.
  \item Select the two global views $(z_0,z_1)$ and compute $\mathcal{L}_{\text{pred}}$ via Hamiltonian flow prediction.
  \item Compute $\mathcal{L}_{\text{budget}}$, $\mathcal{L}_{\text{var}}$, $\mathcal{L}_{\log\det}$, and $\mathcal{L}_{\text{mean}}$
        on all views.
  \item Optimize encoder + (identity) projector + predictor parameters with AdamW, optional gradient clipping,
        and cosine LR schedule with warmup.
\end{enumerate}

\paragraph{Baseline (LeJEPA + SIGReg/HamSIGReg).}
If \texttt{hjepa} is absent:
\begin{enumerate}
  \item Compute $z_{\text{views}}$ as above.
  \item Apply LeJEPA prediction loss: all views match the per-sample mean of global views.
  \item Add a distributional regularizer (e.g.\ SIGReg) with weight \texttt{lambda\_reg}.
  \item If the regularizer includes a learnable metric $H$, optionally update $H$ on a slower/periodic schedule.
\end{enumerate}

\subsection{Algorithms}
\label{subsec:impl-algorithms}
\FloatBarrier

\makeatletter
\let\origalgorithm\algorithm
\let\origendalgorithm\endalgorithm

\renewenvironment{algorithm}[1][\relax]{%
  \par\addvspace{0.35\baselineskip}%
  \refstepcounter{algorithm}%
  \hrule height .8pt depth 0pt\kern 2pt%
  \renewcommand{\caption}[2][\relax]{%
    {\raggedright\textbf{\ALG@name~\thealgorithm}\;##2\par}%
    \ifx\relax##1\relax
      \addcontentsline{loa}{algorithm}{\protect\numberline{\thealgorithm}##2}%
    \else
      \addcontentsline{loa}{algorithm}{\protect\numberline{\thealgorithm}##1}%
    \fi
    \kern 2pt\hrule\kern 2pt%
  }%
}{%
  \kern 2pt\hrule\par\addvspace{0.35\baselineskip}%
}
\makeatother

\paragraph{Notation.}
Let $V$ be the number of augmented views per image, $B$ the batch size, and $K$ the representation dimension.
In MV-HJEPA we enforce $K=2d$ and interpret each representation as a phase-space state
$z=[q;p]\in\mathbb{R}^{2d}$ with $q,p\in\mathbb{R}^{d}$ (channel-wise split, \texttt{split\_qp=true}).
We write $\mathrm{concat}_v x^{(v)}$ for concatenation of views along the batch dimension, and
$\mathrm{reshape}$ for a view/batch reshape with the same underlying memory order as the implementation.


\begin{algorithm}[H]
\caption{Multi-crop view generation (used by \texttt{CIFAR100MultiCrop} and \texttt{ImageNetMultiCrop})}
\label{alg:multicrop}
\begin{algorithmic}[1]
\Require Image $I$, multi-crop config: $(V_g,V_\ell)$, output size $s$, scales $(\alpha_g,\beta_g)$ and $(\alpha_\ell,\beta_\ell)$
\Ensure List of views $\{x^{(v)}\}_{v=1}^{V}$ where $V=V_g+V_\ell$
\State Initialize empty list $\texttt{views}\leftarrow[\;]$
\For{$i=1$ to $V_g$}
  \State $x \leftarrow \mathrm{Augment}(I;\; \mathrm{RandomResizedCrop}(s,\mathrm{scale}=[\alpha_g,\beta_g]),\mathrm{HFlip},\mathrm{CJ},\mathrm{Gray},\mathrm{Blur},\mathrm{Solarize})$
  \State Append $x$ to \texttt{views}
\EndFor
\For{$i=1$ to $V_\ell$}
  \State $x \leftarrow \mathrm{Augment}(I;\; \mathrm{RandomResizedCrop}(s,\mathrm{scale}=[\alpha_\ell,\beta_\ell]),\mathrm{HFlip},\mathrm{CJ},\mathrm{Gray},\mathrm{Blur},\mathrm{Solarize})$
  \State Append $x$ to \texttt{views}
\EndFor
\State \Return \texttt{views}
\end{algorithmic}
\end{algorithm}

\noindent
\textbf{Implementation note (dataset-specific):}
CIFAR views are produced as float tensors with normalization applied in the dataset pipeline; ImageNet views are emitted as compact \texttt{uint8} tensors (via \texttt{PILToTensor}) and normalized on-device in the training loop.

\begin{algorithm}[H]
\caption{ResNet token encoder forward pass (token mode, optional $q/p$ split)}
\label{alg:token-encoder}
\begin{algorithmic}[1]
\Require Image batch $x\in\mathbb{R}^{B\times 3\times H\times W}$, ResNet stem+layers, token layer $\ell\in\{\texttt{layer2},\texttt{layer3},\texttt{layer4}\}$, channel projection to $d_f$, optional pooling to $\texttt{token\_hw}=h=w$, \texttt{split\_qp} flag
\Ensure Representation $z\in\mathbb{R}^{B\times K}$ where $K=h\cdot w\cdot d_f$
\State $u \leftarrow \mathrm{stem}(x)$
\State $u \leftarrow \mathrm{layer1}(u)$
\State $u \leftarrow \mathrm{layer2}(u)$; set $\mathrm{fmap}\leftarrow u$ if $\ell=\texttt{layer2}$
\State $u \leftarrow \mathrm{layer3}(u)$; set $\mathrm{fmap}\leftarrow u$ if $\ell=\texttt{layer3}$
\State $u \leftarrow \mathrm{layer4}(u)$; set $\mathrm{fmap}\leftarrow u$ if $\ell=\texttt{layer4}$
\If{\texttt{token\_hw} is set}
  \State $\mathrm{fmap}\leftarrow \mathrm{AdaptiveAvgPool2d}(\mathrm{fmap},(h,w))$
\EndIf
\State $t \leftarrow \mathrm{Conv1\times1+BN}(\mathrm{fmap}) \in \mathbb{R}^{B\times d_f\times h\times w}$
\State $t \leftarrow \mathrm{permute}(t)\in\mathbb{R}^{B\times h\times w\times d_f}$
\If{\texttt{split\_qp} = false}
  \State $z \leftarrow \mathrm{flatten}(t)\in\mathbb{R}^{B\times (hwd_f)}$
\Else
  \State Split channels: $t\mapsto (t_q,t_p)$ with $t_q,t_p\in\mathbb{R}^{B\times h\times w\times (d_f/2)}$
  \State $q\leftarrow \mathrm{flatten}(t_q)\in\mathbb{R}^{B\times (hw(d_f/2))}$;\;\; $p\leftarrow \mathrm{flatten}(t_p)\in\mathbb{R}^{B\times (hw(d_f/2))}$
  \State $z\leftarrow [q;\,p]\in\mathbb{R}^{B\times (hwd_f)}$
\EndIf
\State \Return $z$
\end{algorithmic}
\end{algorithm}

\begin{algorithm}[H]
\caption{MV-HJEPA training step (matches \texttt{train\_*\_hamjepa.py} when \texttt{hjepa:} is present)}
\label{alg:mv-hjepa-train}
\begin{algorithmic}[1]
\Require Views $\{x^{(v)}\}_{v=1}^V$, encoder $E$, projector $P$ (identity for MV-HJEPA), predictor $\mathrm{Flow}_H$, loss/regularizer weights $\lambda_{\bullet}$
\Require MV-HJEPA requires \texttt{split\_qp=true} and $K=2d$
\State $x_{\text{cat}} \leftarrow \mathrm{concat}_v\,x^{(v)} \in \mathbb{R}^{(VB)\times C\times H\times W}$
\State $z_{\text{cat}} \leftarrow P(E(x_{\text{cat}})) \in \mathbb{R}^{(VB)\times 2d}$
\State $z_{\text{views}} \leftarrow \mathrm{reshape}(z_{\text{cat}})\in \mathbb{R}^{V\times B\times 2d}$
\State $z_0 \leftarrow z_{\text{views}}[0],\;\; z_1 \leftarrow z_{\text{views}}[1]$
\If{CIFAR-100}
  \State $\mathcal{L}_{\text{pred}} \leftarrow \mathrm{HamConsistency}(z_0,z_1)$ with the configured internal \texttt{bidirectional} flag
\Else \Comment{ImageNet-100: internal \texttt{bidirectional=false}}
  \State $\mathcal{L}_{\text{pred}} \leftarrow \mathrm{HamConsistency}(z_0,z_1)$
  \If{\texttt{bidirectional}}
    \State $\mathcal{L}_{\text{pred}} \leftarrow \tfrac12\left(\mathcal{L}_{\text{pred}} + \mathrm{HamConsistency}(z_1,z_0)\right)$
  \EndIf
\EndIf
\State $z_{\text{all}} \leftarrow z_{\text{views}}.\mathrm{reshape}(VB,2d)$
\State Split $z_{\text{all}} \mapsto (q_{\text{all}},p_{\text{all}})$ along the last dimension
\State $\mathcal{L}_{\text{budget}} \leftarrow \mathrm{EnergyBudget}(z_{\text{all}})$
\State $\mathcal{L}_{\text{var}} \leftarrow \mathrm{VarFloor}(q_{\text{all}})\;(+\;\mathrm{VarFloor}(p_{\text{all}})\ \text{on CIFAR-100})$
\State $\mathcal{L}_{\log\det} \leftarrow \mathrm{LogDetFloor}(q_{\text{all}})+\mathrm{LogDetFloor}(p_{\text{all}})$
\State $\mathcal{L}_{\text{mean}} \leftarrow \mathrm{mean}_{\rm dim}\big(\mathrm{mean}_{\rm batch}(z_{\text{all}})^{\odot 2}\big)$ \Comment{\texttt{mean(dim=0).square().mean()}}
\State $\mathcal{L} \leftarrow \mathcal{L}_{\text{pred}}
        + \lambda_{\text{budget}}\mathcal{L}_{\text{budget}}
        + \lambda_{\text{var}}\mathcal{L}_{\text{var}}
        + \lambda_{\log\det}\mathcal{L}_{\log\det}
        + \lambda_{\text{mean}}\mathcal{L}_{\text{mean}}$
\State Update parameters of $\{E,P,\mathrm{Flow}_H\}$ with AdamW; optionally clip gradients and step LR schedule
\end{algorithmic}
\end{algorithm}

\begin{algorithm}[H]
\caption{Hamiltonian Consistency loss (\texttt{HamiltonianConsistencyLoss})}
\label{alg:ham-consistency}
\begin{algorithmic}[1]
\Require States $z_a,z_b\in\mathbb{R}^{B\times 2d}$, predictor $\mathrm{Flow}_H$, options: \texttt{detach\_target}, \texttt{match}\,$\in\{\texttt{q},\texttt{qp}\}$, \texttt{p\_weight}, \texttt{energy\_weight}, \texttt{bidirectional}
\Ensure Scalar loss $\mathcal{L}$
\State $\widehat{z}_{a\to b}\leftarrow \mathrm{Flow}_H(z_a;\mathrm{direction}=+1)$
\State $z_b^{\star}\leftarrow z_b$ if not \texttt{detach\_target} else $\mathrm{stopgrad}(z_b)$
\If{\texttt{match}=\texttt{qp}}
  \State $\ell_{a\to b}\leftarrow \mathrm{MSE}(\widehat{z}_{a\to b}, z_b^{\star})$
\Else \Comment{\texttt{match}=\texttt{q}}
  \State $\ell_{a\to b}\leftarrow \mathrm{MSE}(\widehat{z}_{a\to b}[:,:d], z_b^{\star}[:,:d])$
  \If{$\texttt{p\_weight}>0$}
    \State $\ell_{a\to b}\leftarrow \ell_{a\to b} + \texttt{p\_weight}\cdot \mathrm{MSE}(\widehat{z}_{a\to b}[:,d:], z_b^{\star}[:,d:])$
  \EndIf
\EndIf
\If{$\texttt{energy\_weight}>0$}
  \State Split $z_a\mapsto(q_a,p_a)$ and $\widehat{z}_{a\to b}\mapsto(\hat{q},\hat{p})$
  \State $H_a\leftarrow \mathrm{stopgrad}\big(H(q_a,p_a)\big)$;\;\; $H_{\hat{z}}\leftarrow H(\hat{q},\hat{p})$
  \State $\ell_{a\to b}\leftarrow \ell_{a\to b} + \texttt{energy\_weight}\cdot \mathrm{MSE}(H_{\hat{z}},H_a)$
\EndIf
\State $\mathcal{L}\leftarrow \ell_{a\to b}$
\If{\texttt{bidirectional}}
  \State Compute $\ell_{b\to a}$ identically using $\mathrm{direction}=-1$ and target $z_a$ (optionally detached)
  \State $\mathcal{L}\leftarrow \tfrac12(\ell_{a\to b}+\ell_{b\to a})$
\EndIf
\State \Return $\mathcal{L}$
\end{algorithmic}
\end{algorithm}

\begin{algorithm}[H]
\caption{Hamiltonian Flow Predictor (forward pass, \texttt{HamiltonianFlowPredictor})}
\label{alg:hflow}
\begin{algorithmic}[1]
\Require State $z_0\in\mathbb{R}^{\cdots\times 2d}$, Hamiltonian module $H$, method, steps $S$, base step size $\Delta t>0$, direction $s\in\{+1,-1\}$, \texttt{integrate\_fp32}, \texttt{learn\_dt}
\Ensure State $z_T\in\mathbb{R}^{\cdots\times 2d}$
\State Split $z_0 \mapsto (q_0,p_0)$
\State $\Delta t \leftarrow s\cdot \Delta t$;\;\; if \texttt{learn\_dt}, set $\Delta t\leftarrow s\cdot \mathrm{softplus}(\texttt{raw\_dt})$
\If{\texttt{integrate\_fp32} enabled \textbf{and} $\mathrm{dtype}(z_0)\in\{\mathrm{fp16},\mathrm{bf16}\}$}
  \State Cast $(q_0,p_0)$ to fp32
\EndIf
\If{$H$ is separable \textbf{and} method is leapfrog}
  \State $(q_T,p_T)\leftarrow \mathrm{IntegrateSeparableLeapfrog}(V_\phi,q_0,p_0,\Delta t,S)$
\Else
  \State $(q_T,p_T)\leftarrow \mathrm{IntegrateHamiltonian}(H,q_0,p_0,\Delta t,S,\text{method})$
\EndIf
\State $z_T \leftarrow [q_T;\, p_T]$
\If{\texttt{integrate\_fp32} enabled \textbf{and} original dtype was bf16/fp16}
  \State Cast $z_T$ back to the original dtype
\EndIf
\State \Return $z_T$
\end{algorithmic}
\end{algorithm}

\begin{algorithm}[H]
\caption{Symplectic integration routines (as in \texttt{hamjepa/integrators.py})}
\label{alg:integrators}
\begin{algorithmic}[1]
\Require General Hamiltonian $H(q,p)$ or separable potential $V(q)$, initial $(q,p)$, step size $\Delta t$, steps $S$, method $\in\{\texttt{leapfrog},\texttt{symplectic\_euler}\}$
\Ensure Final $(q,p)$ after $S$ steps
\Function{IntegrateHamiltonian}{$H,q,p,\Delta t,S,\text{method}$}
  \For{$n=1$ to $S$}
    \If{method = symplectic\_euler}
      \State $p \leftarrow p - \Delta t\,\nabla_q H(q,p)$
      \State $q \leftarrow q + \Delta t\,\nabla_p H(q,p)$ \Comment{$\nabla_p$ uses updated $p$}
    \Else \Comment{method = leapfrog (velocity Verlet)}
      \State $p_{1/2} \leftarrow p - \tfrac{\Delta t}{2}\nabla_q H(q,p)$
      \State $q \leftarrow q + \Delta t\,\nabla_p H(q,p_{1/2})$
      \State $p \leftarrow p_{1/2} - \tfrac{\Delta t}{2}\nabla_q H(q,p_{1/2})$
    \EndIf
  \EndFor
  \State \Return $(q,p)$
\EndFunction
\Function{IntegrateSeparableLeapfrog}{$V,q,p,\Delta t,S$}
  \For{$n=1$ to $S$}
    \State $p_{1/2}\leftarrow p - \tfrac{\Delta t}{2}\nabla_q V(q)$
    \State $q \leftarrow q + \Delta t\,p_{1/2}$
    \State $p \leftarrow p_{1/2} - \tfrac{\Delta t}{2}\nabla_q V(q)$ \Comment{$\nabla_q V$ uses updated $q$}
  \EndFor
  \State \Return $(q,p)$
\EndFunction
\end{algorithmic}
\end{algorithm}

\begin{algorithm}[H]
\caption{Phase-space energy budget regularizer (\texttt{PhaseSpaceEnergyBudget})}
\label{alg:energy-budget}
\begin{algorithmic}[1]
\Require States $z\in\mathbb{R}^{\cdots\times 2d}$, targets $(q_{\text{tgt}},p_{\text{tgt}})$, \texttt{ddp\_sync}
\Ensure Scalar loss $\mathcal{L}_{\text{budget}}$
\State Flatten $z\to z_f\in\mathbb{R}^{N\times 2d}$ and cast to fp32
\State Split $z_f\mapsto(q,p)$ with $q,p\in\mathbb{R}^{N\times d}$
\State $q2\_sum\leftarrow \sum q^{\odot 2}$,\;\; $p2\_sum\leftarrow \sum p^{\odot 2}$,\;\; $\texttt{count}\leftarrow Nd$
\If{\texttt{ddp\_sync} and DDP initialized}
  \State All-reduce ($\texttt{SUM}$) $q2\_sum$, $p2\_sum$, and \texttt{count}
\EndIf
\State $q2\_\mu \leftarrow q2\_sum/\texttt{count}$;\;\; $p2\_\mu \leftarrow p2\_sum/\texttt{count}$
\State $\mathcal{L}_{\text{budget}}\leftarrow (q2\_\mu-q_{\text{tgt}})^2+(p2\_\mu-p_{\text{tgt}})^2$
\State \Return $\mathcal{L}_{\text{budget}}$
\end{algorithmic}
\end{algorithm}

\begin{algorithm}[t]
\caption{Variance floor regularizer (\texttt{VarianceFloor})}
\label{alg:var-floor}
\begin{algorithmic}[1]
\Require Features $x\in\mathbb{R}^{N\times d}$, std floor $\sigma_{\min}$, $\varepsilon$, \texttt{ddp\_sync}
\Ensure Scalar loss $\mathcal{L}_{\text{var}}$
\State Cast $x$ to fp32; if $N<2$ return $0$
\State Compute sufficient stats: $s_1\leftarrow \sum_{n} x_n$, $s_2\leftarrow \sum_{n} x_n^{\odot 2}$, and $n\leftarrow N$
\If{\texttt{ddp\_sync} and DDP initialized}
  \State All-reduce ($\texttt{SUM}$) $s_1$, $s_2$, and $n$
\EndIf
\State $\mu \leftarrow s_1/n$;\;\; $\mathrm{var}\leftarrow s_2/n-\mu^{\odot 2}$;\;\; $\sigma\leftarrow \sqrt{\max(\mathrm{var},0)+\varepsilon}$
\State $\mathcal{L}_{\text{var}} \leftarrow \frac{1}{d}\sum_{i=1}^d \mathrm{ReLU}(\sigma_{\min}-\sigma_i)^2$
\State \Return $\mathcal{L}_{\text{var}}$
\end{algorithmic}
\end{algorithm}

\begin{algorithm}[H]
\caption{Projected log-det floor (\texttt{ProjectedLogDetFloor})}
\label{alg:logdet-floor}
\begin{algorithmic}[1]
\Require Features $x\in\mathbb{R}^{B\times D}$, projection dim $k$, log-det floor $\ell_{\min}$, $\varepsilon$, refresh interval $r$, optional $\texttt{pr\_norm\_floor}$ and $\texttt{eigmax\_frac\_ceiling}$, \texttt{ddp\_sync}
\Ensure Scalar loss $\mathcal{L}_{\log\det}$
\State Cast $x$ to fp32; if $B<2$ return $0$
\State $k \leftarrow \min(k,D,B-1)$
\If{need resample (first call / device change / shape change / step $\bmod\, r = 0$)}
  \State Sample $R\sim \mathcal{N}(0,1)^{D\times k}$ and orthonormalize columns via QR: $R\leftarrow \mathrm{qr}(R)$
\EndIf
\State Center and project: $x_c \leftarrow x-\mathrm{mean}(x)$;\;\; $y\leftarrow x_c R\in\mathbb{R}^{B\times k}$
\State Compute sufficient stats: $\texttt{sum\_y}\leftarrow \sum_b y_b$,\;\; $\texttt{sum\_yy}\leftarrow y^\top y$,\;\; $n\leftarrow B$
\If{\texttt{ddp\_sync} and DDP initialized}
  \State All-reduce ($\texttt{SUM}$) \texttt{sum\_y}, \texttt{sum\_yy}, and $n$
\EndIf
\State $n\leftarrow \max(n,2)$;\;\; $\mu_y\leftarrow \texttt{sum\_y}/n$
\State $\Sigma \leftarrow \frac{\texttt{sum\_yy}-n\cdot (\mu_y\mu_y^\top)}{n-1}$;\;\; $\Sigma\leftarrow \tfrac12(\Sigma+\Sigma^\top)+\varepsilon I$
\State $\ell \leftarrow \frac{1}{k}\log\det(\Sigma)$ \Comment{implemented via \texttt{slogdet}}
\State $\mathcal{L}_{\log\det}\leftarrow \mathrm{ReLU}(\ell_{\min}-\ell)^2$
\State Compute $ \mathrm{PR}=\frac{\mathrm{tr}(\Sigma)^2}{\mathrm{tr}(\Sigma^2)}$, $\mathrm{pr\_norm}=\mathrm{PR}/k$, and $\mathrm{eigmax\_frac}=\lambda_{\max}/\mathrm{tr}(\Sigma)$
\If{\texttt{pr\_norm\_floor} is set}
  \State $\mathcal{L}_{\log\det}\leftarrow \mathcal{L}_{\log\det} + \mathrm{ReLU}(\texttt{pr\_norm\_floor}-\mathrm{pr\_norm})^2$
\EndIf
\If{\texttt{eigmax\_frac\_ceiling} is set}
  \State $\mathcal{L}_{\log\det}\leftarrow \mathcal{L}_{\log\det} + \mathrm{ReLU}(\mathrm{eigmax\_frac}-\texttt{eigmax\_frac\_ceiling})^2$
\EndIf
\State \Return $\mathcal{L}_{\log\det}$
\end{algorithmic}
\end{algorithm}

\clearpage
\begin{algorithm}[H]
\caption{SIGReg forward pass (random slicing + Epps--Pulley statistic)}
\label{alg:sigreg}
\begin{algorithmic}[1]
\Require Batch $z\in\mathbb{R}^{\cdots\times N\times D}$, slices $K$, knots $t_1,\dots,t_T$ (with weights $w$), refresh interval $r$, \texttt{ddp\_sync}
\Ensure Scalar statistic $S(z)$ (mean-reduced in our configs)
\State \textbf{(Slicing)} If step $\bmod\, r = 0$, sample $A\in\mathbb{R}^{D\times K}$ with i.i.d. $\mathcal{N}(0,1)$ columns and normalize each column to unit norm
\If{\texttt{ddp\_sync} and DDP initialized}
  \State Ensure all ranks share the same sampling seed for $A$ (broadcast seed from rank 0)
\EndIf
\State Project: $y \leftarrow zA \in \mathbb{R}^{\cdots\times N\times K}$
\State \textbf{(Epps--Pulley)} Compute empirical CF on each slice:
\State \hspace{1em} $\hat{c}(t)=\frac{1}{N}\sum_{n}\cos(t\,y_n)$,\;\; $\hat{s}(t)=\frac{1}{N}\sum_{n}\sin(t\,y_n)$ (optionally globally averaged via all-reduce if \texttt{ddp\_sync})
\State \hspace{1em} Target CF for $\mathcal{N}(0,1)$ is $\phi(t)=e^{-t^2/2}$
\State \hspace{1em} $S_k \leftarrow N\sum_{j=1}^{T} w_j\Big[(\hat{c}(t_j)-\phi(t_j))^2+\hat{s}(t_j)^2\Big]$ for each slice $k$
\State Reduce over slices (and leading dims) by mean: $S(z)\leftarrow \mathrm{mean}_k(S_k)$
\State \Return $S(z)$
\end{algorithmic}
\end{algorithm}

\begin{algorithm}[H]
\caption{LeJEPA baseline step with SIGReg regularization (non-MV branch in \texttt{train\_*\_hamjepa.py})}
\label{alg:lejepa-sigreg-train}
\begin{algorithmic}[1]
\Require Views $\{x^{(v)}\}_{v=1}^V$, encoder $E$, projector $P$, number of global views $V_g$, regularizer $\mathcal{R}$ (SIGReg), weight $\lambda_{\mathrm{reg}}$
\State $x_{\text{cat}} \leftarrow \mathrm{concat}_v\,x^{(v)}$
\State $z_{\text{cat}} \leftarrow P(E(x_{\text{cat}}))\in\mathbb{R}^{(VB)\times D}$
\State $z_{\text{views}} \leftarrow \mathrm{reshape}(z_{\text{cat}})\in\mathbb{R}^{V\times B\times D}$
\State $\bar{z}\leftarrow \frac{1}{V_g}\sum_{v=1}^{V_g} z_{\text{views}}[v] \in \mathbb{R}^{B\times D}$ \Comment{mean of global views}
\State $\mathcal{L}_{\text{pred}} \leftarrow \tfrac12\ \mathrm{mean}\big(\|z_{\text{views}}-\bar{z}\|_2^2\big)$ \Comment{matches \texttt{lejepa\_prediction\_loss}}
\State $\mathcal{L}_{\text{reg}} \leftarrow \mathcal{R}(z_{\text{views}})$ \Comment{SIGReg expects shape $(\cdots,N,D)$}
\State $\mathcal{L} \leftarrow \mathcal{L}_{\text{pred}} + \lambda_{\mathrm{reg}}\mathcal{L}_{\text{reg}}$ \Comment{spectral term omitted here (zero in our SIGReg configs)}
\State Update $\{E,P\}$ with AdamW; optionally step LR schedule and log diagnostics
\end{algorithmic}
\end{algorithm}

\makeatletter
\let\algorithm\origalgorithm
\let\endalgorithm\origendalgorithm
\makeatother

\subsection{YAML Configuration Knobs}
\label{subsec:impl-yaml-knobs}

The repository is fully config-driven: the training scripts read a YAML file and route into either
(i) MV-HJEPA mode when the \texttt{hjepa:} stanza is present (and \texttt{split\_qp=true}),
or (ii) the LeJEPA+regularizer baseline when \texttt{hjepa:} is absent (e.g.\ \texttt{regularizer.type=sigreg}).
Below we document every key appearing in the provided configs, exactly as it is consumed by the code.
The reference configs use \texttt{seed=42}; the replicated $80$-epoch CIFAR-100 suite additionally uses \texttt{seed=43}.

\subsubsection{Data (\texttt{data:})}
\begin{itemize}
  \item \texttt{root}: Dataset root path passed to \texttt{CIFAR100} / \texttt{ImageFolder} loader.
  \item \texttt{batch\_size}: DataLoader batch size (\texttt{batch\_size=batch\_size}). We use larger batches on CIFAR (512) and smaller on ImageNet (256) to balance GPU memory and throughput.
  \item \texttt{num\_workers}: DataLoader worker count. Higher values improve input pipeline throughput when CPU-bound.
  \item \texttt{drop\_last}: If true, drops the last incomplete batch. This avoids smaller last batches, which can destabilize batch statistics and covariance-based regularizers (and helps DDP batch-size consistency).
  \item \texttt{pin\_memory} (ImageNet configs): Enables pinned host memory to speed CPU$\to$GPU transfers.
  \item \texttt{persistent\_workers} (ImageNet configs): Keeps workers alive across epochs to reduce worker spin-up overhead.
  \item \texttt{prefetch\_factor} (ImageNet configs): Worker prefetch depth; increases overlap between CPU augmentation and GPU training.
  \item \texttt{num\_global\_views}, \texttt{num\_local\_views}: Multi-crop configuration. In all provided runs we set two global views and no local views (\texttt{2/0}), matching the two-view training objective in both MV-HJEPA and the LeJEPA baseline.
  \item \texttt{out\_size} (ImageNet configs): Output crop resolution for \texttt{RandomResizedCrop}. We use 224 (standard ImageNet scale).
  \item \texttt{global\_scale}, \texttt{local\_scale} (ImageNet configs): Scale ranges for global/local crops. Only global crops are used when \texttt{num\_local\_views=0}.
  \item \emph{CIFAR note:} the CIFAR YAMLs omit \texttt{out\_size/global\_scale/local\_scale}, so the dataset defaults are used (\texttt{out\_size=32}, \texttt{global\_scale=(0.3,1.0)}, \texttt{local\_scale=(0.05,0.3)}).
\end{itemize}

\subsubsection{Model (\texttt{model:})}
\begin{itemize}
  \item \texttt{encoder\_mode}: \texttt{global} uses GAP (global average pooling)+FC; \texttt{tokens} exposes an intermediate feature map and flattens it into tokens. All provided configs use \texttt{tokens}.
  \item \texttt{encoder\_stem}: \texttt{cifar} modifies ResNet-18 stem (3$\times$3 stride 1, no maxpool) for 32$\times$32 inputs; \texttt{imagenet} keeps the default 7$\times$7 stride 2 + maxpool stem.
  \item \texttt{token\_layer}: Which ResNet stage to expose as the token grid (\texttt{layer2/layer3/layer4}). We use \texttt{layer3}, giving an 8$\times$8 grid on CIFAR and a 14$\times$14 grid on ImageNet before any optional pooling.
  \item \texttt{token\_d\_f}: Number of channels after the $1{\times}1$ token projection (\texttt{Conv1x1+BN}). In token mode, the representation dimension is $K=h\cdot w\cdot \texttt{token\_d\_f}$.
  \item \texttt{token\_hw} (ImageNet configs): Optional adaptive pooling size for the token grid. We set \texttt{token\_hw=8} on ImageNet so that with \texttt{token\_d\_f=32} we obtain $K=8\cdot 8\cdot 32=2048$ (matching \texttt{embed\_dim}).
  \item \texttt{embed\_dim}: Output dimension $K$ of the encoder in token mode. For our settings this is tied to the token grid geometry (e.g.\ $8\cdot 8\cdot 32=2048$).
  \item \texttt{proj\_hidden}, \texttt{proj\_out}: Hidden/output sizes for the optional MLP projector. These are read from the YAML but are only used when \texttt{projector\_type=mlp}.
  \item \texttt{projector\_type}: \texttt{identity} returns $z$ unchanged; \texttt{mlp} applies \texttt{MLPProjector}. MV-HJEPA requires \texttt{identity} so the channel-wise $q/p$ semantics are preserved.
  \item \texttt{split\_qp}: If true, the token channels are split into $(q,p)$ halves per token and concatenated, enforcing $K=2d$ and enabling the phase-space interpretation required by MV-HJEPA.
\end{itemize}

\subsubsection{MV-HJEPA predictor (\texttt{hjepa:})}
(Only present in \texttt{*\_hjepa\_mv.yaml} configs.)
\begin{itemize}
  \item \texttt{hamiltonian}: Choice of Hamiltonian parameterization. We use \texttt{separable}, i.e.\ $H(q,p)=\tfrac12\|p\|^2 + V_\phi(q)$, to enable explicit symplectic leapfrog integration.
  \item \texttt{method}: Integrator choice. We use \texttt{leapfrog} (velocity Verlet), implemented in \texttt{integrate\_separable\_leapfrog} for separable $H$.
  \item \texttt{steps}: Number of integrator steps $S$ per forward pass. We use \texttt{2} to keep compute low while still giving a nontrivial flow.
  \item \texttt{dt}: Fixed step size $\Delta t$. Used as a constant buffer when \texttt{learn\_dt=false}.
  \item \texttt{learn\_dt}: If true, $\Delta t$ becomes learnable via a softplus parameterization. MV-HJEPA disables this (\texttt{false}) and the training script explicitly rejects learnable-$dt$ to avoid degenerate solutions (``dt-collapse'').
  \item \texttt{hidden\_dim}, \texttt{depth}: MLP size/depth used inside the Hamiltonian residual network (for $V_\phi$ in the separable case).
  \item \texttt{residual\_scale}: Multiplier on the Hamiltonian residual output. This is a stability knob: the code is designed so the Hamiltonian has a quadratic base plus a small learned residual.
  \item \texttt{residual\_scale\_warmup\_epochs}: In the CIFAR-100 loop, scales \texttt{residual\_scale} by $\min(1,e/10)$ at epoch $e\ge1$ (implemented by overwriting \texttt{predictor.H.residual\_scale} each epoch). The ImageNet-100 loop does not apply this field and uses the fixed configured residual scale.
  \item \texttt{base\_coeff}: Coefficient on the quadratic base potential term in $V(q)$ (i.e.\ $0.5\,\texttt{base\_coeff}\,\|q\|^2$). A nonzero coefficient supplies a quadratic base, but the learned residual can modify or partially cancel its force.
  \item \texttt{integrate\_fp32}: If true, the predictor casts $z$ to fp32 during integration when training in bf16/fp16, improving numerical stability of gradient-based integration.
\end{itemize}

\subsubsection{MV-HJEPA loss (\texttt{loss:})}
\begin{itemize}
  \item \texttt{match}: Which part of the state to match in the consistency loss:
    \texttt{q} matches positions with an auxiliary $p$ penalty when \texttt{p\_weight}>0 (the CIFAR-100 value is $0.2$), while \texttt{qp} matches the full state.
  \item \texttt{p\_weight}: Only used when \texttt{match=q}. Adds an auxiliary MSE on $p$ with weight \texttt{p\_weight}.
  \item \texttt{detach\_target}: If true, the target branch is detached in the prediction loss (safer default to prevent the target moving to satisfy the predictor).
  \item \texttt{energy\_weight}: Optional extra term matching $H(z)$ before/after the flow. Set to \texttt{0.0} in all provided configs (disabled).
  \item \texttt{bidirectional}: If true, averages both view orders ($z_0\!\to\!z_1$ and $z_1\!\to\!z_0$). In the ImageNet training script this averaging is done explicitly in the loop with positive step size in both orders; in the CIFAR script it is handled inside the loss module with negative step size for the return order.
\end{itemize}

\subsubsection{MV-HJEPA anti-collapse regularizers (\texttt{regularizer:})}
(These are the \emph{phase-space} regularizers used in the MV-HJEPA branch; they are not the SIGReg baseline.)
\begin{itemize}
  \item \texttt{q\_per\_dim\_target}, \texttt{p\_per\_dim\_target}: Targets for the phase-space energy budget:
  they control the desired per-dimension mean-square magnitude of $q$ and $p$.
  \item \texttt{q\_std\_floor}: Per-dimension standard deviation floor used by \texttt{VarianceFloor} on $q$ (and on $p$ in the CIFAR MV script; ImageNet MV applies the variance floor to $q$ only).
  \item \texttt{ddp\_sync}: If true and DDP is initialized, the sufficient statistics for the budget/variance/logdet terms are aggregated across ranks using all-reduce. This makes the regularizers behave as if computed on the global batch.
  \item \texttt{q\_logdet\_proj\_dim}, \texttt{p\_logdet\_proj\_dim}: Projection dimension $k$ for the log-det regularizer on $q$ and $p$.
  \item \texttt{q\_logdet\_floor}, \texttt{p\_logdet\_floor}: Minimum allowed log-determinant per projected dimension (a volume floor) used by \texttt{ProjectedLogDetFloor}.
  \item \texttt{q\_logdet\_eps}, \texttt{p\_logdet\_eps}: Diagonal jitter $\varepsilon I$ added to the covariance before computing \texttt{slogdet}.
  \item \texttt{q\_logdet\_refresh\_interval}, \texttt{p\_logdet\_refresh\_interval}: How often the random projection matrix is resampled (larger values reduce overhead and noise by reusing projections).
  \item \texttt{q\_pr\_norm\_floor}, \texttt{p\_pr\_norm\_floor}: Optional lower bounds on normalized participation ratio $\mathrm{PR}/k$ (encourages higher effective rank).
  \item \texttt{q\_eigmax\_frac\_ceiling}, \texttt{p\_eigmax\_frac\_ceiling}: Optional upper bounds on $\lambda_{\max}/\mathrm{tr}(\Sigma)$ to prevent a single dominant direction.
\end{itemize}

\subsubsection{Training (\texttt{train:})}
\begin{itemize}
  \item \texttt{epochs}: Total training epochs.
  \item \texttt{lr}: Base AdamW learning rate for encoder/projector (and predictor in CIFAR MV).
  \item \texttt{h\_lr} (ImageNet MV config): Separate learning rate used specifically for the Hamiltonian parameters inside the predictor (\texttt{predictor.H.parameters()}); the ImageNet training script builds a two-group optimizer when these parameters exist.
  \item \texttt{weight\_decay}: AdamW weight decay for the main optimizer.
  \item \texttt{warmup\_epochs}: Warmup duration used by the linear warmup + cosine decay scheduler (when enabled by the script).
  \item \texttt{min\_lr\_ratio}: Sets $\eta_{\min}=\texttt{lr}\cdot\texttt{min\_lr\_ratio}$ for cosine annealing.
  \item \texttt{precision}: Mixed precision mode. \texttt{bf16} enables CUDA autocast bf16 in training.
  \item \texttt{ckpt\_dir}: Output directory for checkpoints.
  \item \texttt{checkpoint\_every} (ImageNet configs): Save checkpoints every $k$ epochs (ImageNet script only).
  \item \texttt{grad\_clip}: If $>0$, applies global-norm gradient clipping (\texttt{clip\_grad\_norm\_}). The CIFAR-100 parameter list includes the predictor; the ImageNet-100 clipping call excludes the separately grouped Hamiltonian parameters.
  \item \texttt{log\_every}: Print training diagnostics every \texttt{log\_every} steps.
  \item \texttt{profile\_timing}, \texttt{profile\_every} (ImageNet MV config): Optional lightweight profiling of data time vs step time (only used in the ImageNet script).
  \item \textbf{MV-HJEPA loss weights:} \texttt{lambda\_budget}, \texttt{lambda\_var}, \texttt{lambda\_logdet}, \texttt{lambda\_mean} weight the corresponding regularizers in Algorithm~\ref{alg:mv-hjepa-train}.
  \item \textbf{SIGReg baseline weights:} \texttt{lambda\_reg} is the coefficient on the SIGReg term in Algorithm~\ref{alg:lejepa-sigreg-train}.
  \item \texttt{h\_update\_interval}, \texttt{h\_grad\_clip} (SIGReg YAMLs): These knobs are parsed by the training scripts for the \emph{learnable Hamiltonian regularizer} path (\texttt{regularizer.type=ham\_sigreg} with learnable $H$). In the provided SIGReg-only configs (\texttt{type=sigreg}), there are no $H$ parameters, so these settings are effectively inert.
\end{itemize}

\clearpage
\input{checklist}

\end{document}

%% file: checklist.tex
\section*{NeurIPS Paper Checklist}



\begin{enumerate}

\item {\bf Claims}
    \item[] Question: Do the main claims made in the abstract and introduction accurately reflect the paper's contributions and scope?
    \item[] Answer: \answerYes{} 
    \item[] Justification: The abstract and introduction state the marginal-geometry theory, the canonical Hamiltonian lift under explicit oracle assumptions, and the controlled HamJEPA experiments. The matched controls rule out attributing the empirical gain uniquely to symplecticity or reversibility. The theory claims are supported by Section~\ref{sec:price-of-isotropy-theory} and the appendix; the empirical claims are supported by Section~\ref{sec:results}.
    \item[] Guidelines:
    \begin{itemize}
        \item The answer \answerNA{} means that the abstract and introduction do not include the claims made in the paper.
        \item The abstract and/or introduction should clearly state the claims made, including the contributions made in the paper and important assumptions and limitations. A \answerNo{} or \answerNA{} answer to this question will not be perceived well by the reviewers. 
        \item The claims made should match theoretical and experimental results, and reflect how much the results can be expected to generalize to other settings. 
        \item It is fine to include aspirational goals as motivation as long as it is clear that these goals are not attained by the paper. 
    \end{itemize}

\item {\bf Limitations}
    \item[] Question: Does the paper discuss the limitations of the work performed by the authors?
    \item[] Answer: \answerYes{} 
    \item[] Justification: The paper includes a limitations paragraph in Section~\ref{sec:conclusion}, noting that this is a controlled geometry study rather than a SOTA SSL recipe, that the main baseline is SIGReg, that ImageNet-100 is single-seed, and that the downstream operator $H$ is not learned. It also states that symplecticity and reversibility are not individually necessary in the tested controls, and that learned-marginal baselines and full-recipe generalization remain untested.
    \item[] Guidelines:
    \begin{itemize}
        \item The answer \answerNA{} means that the paper has no limitation while the answer \answerNo{} means that the paper has limitations, but those are not discussed in the paper. 
        \item The authors are encouraged to create a separate ``Limitations'' section in their paper.
        \item The paper should point out any strong assumptions and how robust the results are to violations of these assumptions (e.g., independence assumptions, noiseless settings, model well-specification, asymptotic approximations only holding locally). The authors should reflect on how these assumptions might be violated in practice and what the implications would be.
        \item The authors should reflect on the scope of the claims made, e.g., if the approach was only tested on a few datasets or with a few runs. In general, empirical results often depend on implicit assumptions, which should be articulated.
        \item The authors should reflect on the factors that influence the performance of the approach. For example, a facial recognition algorithm may perform poorly when image resolution is low or images are taken in low lighting. Or a speech-to-text system might not be used reliably to provide closed captions for online lectures because it fails to handle technical jargon.
        \item The authors should discuss the computational efficiency of the proposed algorithms and how they scale with dataset size.
        \item If applicable, the authors should discuss possible limitations of their approach to address problems of privacy and fairness.
        \item While the authors might fear that complete honesty about limitations might be used by reviewers as grounds for rejection, a worse outcome might be that reviewers discover limitations that aren't acknowledged in the paper. The authors should use their best judgment and recognize that individual actions in favor of transparency play an important role in developing norms that preserve the integrity of the community. Reviewers will be specifically instructed to not penalize honesty concerning limitations.
    \end{itemize}

\item {\bf Theory assumptions and proofs}
    \item[] Question: For each theoretical result, does the paper provide the full set of assumptions and a complete (and correct) proof?
    \item[] Answer: \answerYes{} 
    \item[] Justification: The theoretical statements specify their assumptions in the theorem/lemma statements. Complete proofs are provided in the appendix.
    \item[] Guidelines:
    \begin{itemize}
        \item The answer \answerNA{} means that the paper does not include theoretical results. 
        \item All the theorems, formulas, and proofs in the paper should be numbered and cross-referenced.
        \item All assumptions should be clearly stated or referenced in the statement of any theorems.
        \item The proofs can either appear in the main paper or the supplemental material, but if they appear in the supplemental material, the authors are encouraged to provide a short proof sketch to provide intuition. 
        \item Inversely, any informal proof provided in the core of the paper should be complemented by formal proofs provided in appendix or supplemental material.
        \item Theorems and Lemmas that the proof relies upon should be properly referenced. 
    \end{itemize}

    \item {\bf Experimental result reproducibility}
    \item[] Question: Does the paper fully disclose all the information needed to reproduce the main experimental results of the paper to the extent that it affects the main claims and/or conclusions of the paper (regardless of whether the code and data are provided or not)?
    \item[] Answer: \answerYes{} 
    \item[] Justification: The experimental protocol, datasets, view construction, backbone, projector choice, optimization setup, and evaluation metrics are described in Section~\ref{sec:results} and Appendix~\ref{app:impl}. Hyperparameters and additional implementation details are provided in the appendix. Appendix~\ref{app:replicated_controls} specifies the exact additional-control splitting coefficients, preprocessing, probe settings, and uncertainty definitions. The accompanying archive provides source snapshots, configurations, and recorded commands for reproducing the additional experiments. Appendix~\ref{app:source_provenance} distinguishes these reproducibility details from the unresolved exact producing revision of the historical ImageNet-100 checkpoint.
    \item[] Guidelines:
    \begin{itemize}
        \item The answer \answerNA{} means that the paper does not include experiments.
        \item If the paper includes experiments, a \answerNo{} answer to this question will not be perceived well by the reviewers: Making the paper reproducible is important, regardless of whether the code and data are provided or not.
        \item If the contribution is a dataset and\slash or model, the authors should describe the steps taken to make their results reproducible or verifiable. 
        \item Depending on the contribution, reproducibility can be accomplished in various ways. For example, if the contribution is a novel architecture, describing the architecture fully might suffice, or if the contribution is a specific model and empirical evaluation, it may be necessary to either make it possible for others to replicate the model with the same dataset, or provide access to the model. In general. releasing code and data is often one good way to accomplish this, but reproducibility can also be provided via detailed instructions for how to replicate the results, access to a hosted model (e.g., in the case of a large language model), releasing of a model checkpoint, or other means that are appropriate to the research performed.
        \item While NeurIPS does not require releasing code, the conference does require all submissions to provide some reasonable avenue for reproducibility, which may depend on the nature of the contribution. For example
        \begin{enumerate}
            \item If the contribution is primarily a new algorithm, the paper should make it clear how to reproduce that algorithm.
            \item If the contribution is primarily a new model architecture, the paper should describe the architecture clearly and fully.
            \item If the contribution is a new model (e.g., a large language model), then there should either be a way to access this model for reproducing the results or a way to reproduce the model (e.g., with an open-source dataset or instructions for how to construct the dataset).
            \item We recognize that reproducibility may be tricky in some cases, in which case authors are welcome to describe the particular way they provide for reproducibility. In the case of closed-source models, it may be that access to the model is limited in some way (e.g., to registered users), but it should be possible for other researchers to have some path to reproducing or verifying the results.
        \end{enumerate}
    \end{itemize}

\item {\bf Open access to data and code}
    \item[] Question: Does the paper provide open access to the data and code, with sufficient instructions to faithfully reproduce the main experimental results, as described in supplemental material?
    \item[] Answer: \answerYes{} 
    \item[] Justification: The public implementation is linked in the introduction. The accompanying supplementary archive provides the source snapshots, configurations, recorded commands, evaluation definitions, and saved outputs for the additional CIFAR-100 experiments. The code specifies dataset access and preparation; dataset images are not redistributed and ImageNet access remains subject to its terms. Appendix~\ref{app:source_provenance} states the archive scope and the unresolved historical source-to-checkpoint provenance.
    \item[] Guidelines:
    \begin{itemize}
        \item The answer \answerNA{} means that paper does not include experiments requiring code.
        \item Please see the NeurIPS code and data submission guidelines (\url{https://neurips.cc/public/guides/CodeSubmissionPolicy}) for more details.
        \item While we encourage the release of code and data, we understand that this might not be possible, so \answerNo{} is an acceptable answer. Papers cannot be rejected simply for not including code, unless this is central to the contribution (e.g., for a new open-source benchmark).
        \item The instructions should contain the exact command and environment needed to run to reproduce the results. See the NeurIPS code and data submission guidelines (\url{https://neurips.cc/public/guides/CodeSubmissionPolicy}) for more details.
        \item The authors should provide instructions on data access and preparation, including how to access the raw data, preprocessed data, intermediate data, and generated data, etc.
        \item The authors should provide scripts to reproduce all experimental results for the new proposed method and baselines. If only a subset of experiments are reproducible, they should state which ones are omitted from the script and why.
        \item At submission time, to preserve anonymity, the authors should release anonymized versions (if applicable).
        \item Providing as much information as possible in supplemental material (appended to the paper) is recommended, but including URLs to data and code is permitted.
    \end{itemize}

\item {\bf Experimental setting/details}
    \item[] Question: Does the paper specify all the training and test details (e.g., data splits, hyperparameters, how they were chosen, type of optimizer) necessary to understand the results?
    \item[] Answer: \answerYes{} 
    \item[] Justification: Section~\ref{sec:results} describes the controlled protocol, datasets, readouts, and evaluation metrics. Training and hyperparameter details, including the active loss weights and dataset-specific matching and warmup semantics, are provided in Appendix~\ref{app:impl}; the exact additional-control splitting coefficients and evaluation definitions are provided in Appendix~\ref{app:replicated_controls}.
    \item[] Guidelines:
    \begin{itemize}
        \item The answer \answerNA{} means that the paper does not include experiments.
        \item The experimental setting should be presented in the core of the paper to a level of detail that is necessary to appreciate the results and make sense of them.
        \item The full details can be provided either with the code, in appendix, or as supplemental material.
    \end{itemize}

\item {\bf Experiment statistical significance}
    \item[] Question: Does the paper report error bars suitably and correctly defined or other appropriate information about the statistical significance of the experiments?
    \item[] Answer: \answerYes{} 
    \item[] Justification: CIFAR-100 30-epoch results report mean and standard deviation over three seeds. The additional 80-epoch CIFAR-100 comparisons use two independent pretraining seeds; reported error bars are descriptive sample standard deviations across those runs, with each linear-probe value first averaged over five probe optimizations. The original 80-epoch results, ImageNet-100 results, and single-seed sensitivity conditions are explicitly distinguished. No formal significance is claimed from two pretraining seeds, and correlated metric/$k$ sweeps are not independent replications.
    \item[] Guidelines:
    \begin{itemize}
        \item The answer \answerNA{} means that the paper does not include experiments.
        \item The authors should answer \answerYes{} if the results are accompanied by error bars, confidence intervals, or statistical significance tests, at least for the experiments that support the main claims of the paper.
        \item The factors of variability that the error bars are capturing should be clearly stated (for example, train/test split, initialization, random drawing of some parameter, or overall run with given experimental conditions).
        \item The method for calculating the error bars should be explained (closed form formula, call to a library function, bootstrap, etc.)
        \item The assumptions made should be given (e.g., Normally distributed errors).
        \item It should be clear whether the error bar is the standard deviation or the standard error of the mean.
        \item It is OK to report 1-sigma error bars, but one should state it. The authors should preferably report a 2-sigma error bar than state that they have a 96\% CI, if the hypothesis of Normality of errors is not verified.
        \item For asymmetric distributions, the authors should be careful not to show in tables or figures symmetric error bars that would yield results that are out of range (e.g., negative error rates).
        \item If error bars are reported in tables or plots, the authors should explain in the text how they were calculated and reference the corresponding figures or tables in the text.
    \end{itemize}

\item {\bf Experiments compute resources}
    \item[] Question: For each experiment, does the paper provide sufficient information on the computer resources (type of compute workers, memory, time of execution) needed to reproduce the experiments?
    \item[] Answer: \answerYes{} 
    \item[] Justification: Appendix~\ref{app:impl} reports the hardware, training time, batch sizes, number of epochs, and evaluation resources used for the main experiments. Historical ImageNet-100 timings are reported separately for HamJEPA and SIGReg as notebook-cell wall times. The additional matched A100 preflight benchmark reports cold per-step time, throughput, parameter overhead, and peak allocated memory from one timed step per method and seed, with no warmup; these measurements do not establish steady-state throughput.
    \item[] Guidelines:
    \begin{itemize}
        \item The answer \answerNA{} means that the paper does not include experiments.
        \item The paper should indicate the type of compute workers CPU or GPU, internal cluster, or cloud provider, including relevant memory and storage.
        \item The paper should provide the amount of compute required for each of the individual experimental runs as well as estimate the total compute. 
        \item The paper should disclose whether the full research project required more compute than the experiments reported in the paper (e.g., preliminary or failed experiments that didn't make it into the paper). 
    \end{itemize}
    
\item {\bf Code of ethics}
    \item[] Question: Does the research conducted in the paper conform, in every respect, with the NeurIPS Code of Ethics \url{https://neurips.cc/public/EthicsGuidelines}?
    \item[] Answer: \answerYes{} 
    \item[] Justification: The work is a foundational representation-learning study on standard image benchmarks and does not involve human-subject experiments, scraped private data, security vulnerabilities, or high-risk deployment claims.
    \item[] Guidelines:
    \begin{itemize}
        \item The answer \answerNA{} means that the authors have not reviewed the NeurIPS Code of Ethics.
        \item If the authors answer \answerNo, they should explain the special circumstances that require a deviation from the Code of Ethics.
        \item The authors should make sure to preserve anonymity (e.g., if there is a special consideration due to laws or regulations in their jurisdiction).
    \end{itemize}

\item {\bf Broader impacts}
    \item[] Question: Does the paper discuss both potential positive societal impacts and negative societal impacts of the work performed?
    \item[] Answer: \answerYes{} 
    \item[] Justification: HamJEPA may improve controlled visual representation learning. Potential negative impacts are generic to improved visual representations, including downstream misuse in surveillance, biometric identification, or biased high-stakes classifiers; deployment should require dataset documentation, privacy review, and fairness evaluation.
    \item[] Guidelines:
    \begin{itemize}
        \item The answer \answerNA{} means that there is no societal impact of the work performed.
        \item If the authors answer \answerNA{} or \answerNo, they should explain why their work has no societal impact or why the paper does not address societal impact.
        \item Examples of negative societal impacts include potential malicious or unintended uses (e.g., disinformation, generating fake profiles, surveillance), fairness considerations (e.g., deployment of technologies that could make decisions that unfairly impact specific groups), privacy considerations, and security considerations.
        \item The conference expects that many papers will be foundational research and not tied to particular applications, let alone deployments. However, if there is a direct path to any negative applications, the authors should point it out. For example, it is legitimate to point out that an improvement in the quality of generative models could be used to generate Deepfakes for disinformation. On the other hand, it is not needed to point out that a generic algorithm for optimizing neural networks could enable people to train models that generate Deepfakes faster.
        \item The authors should consider possible harms that could arise when the technology is being used as intended and functioning correctly, harms that could arise when the technology is being used as intended but gives incorrect results, and harms following from (intentional or unintentional) misuse of the technology.
        \item If there are negative societal impacts, the authors could also discuss possible mitigation strategies (e.g., gated release of models, providing defenses in addition to attacks, mechanisms for monitoring misuse, mechanisms to monitor how a system learns from feedback over time, improving the efficiency and accessibility of ML).
    \end{itemize}
    
\item {\bf Safeguards}
    \item[] Question: Does the paper describe safeguards that have been put in place for responsible release of data or models that have a high risk for misuse (e.g., pre-trained language models, image generators, or scraped datasets)?
    \item[] Answer: \answerNA{} 
    \item[] Justification: The paper does not release a high-risk generative model, surveillance system, scraped dataset, security exploit, or dual-use capability requiring special access controls.
    \item[] Guidelines:
    \begin{itemize}
        \item The answer \answerNA{} means that the paper poses no such risks.
        \item Released models that have a high risk for misuse or dual-use should be released with necessary safeguards to allow for controlled use of the model, for example by requiring that users adhere to usage guidelines or restrictions to access the model or implementing safety filters. 
        \item Datasets that have been scraped from the Internet could pose safety risks. The authors should describe how they avoided releasing unsafe images.
        \item We recognize that providing effective safeguards is challenging, and many papers do not require this, but we encourage authors to take this into account and make a best faith effort.
    \end{itemize}

\item {\bf Licenses for existing assets}
    \item[] Question: Are the creators or original owners of assets (e.g., code, data, models), used in the paper, properly credited and are the license and terms of use explicitly mentioned and properly respected?
    \item[] Answer: \answerNo{} 
    \item[] Justification: The paper credits the original datasets and PyTorch, and Appendix~\ref{app:impl} gives the verified software license links and ImageNet access terms. The official CIFAR distribution page does not state a separate named dataset license, and the supplied artifacts do not pin all runtime versions. No datasets are redistributed.
    \item[] Guidelines:
    \begin{itemize}
        \item The answer \answerNA{} means that the paper does not use existing assets.
        \item The authors should cite the original paper that produced the code package or dataset.
        \item The authors should state which version of the asset is used and, if possible, include a URL.
        \item The name of the license (e.g., CC-BY 4.0) should be included for each asset.
        \item For scraped data from a particular source (e.g., website), the copyright and terms of service of that source should be provided.
        \item If assets are released, the license, copyright information, and terms of use in the package should be provided. For popular datasets, \url{paperswithcode.com/datasets} has curated licenses for some datasets. Their licensing guide can help determine the license of a dataset.
        \item For existing datasets that are re-packaged, both the original license and the license of the derived asset (if it has changed) should be provided.
        \item If this information is not available online, the authors are encouraged to reach out to the asset's creators.
    \end{itemize}

\item {\bf New assets}
    \item[] Question: Are new assets introduced in the paper well documented and is the documentation provided alongside the assets?
    \item[] Answer: \answerNo{} 
    \item[] Justification: The paper does not introduce a new dataset or release model weights as a required research artifact. The accompanying source and result archive documents the additional controls, evaluation definitions, configurations, recorded commands, and provenance limitations. The public implementation's license and LeJEPA attribution are identified, but the recovered source snapshots do not include a complete statement of licensing scope for every archived component.
    \item[] Guidelines:
    \begin{itemize}
        \item The answer \answerNA{} means that the paper does not release new assets.
        \item Researchers should communicate the details of the dataset\slash code\slash model as part of their submissions via structured templates. This includes details about training, license, limitations, etc. 
        \item The paper should discuss whether and how consent was obtained from people whose asset is used.
        \item At submission time, remember to anonymize your assets (if applicable). You can either create an anonymized URL or include an anonymized zip file.
    \end{itemize}

\item {\bf Crowdsourcing and research with human subjects}
    \item[] Question: For crowdsourcing experiments and research with human subjects, does the paper include the full text of instructions given to participants and screenshots, if applicable, as well as details about compensation (if any)? 
    \item[] Answer: \answerNA{} 
    \item[] Justification: The work does not involve crowdsourcing or research with human subjects.
    \item[] Guidelines:
    \begin{itemize}
        \item The answer \answerNA{} means that the paper does not involve crowdsourcing nor research with human subjects.
        \item Including this information in the supplemental material is fine, but if the main contribution of the paper involves human subjects, then as much detail as possible should be included in the main paper. 
        \item According to the NeurIPS Code of Ethics, workers involved in data collection, curation, or other labor should be paid at least the minimum wage in the country of the data collector. 
    \end{itemize}

\item {\bf Institutional review board (IRB) approvals or equivalent for research with human subjects}
    \item[] Question: Does the paper describe potential risks incurred by study participants, whether such risks were disclosed to the subjects, and whether Institutional Review Board (IRB) approvals (or an equivalent approval/review based on the requirements of your country or institution) were obtained?
    \item[] Answer: \answerNA{} 
    \item[] Justification: The work does not involve human-subject experiments, so IRB approval is not applicable.
    \item[] Guidelines:
    \begin{itemize}
        \item The answer \answerNA{} means that the paper does not involve crowdsourcing nor research with human subjects.
        \item Depending on the country in which research is conducted, IRB approval (or equivalent) may be required for any human subjects research. If you obtained IRB approval, you should clearly state this in the paper. 
        \item We recognize that the procedures for this may vary significantly between institutions and locations, and we expect authors to adhere to the NeurIPS Code of Ethics and the guidelines for their institution. 
        \item For initial submissions, do not include any information that would break anonymity (if applicable), such as the institution conducting the review.
    \end{itemize}

\item {\bf Declaration of LLM usage}
    \item[] Question: Does the paper describe the usage of LLMs if it is an important, original, or non-standard component of the core methods in this research? Note that if the LLM is used only for writing, editing, or formatting purposes and does \emph{not} impact the core methodology, scientific rigor, or originality of the research, declaration is not required.
    \item[] Answer: \answerNA{} 
    \item[] Justification: LLMs are not an important, original, or non-standard component of the core method or experiments. Any use for writing, editing, or formatting does not affect the scientific method or results.
    \item[] Guidelines:
    \begin{itemize}
        \item The answer \answerNA{} means that the core method development in this research does not involve LLMs as any important, original, or non-standard components.
        \item Please refer to our LLM policy in the NeurIPS handbook for what should or should not be described.
    \end{itemize}

\end{enumerate}